%% file: main_nips.tex
\documentclass{article}

\usepackage[preprint]{neurips_2021}

\input{support_algo}

\input{support_english}

\input{support_figure}
\input{support_linking}

\input{support_math}
\input{support_mathenv}

\input{support_mathvar}

\input{support_refer}

\input{support_table}
\input{support_misc}

\newcommand{\manus}{paper} %

\title{Approximate discounting-free policy evaluation \\
from transient and recurrent states}

\author{%
  Vektor Dewanto, Marcus Gallagher \\
  School of Information Technology and Electrical Engineering \\
  University of Queensland, Australia \\
  \texttt{v.dewanto@uqconnect.edu.au, marcusg@uq.edu.au} \\
}

\begin{document}
\maketitle

\input{abstract}
\input{intro}
\input{prelim}
\input{seminormlstd}

\input{unify}

\input{approxbias}

\input{xprmtsetup}

\input{xprmtresult}

\input{conclude}

\afterpage{
\clearpage%
\bibliography{main}
\bibliographystyle{apalike}
\clearpage%
}

\end{document}

%% file: support_algo.tex
\usepackage[linesnumbered,ruled,vlined]{algorithm2e}

\SetKwComment{Comment}{$\triangleright$\ }{}

\SetKwInput{KwInput}{Input}                %
\SetKwInput{KwOutput}{Output}              %

\SetKw{Andd}{and} %
\SetKw{Break}{break}
\SetKw{Continue}{continue}
\SetKw{Return}{return}

%% file: support_english.tex
\newcommand{\cf}{cf.~}

\newcommand{\eg}{e.g.~}

\newcommand{\ie}{i.e.~}

%% file: support_figure.tex
\usepackage{tikz}
\usetikzlibrary{positioning, shapes.geometric, backgrounds}

\usetikzlibrary{arrows.meta}

\usepackage{subcaption}

\usepackage{grffile}

\usepackage{graphicx} %
\DeclareGraphicsExtensions{.pdf,.jpeg,.png,.jpg}

\usepackage{rotating}

\usepackage{wrapfig}

%% file: support_linking.tex
\usepackage[hidelinks, bookmarksnumbered, backref=page]{hyperref}
\renewcommand*{\backref}[1]{}
\renewcommand*{\backrefalt}[4]{
  \ifcase #1 (Broken backref)
  \or        (page~#2)
  \else      (pages~#2)
  \fi
}

\hypersetup{
  colorlinks   = true,
  urlcolor     = blue,
  linkcolor    = blue,
  citecolor   = blue
}

\usepackage{url}            %

%% file: support_math.tex
\usepackage{scalerel}
\usepackage{amsmath}
\usepackage{amsthm}
\usepackage{amssymb}
\usepackage{amsfonts}       %
\usepackage{mathtools}
\usepackage{cancel}
\usepackage{bbm} %

\usepackage{mathtools}

\DeclarePairedDelimiter\floor{\lfloor}{\rfloor}

\newcommand{\diag}{\mathrm{diag}}

\newcommand{\E}[2]{\mathbb{E}_{#1} \left[ #2 \right]} %

\newcommand{\norm}[1]{\|#1\|} %

\newcommand{\prob}[1]{\mathrm{Pr}\{#1\}}

\newcommand{\real}[1]{\mathbb{R}^{#1}}
\newcommand{\integer}[1]{\mathbb{Z}^{#1}}

\newcommand{\setname}[1]{\mathcal{#1}}
\newcommand{\setsize}[1]{|\mathcal{#1}|}
\newcommand{\setsizesubt}[1]{|\mathcal{#1}_t|}

\newcommand{\spanspace}{\mathrm{span}}

\newcommand{\eqset}{\mathrel{\overset{\makebox[0pt]{\mbox{\normalfont\tiny\sffamily set}}}{=}}}

\newcommand{\eqdef}{\coloneqq}
\newcommand{\eqdefr}{\eqqcolon}

\DeclareMathOperator*{\argmin}{\arg\!\min}

\newcommand{\vecb}[1]{\boldsymbol{#1}}
\newcommand{\mat}[1]{\boldsymbol{#1}}

\let\originalleft\left
\let\originalright\right
\renewcommand{\left}{\mathopen{}\mathclose\bgroup\originalleft}
\renewcommand{\right}{\aftergroup\egroup\originalright}

%% file: support_mathenv.tex
\newtheorem{theorem}{Theorem}[section]

\newtheorem{definition}{Definition}[section]
\newtheorem{lemma}{Lemma}[section]

\newtheorem{assumption}{Assumption}[section]

%% file: support_mathvar.tex
\newcommand{\szrat}{s_{\mathrm{zrat}}} %

\newcommand{\sref}{s_{\mathrm{ref}}}

\newcommand{\strans}{s_{\mathrm{tr}}}

\newcommand{\diagpstar}{\mat{D}_{\!\!\vecb{p}^{\!\star}}}
\newcommand{\diagptilde}{\mat{D}_{\!\tilde{\vecb{p}}}}
\newcommand{\diagptildesqrt}{ \mat{D}_{\!\tilde{\vecb{p}}}^{\!\frac{1}{2}} }
\newcommand{\diagsvd}{\mat{D}_{\!\!\sigma}}
\newcommand{\diagsvds}{\mat{D}_{\!\!\sigma}^2}
\newcommand{\diageig}{\mat{D}_{\!\!\mu}}

\newcommand{\ptrmat}{ \mat{P}_{\!\!\mathrm{tr}} } %

\newcommand{\settrans}{\mathcal{S}_{\mathrm{tr}}}
\newcommand{\settr}{\settrans}
\newcommand{\settrsize}{ |\settr| }

\newcommand{\srefset}{\mathcal{S}_{\mathrm{ref}}}

\newcommand{\isd}{\mathring{p}} %
\newcommand{\isdvec}{\mathring{\vecb{p}}}
\newcommand{\isdvecrow}{ \isdvec^{\!\intercal} }

\newcommand{\tmax}{t_{\mathrm{max}}}

\newcommand{\txepmax}{\hat{t}_{\mathrm{max}}^{\mathrm{xep}}} %

\newcommand{\tmix}{t_{ \mathrm{mix}} } %
\newcommand{\tabs}{t_{\!\mathrm{abs}}}

\newcommand{\tabsmax}{\tabs^{\scaleto{\mathrm{max}\mathstrut}{5pt}}} %
\newcommand{\tabsmaxb}{\tabs^{\mathrm{max}}}

\newcommand{\ems}{\tilde{e}_\mathrm{MS}} %
\newcommand{\emst}{e_\mathrm{MS}^t}
\newcommand{\emstot}{\varepsilon_\mathrm{MS}}

\newcommand{\epb}{\tilde{e}_{\mathbb{PB}}}
\newcommand{\epbt}{e_{\mathbb{PB}}^{t}} %
\newcommand{\epbtot}{ \varepsilon_{\mathbb{PB}} }

\newcommand{\bo}{\mathbb{B}} %
\newcommand{\po}{\mathbb{P}} %

\newcommand{\nanchor}{n_{\mathrm{a}}} %
\newcommand{\nanchormax}{n_{\mathrm{a}}^{\mathrm{max}}}
\newcommand{\nsample}{n_{\mathrm{sam}}} %

\newcommand{\nxep}{n_{\mathrm{xep}}}

\newcommand{\zpinvsqrt}{ \mat{Z}^{\nicefrac{\dagger}{2}} }

%% file: support_refer.tex
\newcommand{\algref}[1]{Algo~\ref{#1}}
\newcommand{\algrefand}[2]{Algo~\ref{#1} and~\ref{#2}}

\newcommand{\assref}[1]{Assumption~\ref{#1}}

\newcommand{\defref}[1]{Def~\ref{#1}}

\newcommand{\figref}[1]{Fig~\ref{#1}}

\newcommand{\itemref}[1]{Item~\ref{#1}}

\newcommand{\lineref}[1]{Line~\ref{#1}}
\newcommand{\lmmref}[1]{Lemma~\ref{#1}}
\newcommand{\lmmrefand}[2]{Lemmas~\ref{#1} and~\ref{#2}}

\newcommand{\secref}[1]{Sec~\ref{#1}}

\newcommand{\secrefand}[2]{Secs~\ref{#1} and \ref{#2}}

\newcommand{\thmref}[1]{Thm~\ref{#1}}
\newcommand{\tblref}[1]{Table~\ref{#1}}

\newcommand{\tblrefto}[2]{Tables~\ref{#1} to~\ref{#2}}

\newcommand{\app}[1]{Appendix~#1}
\newcommand{\assume}[1]{Assumption~#1}

\newcommand{\cor}[1]{Corollary~#1} %

\newcommand{\deff}[1]{Def~#1}

\newcommand{\lmm}[1]{Lemma~#1}

\newcommand{\page}[1]{p#1}

\newcommand{\tbl}[1]{Table~#1}
\newcommand{\secc}[1]{Sec~#1}

\newcommand{\thm}[1]{Thm~#1}

%% file: support_table.tex
\usepackage{booktabs}       %
\usepackage{multirow}
\usepackage{longtable}

\usepackage{colortbl}

%% file: support_misc.tex
\usepackage{enumitem}
\setlist{nosep}

\usepackage{xcolor}
\colorlet{darkgreen}{green!50!black}

\definecolor{stateblue}{cmyk}{0.96,0,0,0}
\definecolor{lightgray}{gray}{0.85}
\definecolor{halfgreen}{rgb}{0,0.5,0}

\usepackage{pdflscape}
\usepackage{afterpage}

\usepackage{nicefrac}       %
\usepackage{microtype}      %
\usepackage[yyyymmdd,hhmmss]{datetime}
\usepackage{lipsum}

\usepackage[utf8]{inputenc} %
\usepackage[T1]{fontenc}    %

\usepackage{import}

\usepackage{epigraph}

\usepackage{acronym}

\usepackage{multicol}

%% file: abstract.tex
\begin{abstract}
In order to distinguish policies that prescribe
good from bad actions in transient states,
we need to evaluate the so-called bias of a policy from transient states.
However, we observe that most (if not all) works in
approximate discounting-free policy evaluation thus far are developed
for estimating the bias solely from recurrent states.
We therefore propose a system of approximators for the bias
(specifically, its relative value) from transient and recurrent states.
Its key ingredient is a seminorm LSTD (least-squares temporal difference),
for which we derive its minimizer expression that enables approximation by
sampling required in model-free reinforcement learning.
This seminorm LSTD also facilitates the formulation of a general unifying procedure
for LSTD-based policy value approximators.
Experimental results validate the effectiveness of our proposed method.
\end{abstract}

%% file: intro.tex
\section{Introduction}

Consider an environment where there are two types of states:
those that are visited infinitely many times by an agent, and
those that are not (even though the agent is modelled to operate up to infinity).
The members of the former group are called recurrent states, whereas
those of the latter are called transient states.
If all recurrent states form a single closed irreducible set, then we have
the so-called unichain Markov chain (MC).
It is closed in that once the agent is in any member of the set,
the agent cannot go outside to any non-member state.
It is irreducible because from any member of the set, the agent can
visit any other member.
In reinforcement learning (RL), such a unichain MC is induced by (at least)
one of the stationary policies of the Markov decision process (MDP) model,
for which we call such a model a unichain MDP.

The work in this \manus~is concerned with evaluating a stationary policy~$\pi$
from both transient and recurrent states in terms of
discounting-free evaluation functions.
Particularly, the policy value function of interest is the bias (denoted by $v_b$)
of $\pi$ as follows,
\begin{equation}
v_b(\pi, s)
\eqdef \lim_{t_{\mathrm{max}} \to \infty}
\E{A_t \sim \pi(\cdot|s_t), S_{t+1} \sim p(\cdot| s_t, a_t)}{
    \sum_{t=0}^{t_{\mathrm{max}} - 1}
    \Big( r(S_t, A_t) - v_g(\pi) \Big) \Big| S_0 = s, \pi},
\quad \forall s \in \setname{S},
\label{equ:bias_def}
\end{equation}
where $S_t$ and $A_t$ are discrete state and action random variables on
the state set $\setname{S}$ and action set $\setname{A}$
of an infinite-horizon MDP with one-step state transition distribution~$p$,
and reward function~$r$.
Here, $v_g$ denotes the gain (the average-reward) value function, which is
state-invariant whenever the induced MC is unichain.
Both bias and gain do not involve any discount factor (hence, they are said to
be discounting-free; \cf the discounted reward value function).

Evaluating the bias from both state types is essential for carrying out
further policy selection on gain-optimal policies that induce unichain MCs.
This is because the gain value function only concerns with the long-run rewards
(which are earned in recurrent states).
It ignores rewards earned at the outset in transient states, in which
gain-optimal policies therefore cannot distinguish ``good'' from ``bad'' actions.
In other words, they are suboptimal in transient states with respect to
the finest (the most selective) optimality criterion, \ie the Blackwell optimality.

Despite the aforementioned importance, we observe that most works in RL are
designed for approximately evaluating a policy from recurrent states,
specifically for MDPs whose all induced MCs have only recurrent states.
The stationary state distribution is used for weighting the state-wise error terms
in the error function.
This applies to both discounted and discounting-free policy evaluation,
\eg \citep[\assume{3}]{liu_2021_tdgs}, \citep[\secc{2.4.2}]{dann_2014_petd}.
There are a few works that estimate the policy values from transient states.
However, they are applicable merely for MCs with a single recurrent state
whose reward is zero and known (thus no estimation is needed).
For instance, \citet[\thm{1}]{bradtke_1996_lstd} proposed
a discounted-reward estimator for environments with
multiple transient states and a single 0-reward absorbing terminal state.
The error terms are weighted by the visitation probabilities of transient states
from the initial time until absorption, which is known to happen at
the last timestep of a trial (as the agent reaches the absorbing terminal state).

In this \manus, we propose techniques that approximate the bias value
(of any stationary policy) from multiple transient and recurrent states
in model-free RL.
This requires value approximation from both state types,
instead of either one out of two types as the above-mentioned existing works.
Moreover, the state classification is unknown since
the agent does not know the state transition distribution
and does not attempt to estimate it.
This also implies that the agent does not know when it is absorbed into
the closed, irreducible recurrent state set (\ie the absorption time).
If the state classification was known, the state set could be sliced
and two approximators could be built:
one for the recurrent states and one for the transient states;
taking advantage of the existing works.
However in that case, some additional work would still be needed for two reasons.
First is that all recurrent states cannot simply be separated from the state set
(for the sake of the transient state value approximator)
because there must be transitions from some transient states to recurrent states,
which may have non-zero rewards affecting transient state values.
Second is because those two individual bias approximators have different offsets
from the true bias (due to the nature of the error function that they minimize).
Therefore, their approximation results need to be calibrated before being used
simultaneously in a formula that involves bias values of
both transient and recurrent states.
More exposition about these two issues are provided later
in \secrefand{sec:seminorm_lstd}{sec:approxbias}.

We present a system of approximators for the bias (relative) value.
Each approximator is based on least-squares temporal difference (LSTD).
In one extreme (where the computation cost is put aside), the system
instantiates stepwise LSTD approximators that use stepwise state distributions,
denoted as $p_\pi^t$, to weight the state-wise error terms.
Here, $p_\pi^t(s) \eqdef \E{S_0 \sim \isd}{\prob{S_t = s| s_0, \pi}},
\forall s \in \setname{S}$, which indicates the probability of visiting a state $s$
in $t$ timesteps when the agent begins at an initial state $S_0 \sim \isd$
then follows a policy~$\pi$.
Such a system dismisses the need for state classification, but
poses at least three challenges, for which we contribute some solutions.

\textbf{First}, the stepwise state distribution $p_\pi^t$ may not have
the whole state set as its support.
For example, when the agent can only begin in transient states,
$p_\pi^{t=0}$ has zero probability for any recurrent state.
The same goes for transient states with respect to $p_\pi^t$ after
the stationary state distribution is reached
(as there is no chance of visiting transient states
once the recurrent class is entered).
Consequently, the diagonal matrix derived from $p_\pi^t$ may be
positive semidefinite (PSD).
This necessitates an LSTD approximator that involves a seminorm,
hence a generalized pseudoinverse.
The main difficulty comes from the fact that the reverse product law
does not apply to pseudoinverse.
We derive the optimal (minimizing) parameter of the seminorm LSTD, and
its sampling-based estimator in \secref{sec:seminorm_lstd}.

\textbf{Second}, a system of \emph{stepwise} approximators
requires at least one parameter per timestep.
This implies an infinite number of parameters whenever there is
an infinite number of timesteps (as in infinite-horizon MDPs).
To be practical, we propose a procedure that accommodates the specification of
the desired number of approximators (hence, the desired number of parameters in a system).
It determines a number of timestep neighborhoods according to
sampling-based estimated distances among stepwise state distributions.
For every neighborhood, it then applies a seminorm LSTD that weights
the state-wise error terms using the average state distributions on that neighborhood.
This procedure is general and unifies existing LSTD-based methods,
as explained in \secref{sec:unify_lstd_based}.

\textbf{Third}, the approximators of the proposed system estimate the bias
only up to some offset; generally one unique offset for each approximator.
This is inherently due to the limitation of the bias error-function
that those approximators individually minimize.
In order to use the resulting approximations (of all approximators)
in a formula that jointly involves transient and recurrent state values,
we need to calibrate those offsets so that all approximations have
one common offset with respect to the true bias, which is time-invariant.
In \secref{sec:approxbias}, we describe such offset calibration along with
the pseudocode of the proposed system of seminorm LSTD approximators
for the bias (\ie its relative value).

We provide experimental results in \secref{nbwpval:xprmtresult}, which is
preceded by experimental setup in \secref{nbwpval:xprmtsetup}.
Finally, we conclude this work in \secref{nbwpval:conclude}, where
we also describe its limitations as well as some avenues for future research.
The next \secref{nbwpval:prelim} presents some necessary prerequisites
for this work.

%% file: prelim.tex
\section{Preliminaries} \label{nbwpval:prelim}

Given a stationary policy $\pi$, we are interested in computing
its bias value $v_b(\pi, s), \forall s \in \setname{S}$
as in \eqref{equ:bias_def}.
Since the policy and the value function type are fixed,
we often simplify their notations and write $v(s) \eqdef v_b(\pi, s)$.
This $v(s)$ is then called a state value of $s$.

A parametric state-value approximator is parameterized by a parameter vector
$\vecb{w} \in \setname{W} = \real{\dim(\vecb{w})}$.
Linear parameterization (which is the focus of this work) gives
\begin{equation} \label{equ:vhat_linear}
\hat{v}(s, \vecb{w}) \eqdef \vecb{w}^\intercal \cdot \vecb{f}(s) \approx v(s),
\quad \forall s \in \setname{S}, \forall \vecb{w} \in \setname{W},
\quad \text{equivalently,} \quad
\hat{\vecb{v}}(\vecb{w}) = \mat{F} \vecb{w},
\end{equation}
where $v(s)$ denotes the true (ground-truth) state value of $s$,
$\vecb{f}(s) \in \real{\dim(\vecb{w})}$ the state feature vector of~$s$,
and $\mat{F} \in \real{\setsize{S} \times \dim(\vecb{w})}$
the corresponding state feature matrix
(whose $s$-th row contains $\vecb{f}^\intercal(s)$).
Here, the approximate state value vector $\hat{\vecb{v}} \in \real{\setsize{S}}$
is obtained by stacking all scalar approximations
$\hat{v}(s) \in \real{} , \forall s \in \setname{S}$ on top of each other.

One way to learn $\vecb{w}$ is by minimizing
the weighted mean squared projected Bellman error (MSPBE),
denoted as $\epb(\vecb{w})$ in \eqref{equ:epb}.
This error function is derived based on the identity in
the average-reward Bellman equation, namely
\begin{equation}
v(s_t) = r(s_t) - g + \E{p(\cdot|s_t)}{v(S_{t+1})},
\quad \forall s_t \in \setname{S},\ \text{equivalently,} \quad
\vecb{v} = \vecb{r} - \vecb{g} + \mat{P}\vecb{v} \eqdefr \bo[\vecb{v}],
\label{equ:poisson_avgrew}
\end{equation}
and some projection to obtain the representation of $\bo[\vecb{v}]$
in the parameter space.
Here, the reward function $r(s_t) = \E{\pi}{r(s_t, A_t)}$ corresponds to
the reward vector $\vecb{r} \in \real{\setsize{S}}$,
the gain $g \eqdef v_g(\pi)$ corresponds to the gain vector
$\vecb{g} \eqdef g \vecb{1} \in \real{\setsize{S}}$
(using the vector $\vecb{1}$, whose entries are all 1's), and
$\mat{P} \in \real{\setsize{S} \times \setsize{S}}$ is
the one-step state transition stochastic matrix of an induced MC, whose
$s_t$-th row represents a next-state conditional distribution
$p_\pi(S_{t+1}|s_t) = \sum_{a_t \in \setname{A}} \pi(a_t|s_t) p(S_{t+1}| s_t, a_t)$.
The operator $\bo: \real{\setsize{S}} \mapsto \real{\setsize{S}}$
is termed as the Bellman policy-evaluation operator on
$\vecb{v} \in \real{\setsize{S}}$.

The MSPBE is defined as follows,
\begin{equation}
\epb(\vecb{w})
\eqdef \norm{
    \underbrace{
        \hat{\vecb{v}}(\vecb{w}) - \tilde{\po} \bo \hat{\vecb{v}}(\vecb{w})
    }_{ \Delta_{\hat{\vecb{v}}} }
}_{\tilde{\vecb{p}}}^2
= \Delta_{\hat{\vecb{v}}}^{\!\!\intercal}\ \diagptilde\ \Delta_{\hat{\vecb{v}}}
= \sum_{s \in \setname{S}} \tilde{p}(s)\ \Delta_{\hat{v}[s]}^{\!2},
\label{equ:epb}
\end{equation}
where $\tilde{\vecb{p}} \in \real{\setsize{S}}$ is a vector of probability values
of some state distribution $\tilde{p}(s) = \prob{S = s}, \forall s \in \setname{S}$,
and $\diagptilde$ is an $\setsize{S}$-by-$\setsize{S}$ diagonal matrix with
$\tilde{\vecb{p}}$ along its diagonal.
Here, $\tilde{\po}$ denotes a projection operator such that
\begin{equation} \label{equ:epb_projector}
\tilde{\po} \vecb{v} = \mat{F} \vecb{w}^\diamond,\ \text{where}\
\vecb{w}^\diamond = \argmin_{\vecb{w} \in \setname{W}}
    \Big\{ \norm{\mat{F} \vecb{w} - \vecb{v}}_{\tilde{\vecb{p}}}^2
    = \sum_{s \in \setname{S}} \tilde{p}(s)\
        ( \vecb{f}^{\intercal}(s)\ \vecb{w} - v(s) )^{2} \Big\}.
\end{equation}

At this stage, what is left to fully define $\epb$ is
the state distribution $\tilde{p}$, whose probability values serve as weights
in \eqref{equ:epb} and \eqref{equ:epb_projector}.
For recurrent MCs, one natural choice for $\tilde{p}$ is
the stationary state distribution $p^\star$ that indicates the state visitation
frequency in the long-run.
More precisely,
\begin{equation} \label{equ:pstar_lim}
p^\star(s) = \mathbb{E}_{S_0 \sim \isd}
\Big[ p^\star(s|s_0)
\eqdef \lim_{\tmax \to \infty} \frac{1}{\tmax} \sum_{t=0}^{\tmax - 1} p^t(s | s_0)
= \underbrace{
        \lim_{\tmax \to \infty} p^{\tmax}(s | s_0)
    }_\text{when the MC is aperiodic}
\Big],
\quad \forall s \in \setname{S},
\end{equation}
where $p^\star(s|s_0)$ is the limiting distribution of the stepwise $p^t(s|s_0)$
as $t$ goes to infinity
(nonetheless, $p^\star$ may be achieved in finite time).
Since all states of a recurrent MC are recurrent, its $p^\star$ has
the whole state set as its support, \ie $p^\star(s) > 0, \forall s \in \setname{S}$.
This is advantageous because $\diagpstar$ is positive definite (PD) so that
\eqref{equ:epb} and \eqref{equ:epb_projector} involve
a (weighted Euclidean) norm, and the inverse $\diagpstar^{-1}$ exists.

\begin{assumption} \label{assume:indep_fea}
The state feature matrix $\mat{F}$ has a full column rank.
This is equivalent to saying that all state feature vectors are linearly independent.
\emph{(Remark: this assumption is not required by our proposed seminorm LSTD
in \secref{sec:seminorm_lstd}.)}
\end{assumption}

In fact, setting $\tilde{\vecb{p}} \gets \vecb{p}^\star$
leads to the LSTD method for recurrent MDPs \citep{yu_2009_lspe}.
Whenever \assref{assume:indep_fea} is satisfied, the projection operator
in \eqref{equ:epb} is defined as
$\po \eqdef
\mat{F} (\mat{F}^\intercal \diagpstar \mat{F})^{-1}
\mat{F}^\intercal \diagpstar$.
Then, the optimal parameter value (which minimizes $\epb$) is given by
\begin{equation}
\vecb{w}^* = \mat{X}^{-1} \vecb{y},
\qquad \text{where $\vecb{w}^* \in \setname{W}$ and}
\label{equ:td_fixedpoint}
\end{equation}
\begin{align}
\mat{X}
& = \sum_{s \in \setname{S}} p^\star(s) \sum_{s' \in \setname{S}} p(s'| s)
    \Big[ \vecb{f}(s)  \Big( \vecb{f}(s) - \vecb{f}(s') \Big)^{\!\!\intercal} \Big]
    = \sum_{s \in \setname{S}} p^\star(s) \vecb{f}(s)
    \Big[ \Big(
        \vecb{f}(s) - \sum_{s' \in \setname{S}} p(s'| s) \vecb{f}(s')
        \Big)^{\!\!\intercal} \Big] \notag \\
& = \mat{F}^\intercal \diagpstar (\mat{I} - \mat{P}) \mat{F}
    \qquad \in \real{\dim(\vecb{w}) \times \dim(\vecb{w})},\
    \text{with an identity matrix
        $\mat{I} \in \real{\setsize{S} \times \setsize{S}}$, and}
    \label{equ:x_sampling} \\
\vecb{y}
& = \sum_{s \in \setname{S}} p^\star(s) \sum_{a \in \setname{A}} \pi(a|s)
    \Big[ (r(s, a) - g) \vecb{f}(s) \Big]
    = \sum_{s \in \setname{S}} p^\star(s) \Big[ (r(s) - g) \vecb{f}(s) \Big]
    \notag \\
& = \mat{F}^\intercal \diagpstar (\vecb{r} - \vecb{g})
    \qquad \in \real{\dim(\vecb{w})}. \label{equ:y_sampling}
\end{align}
The minimizer $\vecb{w}^*$ in \eqref{equ:td_fixedpoint} involves $\mat{X}^{-1}$,
which exists whenever \assref{assume:indep_fea} is satisfied, and
the singularity of $(\mat{I} - \mat{P})$ is remedied.
For example, by introducing an eligibility factor $\lambda \in (0, 1)$ such that
$\mat{X}$ then involves $(\mat{I} - \mat{P}^{(\lambda)})$, where
$\mat{P}^{(\lambda)}
\eqdef (1 - \lambda) \sum_{\tau=0}^\infty \lambda^\tau \mat{P}^{\tau+1}$.
Another technique is replacing $\mat{X}$ altogether with its non-singular approximation
by some perturbation \citep[\lmm{7}, \cor{1}]{tsitsiklis_1999_avgtd}.
Note that $(\mat{I} - \mat{P})$ is not invertible \citep[\page{596}]{puterman_1994_mdp}.

An LSTD-based method approximately computes the minimizer $\vecb{w}^*$
in \eqref{equ:td_fixedpoint} by the sample means of $\mat{X}$ and $\vecb{y}$
according to \eqref{equ:x_sampling} and \eqref{equ:y_sampling}, respectively.
That is,
\begin{equation}
\hat{\vecb{w}}^*
= \hat{\mat{X}}^{-1} \hat{\vecb{y}}
= \Big( \frac{1}{\nsample} \sum_{i=1}^{\nsample}
    \vecb{f}(s_i) ( \vecb{f}(s_i) - \vecb{f}(s_i'))^{\intercal} \Big)^{-1}
   \Big( \frac{1}{\nsample} \sum_{i=1}^{\nsample} (r_i - g) \vecb{f}(s_i) \Big),
\label{equ:norm_lstd_sampling}
\end{equation}
where $\nsample$ denotes the number of
state $s_i$, next state $s_i'$, and reward $r_i$ samples,
which are collected by the agent through interaction with its environment.
Typically in practice, $\hat{\mat{X}} \gets \hat{\mat{X}} + \epsilon \mat{I}$ for
some small positive $\epsilon > 0$ in order to ensure
the approximation matrix $\hat{\mat{X}}$ is invertible.

One interesting property of LSTD based on $p^\star$ is
that its minimizer \eqref{equ:td_fixedpoint} is also the solution of
the semi-gradient TD method for recurrent MDPs \citep{tsitsiklis_1999_avgtd}.
This method minimizes the weighted mean squared error (MSE) as follows,
\begin{equation}
\ems(\vecb{w})
\eqdef \sum_{s \in \setname{S}} \tilde{p}(s) [ v(s) - \hat{v}(s; \vecb{w}) ]^2
= \norm{\vecb{v} - \hat{\vecb{v}}(\vecb{w})}_{\tilde{\vecb{p}}}^2
= [\vecb{v} - \hat{\vecb{v}}(\vecb{w})]^\intercal\ \diagptilde\
    [\vecb{v} - \hat{\vecb{v}}(\vecb{w})],
\label{equ:ems}
\end{equation}
where $\vecb{v}$ denotes the true (ground-truth) value, while
$\tilde{\vecb{p}} \gets \vecb{p}^\star$ and $\diagptilde$ are
the state distribution and the corresponding diagonal matrix, respectively.
The semi-gradient TD method follows the stochastic gradient descent (SGD) for
updating its parameter $\vecb{w}$.
For linear pameterization \eqref{equ:vhat_linear} such that
$\nabla \hat{v}(S; \vecb{w}) = \vecb{f}(S)$, the SGD update rule for
an approximation iterate $\hat{\vecb{w}}^* \approx \vecb{w}^*$
is given by
\begin{equation}
\underbrace{
    \hat{\vecb{w}}^* \gets  \hat{\vecb{w}}^* - \alpha \hat{\nabla} \ems(\hat{\vecb{w}}^*),
}_\text{\cf the LSTD estimator in \eqref{equ:norm_lstd_sampling}}
\quad \text{with}\
\hat{\nabla} \ems(\hat{\vecb{w}}^*)
\eqdef - \Big(
    \underbrace{
        r(s) - g + \hat{v}(s'; \hat{\vecb{w}}^*)
    }_\text{$\approx v(s)$ based on \eqref{equ:poisson_avgrew}}
    - \hat{v}(s; \hat{\vecb{w}}^*) \Big) \vecb{f}(s),
\label{equ:semigrad_update}
\end{equation}
where $\alpha$ is some positive learning rate and $\hat{\nabla} \ems$ is
the stochastic estimate of the gradient of $\ems$ \eqref{equ:ems} by
one current state $s$ and one next state $s'$ sampled from
$p^\star$ and $p(\cdot|s)$, respectively.
Since the true $v(s)$ is unknown in RL, an approximation is substituted for it
in \eqref{equ:semigrad_update} only after taking the gradient
(hence, the term semi-gradient\footnote{
    In contrast, LSTD methods are based on (true) gradients of
    the error function $\epb$.
    This is possible since $\nabla \epb$ does not involve the true value $v(s)$,
    see \eqref{equ:epb_grad}.
}).
Such approximation is based on the Bellman equation \eqref{equ:poisson_avgrew}.
It can be shown that the approximation iterate
$\hat{\vecb{w}}^*$ \eqref{equ:semigrad_update} converges to the LSTD's minimizer
$\vecb{w}^*$ in \eqref{equ:td_fixedpoint},
for which $\vecb{w}^*$ is called the TD fixed point \citep[\page{206}]{sutton_2018_irl}.
Note that a semi-gradient TD method needs the specification of
the learning rate $\alpha$ and the initial value for $\hat{\vecb{w}}^*$.

%% file: seminormlstd.tex
\section{Seminorm LSTD approximators} \label{sec:seminorm_lstd}

In this section, we present an LSTD approximator that minimizes $\epb$
\eqref{equ:epb}, whose state distribution $\tilde{p}$ does not necessarily
have the whole state set as its support,
\ie $\tilde{p}(s) \ge 0, \forall s \in \setname{S}$.
This $\tilde{p}$ induces a positive semidefinite (PSD) diagonal matrix $\diagptilde$,
hence a seminorm $\epb$.
Consequently, minimizing $\epb$ and deriving its projector $\tilde{\po}$
\eqref{equ:epb_projector} require solving $\diagptilde$-seminorm LS problems.
We call the corresponding state-value approximator based on such $\epb$
as \emph{a seminorm LSTD}.

A seminorm LSTD is useful for unichain MCs with multiple transient and
recurrent states (and with certain reward structures).
For example, since each state type has different timing
(transient states are visited at the outset before absorption, whereas
recurrent states in the long-run),
a proper $\tilde{p}$ is different for each type so that
the support of a proper type-specific $\tilde{p}$ only contains a subset of
the state set; inducing a seminorm $\epb$.
It is proper in that it provides reasonable weighting for the state-wise error terms
in \eqref{equ:epb} and \eqref{equ:epb_projector}, and
that it enables state sampling in \eqref{equ:x_sampling} and \eqref{equ:y_sampling}.
More importantly, a seminorm LSTD facilitates the derivation of
a general approximation procedure (see \secref{sec:unify_lstd_based}).

The main result of this Section is a sampling-enabler expression for the minimizer
of $\epb$ of a seminorm LSTD.
It is presented in \thmref{thm:epb_minimizer} (\secref{sec:samplingenabler_minimizer}).
For that, the preceding \secref{sec:seminorm_part} contains
the projection operator for the seminorm $\epb$ and two necessary lemmas for the minimizer.

\input{seminormlstd_part.tex}

\input{seminormlstd_minimizer.tex}

%% file: seminormlstd_part.tex
\subsection{Necessary components for the error function and the minimizer}
\label{sec:seminorm_part}

We begin with the projection operator $\tilde{\po}$ that involves
the $\diagptilde$-seminorm.
It is stated in \lmmref{lmm:projector} below.
Recall that $\tilde{\po}$ projects any value $\vecb{v}$ onto the space of
representable parameterized approximators.
To proceed, we need the following \defref{def:mp_pinv}.
\begin{definition} \label{def:mp_pinv}
Given a matrix $\mat{A} \in \real{m \times n}$,
then its Moore-Penrose pseudoinverse $\mat{A}^\dagger \in \real{n \times m}$
is the unique matrix such that
(i) $\mat{A} \mat{A}^\dagger \mat{A} =\mat{A}$,
(ii) $\mat{A}^\dagger \mat{A} \mat{A}^\dagger =\mat{A}^\dagger$,
(iii) $(\mat{A} \mat{A}^\dagger)^{\!\intercal} = \mat{A} \mat{A}^\dagger$, and
(iv) $(\mat{A}^\dagger \mat{A})^{\!\intercal} = \mat{A}^\dagger \mat{A}$.
See \citet[\deff{1.1.3}]{campbell_2009_ginv}.
\end{definition}

\input{seminormlstd_part_projector}

The next \lmmref{lmm:z_zpinv_psd} describes the relevant properties of
matrix $\mat{Z}$, which emerges during the foregoing derivation of $\tilde{\po}$.
This is essential because $\mat{Z}$ and its pseudoinverse $\mat{Z}^\dagger$
(along with its matrix square root) play an important role in the derivation of
the minimizer of $\epb$ (see \thmref{thm:epb_minimizer}).
\input{seminormlstd_part_zprop}

%% file: seminormlstd_part_projector.tex
\begin{lemma} \label{lmm:projector}
The projection operator $\tilde{\po}$ involving the $\diagptilde$-seminorm
is given by
\begin{equation*}
\tilde{\po} = \mat{F} \mat{Z}^\dagger \mat{F}^\intercal \diagptilde,
\quad \text{where}\
\mat{Z} \eqdef \mat{F}^\intercal \diagptilde \mat{F}
= \E{S \sim \tilde{p}}{\vecb{f}(S) \vecb{f}(S)^\intercal}
\in \real{\dim(\vecb{w}) \times \dim(\vecb{w})}.
\end{equation*}
Here, the state distribution $\tilde{p}$ may have zero probabilities for some states,
\ie $\tilde{p}(s) \ge 0, \forall s \in \setname{S}$.
The superscript $\dagger$ indicates the Moore-Penrose pseudoinverse
(\defref{def:mp_pinv}).
\end{lemma}
\begin{proof}
The projection operator $\tilde{\po}$ is a matrix that satisfies
$\tilde{\po} \vecb{v} = \mat{F} \vecb{w}^\diamond$, where
\begin{align*}
\vecb{w}^\diamond
& = \argmin_{\vecb{w} \in \setname{W}} \Big\{
    \norm{ \{\hat{\vecb{v}} = \mat{F} \vecb{w} \} - \vecb{v}}_{\diagptilde}^2
    = \sum_{s \in \setname{S}} [ \tilde{p}^{\frac{1}{2}}(s) ]^2\
        [ \hat{v}(s; \vecb{w}) - v(s) ]^2
    = \norm{ \diagptildesqrt (\mat{F} \vecb{w} - \vecb{v}) }_{2}^2
\Big\}.
\end{align*}
Finding $\vecb{w}^\diamond$ amounts to solving for
\begin{align}
\text{the $\diagptilde$-seminorm LS solutions of}\quad
    \mat{F} \vecb{w} &= \vecb{v}, \text {or equivalently,} \notag \\
\text{the Euclidean-norm LS solutions of}\quad
    \diagptildesqrt \mat{F} \vecb{w} & = \diagptildesqrt \vecb{v}.
    \label{equ:euclid_ls}
\end{align}
The latter has the following general form \citep[\page{106}]{ben_2003_ginv},
\begin{align}
\vecb{w}^\diamond
& = [\diagptildesqrt \mat{F}]^\dagger \diagptildesqrt \vecb{v}
    + [\mat{I} - (\diagptildesqrt \mat{F})^\dagger (\diagptildesqrt \mat{F})] \vecb{c} \notag \\
& = \underbrace{
        [(\diagptildesqrt \mat{F})^\intercal \diagptildesqrt \mat{F}]^\dagger
        (\diagptildesqrt \mat{F})^\intercal \diagptildesqrt \vecb{v}
    }_{[\mat{F}^\intercal \diagptilde \mat{F}]^\dagger \mat{F}^\intercal \diagptilde \vecb{v}}
    + [\mat{I} - ((\diagptildesqrt \mat{F})^\intercal \diagptildesqrt \mat{F})^\dagger
        ((\diagptildesqrt \mat{F})^\intercal \diagptildesqrt \mat{F})] \vecb{c},
\label{equ:general_sol}
\end{align}
for an arbitrary vector $\vecb{c} \in \real{\dim(\vecb{w})}$.
Since the gradient at $\vecb{w}^\diamond$ vanishes (a necessary condition for the minimizer),
it can be shown that $\vecb{w}^\diamond$ is also the solution of the normal equation of
\eqref{equ:euclid_ls} as in \citep[\thm{2.1.2}]{campbell_2009_ginv}. %
That is,
\begin{align*}
\nabla \norm{ \mat{F} \vecb{w} - \vecb{v} }_{\diagptilde}^2
= 2 \mat{F}^\intercal \diagptilde (\mat{F} \vecb{w}^\diamond - \vecb{v})
    & = \vecb{0} \tag{Whenever $\vecb{w} = \vecb{w}^\diamond$} \\
\Longleftrightarrow
    \mat{F}^\intercal \diagptilde \mat{F} \vecb{w}^\diamond
    & =  \mat{F}^\intercal \diagptilde \vecb{v}
    \tag{\cf \eqref{equ:general_sol}} \\
\Longleftrightarrow
    (\diagptildesqrt \mat{F})^\intercal \diagptildesqrt \mat{F} \vecb{w}^\diamond
    & = (\diagptildesqrt \mat{F})^\intercal \diagptildesqrt \vecb{v}.
\tag{The normal equation of \eqref{equ:euclid_ls}}
\end{align*}
By setting $\vecb{c}$ to zero in \eqref{equ:general_sol},
we obtain one $\diagptilde$-seminorm LS solution, denoted as $\tilde{\vecb{w}}^\diamond$.
The projection then takes the form of
\begin{equation*}
\tilde{\po} \vecb{v} = \mat{F} \tilde{\vecb{w}}^\diamond
= \mat{F} \{ [\mat{F}^\intercal \diagptilde \mat{F}]^\dagger
    \mat{F}^\intercal \diagptilde \vecb{v} \}
= \mat{F} \{ \mat{Z}^\dagger \mat{F}^\intercal \diagptilde \vecb{v} \}, \quad
\text{hence,}\
\tilde{\po} = \mat{F} \mat{Z}^\dagger \mat{F}^\intercal \diagptilde.
\end{equation*}
Note that this $\tilde{\vecb{w}}^\diamond$ is not
the minimal $\diagptilde$-seminorm $\diagptilde$-LS solution, \ie
$\tilde{\vecb{w}}^\diamond
\ne \argmin_{\vecb{w}^\diamond} \norm{\vecb{w}^\diamond}_{\diagptilde}$,
see \citet[\thm{2}]{proszynski_1995_snls}.
This concludes the proof.
\end{proof}

%% file: seminormlstd_part_zprop.tex
\begin{lemma} \label{lmm:z_zpinv_psd}
These real matrices $\mat{Z}$, $\mat{Z}^\dagger$, and $\mat{Z}^{\nicefrac{\dagger}{2}}$
are symmetric positive semidefinite (PSD).
Here, $\mat{Z} \eqdef \mat{F}^\intercal \diagptilde \mat{F}$, and
$\mat{Z}^{\nicefrac{\dagger}{2}}
\eqdef (\mat{Z}^\dagger)^{\frac{1}{2}} = \sqrt{\mat{Z}^\dagger}$,
which is the matrix square root of $\mat{Z}^\dagger$.
\end{lemma}
\begin{proof}
First, $\mat{Z}$ involves a PSD diagonal matrix
$\diagptilde = (\diagptildesqrt)^2$ with its unique matrix square root $\diagptildesqrt$.
Let $\mat{G} \eqdef \diagptildesqrt \mat{F}$.
Expressing $\mat{Z}$ as a Gram matrix gives
$\mat{Z}^{\intercal}
= (\mat{G}^\intercal \mat{G})^\intercal = \mat{G}^\intercal \mat{G} = \mat{Z}$,
which shows that $\mat{Z}$ is symmetric.
Moreover,
\begin{equation}
\vecb{u}^\intercal \mat{Z} \vecb{u}
= \vecb{u}^\intercal \mat{F}^\intercal \diagptildesqrt \diagptildesqrt \mat{F} \vecb{u}
= (\diagptildesqrt \mat{F} \vecb{u})^\intercal (\diagptildesqrt \mat{F} \vecb{u})
= \norm{\diagptildesqrt \mat{F} \vecb{u}}_2^2 \ge 0,
\quad \forall \vecb{u} \in \real{\dim(\vecb{f})}.
\label{equ:z_psd}
\end{equation}
Hence, $\mat{Z}$ is symmetric positive semidefinite (PSD).

Second, let the singular value decomposition (SVD) of $\mat{Z}$ is
given by $\mat{M} \diagsvd \mat{N}^{\intercal}$.
Then, we have
\begin{align*}
\mat{Z} \mat{Z}^{\intercal}
& = \mat{M} \diagsvd \mat{N}^{\intercal} \mat{N} \diagsvd \mat{M}^{\intercal}
  = \mat{M} \diagsvds \mat{M}^{\intercal},
  \tag{Since $\mat{N}$ is orthogonal} \\
\mat{Z}^{\intercal} \mat{Z}
& = \mat{N} \diagsvd \mat{M}^{\intercal} \mat{M} \diagsvd \mat{N}^{\intercal}
  = \mat{N} \diagsvds \mat{N}^{\intercal}
  \tag{Since $\mat{M}$ is orthogonal}.
\end{align*}
Because $\mat{Z}$ is symmetric (hence, normal), we have
$\mat{Z} \mat{Z}^{\intercal} = \mat{Z}^{\intercal} \mat{Z} = \mat{Z}^2$.
Thus, $\mat{M} = \mat{N}$ whose columns are the orthogonal eigenvectors of $\mat{Z}^2$,
which are then normalized to become unit vectors in order to have
an orthogonal matrix $\mat{M}$.
By SVD, the singular value diagonal matrix $\diagsvd$ contains
the squared roots of eigenvalues of $\mat{Z}^2$.

Let $\mu$ be the eigenvalue of $\mat{Z}$ with eigenvector $\vecb{u}$, then
\begin{equation*}
\mat{Z} \vecb{u} = \mu \vecb{u},\quad \text{and}\quad
\mat{Z}^2 \vecb{u} = \mat{Z} (\mat{Z} \vecb{u}) = \mat{Z} (\mu \vecb{u})
    = \mu (\mat{Z} \vecb{u}) = \mu (\mu \vecb{u}) = \mu^2 \vecb{u},
\end{equation*}
which shows that $\vecb{u}$ is an eigenvector of $\mat{Z}^2$
with the eigenvalue $\mu^2$.
This holds for all eigenvalues of~$\mat{Z}$, which become the diagonal entries
of $\diagsvd$ (such eigenvalues are non-negative since $\mat{Z}$ is PSD).
Moreover, because $\mat{Z}$ is symmetric,
both $\mat{Z}$ and $\mat{Z}^2$ have the same set of orthogonal eigenvectors,
which becomes the columns of $\mat{M}$.
Thus, the eigen (spectral) decomposition (EigD) of $\mat{Z}$, namely
$\mat{U} \diageig \mat{U}^\intercal$, is also a valid SVD.
Consequently,
\begin{equation} \label{equ:zpinv_breakdown}
\mat{Z}^\dagger
= (\underbrace{\mat{M} \diagsvd \mat{N}^{\intercal}}_\text{SVD of $\mat{Z}$})^\dagger
= (\underbrace{\mat{U} \diageig \mat{U}^\intercal}_\text{EigD of $\mat{Z}$})^\dagger
= \mat{U} \diageig^\dagger \mat{U}^\intercal
= \mat{U} \diageig^{\nicefrac{\dagger}{2}} \diageig^{\nicefrac{\dagger}{2}} \mat{U}^\intercal
= (\diageig^{\nicefrac{\dagger}{2}} \mat{U}^\intercal)^\intercal
    (\diageig^{\nicefrac{\dagger}{2}} \mat{U}^\intercal),
\end{equation}
where $\diageig^\dagger$ is obtained by taking the reciprocal of non-zeroes entries
of $\diageig$.
Thus, $\mat{Z}^\dagger$ can be expressed as a Gram matrix,
which is always PSD as shown before in \eqref{equ:z_psd}.
Since $\mat{Z}^\dagger$ is PSD, there exists exactly one (symmetric) PSD matrix
$\zpinvsqrt$ such that $\mat{Z}^\dagger = \zpinvsqrt \zpinvsqrt$.
From \eqref{equ:zpinv_breakdown} above, we have
$\zpinvsqrt = \diageig^{\nicefrac{\dagger}{2}} \mat{U}^\intercal$.
This concludes the proof, whose alternatives can be found in
\citetext{\citealp[\cor{3}]{lewis_1968_psd}; \citealp[\thm{20.5.3}]{harville_1997_mat}}.
\end{proof}

%% file: seminormlstd_minimizer.tex
\subsection{The sampling-enabler expression for the minimizer}
\label{sec:samplingenabler_minimizer}

The core component of a seminorm LSTD is the minimizer $\tilde{\vecb{w}}^*$
of its error $\epb$, which involves a PSD matrix $\diagptilde$.
In particular for model-free RL, we need an expression of $\tilde{\vecb{w}}^*$
that enables sampling based approximation, akin to \eqref{equ:td_fixedpoint}.
By utilizing \lmmrefand{lmm:projector}{lmm:z_zpinv_psd} from
the previous \secref{sec:seminorm_part},
we are now ready to derive such a sampling-enabler expression.
It is stated in the following \thmref{thm:epb_minimizer}.
\input{seminormlstd_minimizer_thm}

Pertaining to \thmref{thm:epb_minimizer}, we remark that the formula simplication
in \eqref{equ:wstar_seminorm} is crucial because
the resulting expression enables unbiased sampling-based estimation
for the minimizer $\tilde{\vecb{w}}^*$ in model-free RL.
This is possible because the last expression in \eqref{equ:wstar_seminorm}
involves only one factor of $\mat{X}$.
In contrast, the expression before simplification has three factors of $\mat{X}$,
hence it does not enable such unbiased estimation for $\tilde{\vecb{w}}^*$.
The reason stems from the fact that $\mat{X}$ depends on the next-state random variable
through $\E{p(\cdot|s)}{\vecb{f}(S')}$, which
leads to a similar situation as described by \citet[\page{272}]{sutton_2018_irl}.
They explain that multiple independent samples of next states are required
to obtain an unbiased estimate of the product of multiple factors that
involve expectations of next states.
Such independent next-state samples are only available in deterministic transition
(where the next state is not random),
or in simulation where the agent can roll-back from any state to its previous state.
This sampling requirement cannot be accommodated in model-free RL settings since
the agent cannot roll-back to its previous state and
transitions are generally stochastic.

In addition, we also remark that the simplication in \eqref{equ:wstar_seminorm}
is carried out without introducing any error
(putting aside errors due to numerical computation).
In comparison, simplifying $\tilde{\vecb{w}}^*$ (to involve only one factor of $\mat{X}$)
through the reverse order law for the pseudoinverse is possible but with some errors
because the identity
$(\mat{X}^\intercal \mat{Z}^\dagger \mat{X})^\dagger
= \mat{X}^\dagger \mat{Z} \mat{X}^{\intercal \dagger}$
requires strict conditions \citep{hartwig_1986_rev, tian_2019_rev}.
Two example simplifications with errors are as follows,
\begin{itemize}
\item by orthogonal approximation $\mat{X}_{\!\!\perp} \approx \mat{X}$
    (where $\mat{X}_{\!\!\perp}^\dagger = \mat{X}_{\!\!\perp}^{-1} = \mat{X}_{\!\!\perp}^\intercal$)
    and the identity
    $(\mat{X}_{\!\!\perp}^\intercal \mat{Z}^\dagger \mat{X}_{\!\!\perp})^\dagger
    = \mat{X}_{\!\!\perp}^\intercal \mat{Z} \mat{X}_{\!\!\perp}$
    \citep[Theorem 1.2.1: 7]{campbell_2009_ginv}
    such that
\begin{equation*}
\tilde{\vecb{w}}^*
= (\mat{X}^\intercal \mat{Z}^\dagger \mat{X})^\dagger
    \mat{X}^\intercal \mat{Z}^\dagger \vecb{y}
\approx (\mat{X}_{\!\!\perp}^\intercal \mat{Z}^\dagger \mat{X}_{\!\!\perp})^\dagger
    \mat{X}_{\!\!\perp}^\intercal \mat{Z}^\dagger \vecb{y}
= \mat{X}_{\!\!\perp}^\intercal \mat{Z} \mat{Z}^\dagger \vecb{y},
\end{equation*}

\item by nullifying the constant matrices $\mat{C}_1$ and $\mat{C}_2$
    (in below expression) such that
\begin{align*}
\tilde{\vecb{w}}^*
& = (\mat{X}^\intercal \mat{Z}^\dagger \mat{X})^\dagger
    \mat{X}^\intercal \mat{Z}^\dagger \vecb{y}
= (\mat{X}^\dagger \mat{Z} \mat{X}^{\intercal \dagger} + \mat{C}_{\!1})
    \mat{X}^\intercal \mat{Z}^\dagger \vecb{y} \\
& \approx \mat{X}^\dagger \mat{Z} \mat{X}^{\intercal \dagger}
    \mat{X}^\intercal \mat{Z}^\dagger \vecb{y}
= \mat{X}^\dagger \mat{Z} (\mat{I} + \mat{C}_{\!2}) \mat{Z}^\dagger \vecb{y}
\approx \mat{X}^\dagger \mat{Z} \mat{Z}^\dagger \vecb{y}.
\end{align*}
\end{itemize}

Finally, we present the proof for \thmref{thm:epb_minimizer} about
the minimizer $\tilde{\vecb{w}}^*$ below.
\input{seminormlstd_minimizer_thmproof}

%% file: seminormlstd_minimizer_thm.tex
\begin{theorem} \label{thm:epb_minimizer}

One minimizer of the error $\epb(\vecb{w})$ in \eqref{equ:epb}, which involves
a state distribution $\tilde{p}$ with
$\tilde{p}(s) \ge 0, \forall s \in \setname{S}$
(hence, $\epb$ is a seminorm with a PSD diagonal matrix $\diagptilde$),
is given by
\begin{equation} \label{equ:wstar_seminorm}
\tilde{\vecb{w}}^*
= (\mat{X}^\intercal \mat{Z}^\dagger \mat{X})^\dagger
    \mat{X}^\intercal \mat{Z}^\dagger \vecb{y}
= (\mat{Z}^{\nicefrac{\dagger}{2}} \mat{X})^\dagger
    \mat{Z}^{\nicefrac{\dagger}{2}} \vecb{y},
\qquad \text{where}
\end{equation}
\begin{align*}
\mat{X}
& \eqdef \E{S \sim \tilde{p}, S' \sim p(\cdot| s)}{
    \vecb{f}(S) \big( \vecb{f}(S) - \vecb{f}(S') \big)^{\!\intercal}},
    \tag{with state feature $\vecb{f}(s)$ and one-step transition $p$} \\
\mat{Z}
& \eqdef \E{S \sim \tilde{p}}{\vecb{f}(S) \vecb{f}(S)^\intercal}, \quad \text{and}
    \tag{with state feature $\vecb{f}(s)$ as above} \\
\vecb{y}
& \eqdef \E{S \sim \tilde{p}}{\big( r(S) - g \big) \vecb{f}(S)}.
    \tag{with state reward $r(s)$, gain $g$ and $\vecb{f}(s)$ as above}
\end{align*}

Here, $p(\cdot| s)$ denotes the next state conditional distribution
(\ie the one-step state transition distribution given the current state $s$).
Note that we abuse the notations $\mat{X}$ and $\vecb{y}$, which are also used
in \eqref{equ:x_sampling} and \eqref{equ:y_sampling} but with
a different state distribution.

\end{theorem}

%% file: seminormlstd_minimizer_thmproof.tex
\begin{proof} \label{thm:epb_minimizer_proof}
(of \thmref{thm:epb_minimizer})
The MSBPE error $\epb$ in \eqref{equ:epb} can be expressed as follows,
\begin{align*}
\epb(\vecb{w})
& = \norm{\hat{\vecb{v}} - \tilde{\po} \bo \hat{\vecb{v}}}_{\diagptilde}^2
    = \norm{\tilde{\po} \hat{\vecb{v}} - \tilde{\po} \bo \hat{\vecb{v}}}_{\diagptilde}^2
    = \norm{\tilde{\po} [ \hat{\vecb{v}} - \bo \hat{\vecb{v}} ] }_{\diagptilde}^2
    \tag{Recall $\hat{\vecb{v}}$ is representable} \\
& = \{\tilde{\po} [ \hat{\vecb{v}} - \bo \hat{\vecb{v}} ] \}^\intercal \diagptilde
    \{\tilde{\po} [ \hat{\vecb{v}} - \bo \hat{\vecb{v}} ] \}
    = [ \hat{\vecb{v}} - \bo \hat{\vecb{v}} ]^\intercal
    \{\tilde{\po}^\intercal \diagptilde \tilde{\po} \}
    [ \hat{\vecb{v}} - \bo \hat{\vecb{v}} ] \notag \\
& = [ \hat{\vecb{v}} - \bo \hat{\vecb{v}} ]^\intercal
    \{ \mat{F} \mat{Z}^\dagger \mat{F}^\intercal \diagptilde \}^\intercal
    \diagptilde \{ \mat{F} \mat{Z}^\dagger \mat{F}^\intercal \diagptilde \}
    [ \hat{\vecb{v}} - \bo \hat{\vecb{v}} ]
    \tag{Expand $\tilde{\po}$ from \lmmref{lmm:projector}} \\
& = [ \hat{\vecb{v}} - \bo \hat{\vecb{v}} ]^\intercal
    \{ \diagptilde \mat{F} \mat{Z}^\dagger \mat{F}^\intercal  \}
    \diagptilde
    \{ \mat{F} \mat{Z}^\dagger \mat{F}^\intercal \diagptilde \}
    [ \hat{\vecb{v}} - \bo \hat{\vecb{v}} ]
    \tag{$\mat{Z}^\dagger$ is symmetric (\lmmref{lmm:z_zpinv_psd})} \\
& = \{ \mat{F}^\intercal \diagptilde [ \hat{\vecb{v}} - \bo \hat{\vecb{v}} ] \}^\intercal
    \mat{Z}^\dagger (\mat{F}^\intercal \diagptilde \mat{F}) \mat{Z}^\dagger
    \{ \mat{F}^\intercal \diagptilde
    [ \hat{\vecb{v}} - \bo \hat{\vecb{v}} ] \} \notag \\
& = \{ \mat{F}^\intercal \diagptilde [ \hat{\vecb{v}} - \bo \hat{\vecb{v}} ] \}^\intercal
    \mat{Z}^\dagger \{ \mat{F}^\intercal \diagptilde
        [ \hat{\vecb{v}} - \bo \hat{\vecb{v}} ] \}
    \tag{Apply the condition (ii) in \defref{def:mp_pinv}} \\
& = \norm{\mat{F}^\intercal \diagptilde [ \hat{\vecb{v}} - \bo \hat{\vecb{v}} ]
    }_{\mat{Z}^\dagger}^2
    \tag{$\mat{Z}^\dagger$ is PSD (\lmmref{lmm:z_zpinv_psd}),
        hence $\mat{Z}^\dagger$-seminorm} \\
& = \norm{ \mat{F}^\intercal \diagptilde [ \mat{F} \vecb{w} - \mat{P} \mat{F} \vecb{w}
    - (\vecb{r} - \vecb{g}) ] }_{\mat{Z}^\dagger}^2
    \tag{Expand $\hat{\vecb{v}}$ from \eqref{equ:vhat_linear}
        and $\bo$ from \eqref{equ:poisson_avgrew}}\\
& = \norm{\mat{F}^\intercal \diagptilde \mat{F} \vecb{w}
        - \mat{F}^\intercal \diagptilde \mat{P} \mat{F} \vecb{w}
        - \mat{F}^\intercal \diagptilde (\vecb{r} - \vecb{g})
    }_{\mat{Z}^\dagger}^2 \\
& = \norm{
        \underbrace{
            \mat{F}^\intercal \diagptilde (\mat{I} - \mat{P}) \mat{F}
        }_{\mat{X}} \vecb{w}
        - \underbrace{\mat{F}^\intercal \diagptilde (\vecb{r} - \vecb{g})}_{\vecb{y}}
    }_{\mat{Z}^\dagger}^2
    = \norm{\mat{X} \vecb{w} - \vecb{y}}_{\mat{Z}^\dagger}^2.
\end{align*}
The above steps are inspired by \citet[\app{A}]{dann_2014_petd} who derived
the (norm) LSTD based on the stationary state distribution $p^\star$
for the discounted-reward value function for recurrent MDPs.

Therefore, minimizing $\epb$ (which is a seminorm as $\diagptilde$ is PSD)
amounts to solving for
\begin{align*}
\text{the $\mat{Z}^\dagger$-seminorm LS solutions of}\quad
    \mat{X} \vecb{w} &= \vecb{y}, \quad \text{or equivalently,} \\
\text{the Euclidean-norm LS solutions of}\quad
    \mat{Z}^{\nicefrac{\dagger}{2}} \mat{X} \vecb{w}
    & = \mat{Z}^{\nicefrac{\dagger}{2}} \vecb{y}
    \tag{Similar to \eqref{equ:euclid_ls}}.
\end{align*}
Taking the gradient of $\epb$ and setting it to zero for
a minimizer $\tilde{\vecb{w}}^*$ gives
\begin{equation}
\nabla \epb(\vecb{w})
= \nabla \norm{\mat{X} \vecb{w} - \vecb{y}}_{\mat{Z}^\dagger}^2
= 2\mat{X}^\intercal \mat{Z}^\dagger (\mat{X} \vecb{w} - \vecb{y})
\eqset \vecb{0}
\quad \Longleftrightarrow \quad
\mat{X}^\intercal \mat{Z}^\dagger \mat{X} \tilde{\vecb{w}}^*
= \mat{X}^\intercal \mat{Z}^\dagger \vecb{y}.
\label{equ:epb_grad}
\end{equation}
In a similar fashion as the derivation of $\tilde{\po}$ (\lmmref{lmm:projector}),
one solution for \eqref{equ:epb_grad} is given by
\begin{align*}
\tilde{\vecb{w}}^*
& = (\mat{X}^\intercal \mat{Z}^\dagger \mat{X})^\dagger
    \mat{X}^\intercal \mat{Z}^\dagger \vecb{y}
= (\mat{X}^\intercal \mat{Z}^{\nicefrac{\dagger}{2}} \mat{Z}^{\nicefrac{\dagger}{2}}
    \mat{X})^\dagger
    \mat{X}^\intercal \mat{Z}^{\nicefrac{\dagger}{2}} \mat{Z}^{\nicefrac{\dagger}{2}}
    \vecb{y}
    \tag{See \eqref{equ:zpinv_breakdown}} \\
& = (\mat{L}^{\!\intercal} \mat{L})^\dagger \mat{L}^{\!\intercal}
    \mat{Z}^{\nicefrac{\dagger}{2}} \vecb{y}
    \tag{Let $\mat{L} \eqdef \mat{Z}^{\nicefrac{\dagger}{2}} \mat{X}$, so
    $\mat{L}^{\!\intercal} = \mat{X}^\intercal \mat{Z}^{\nicefrac{\dagger}{2}}$
    as $\mat{Z}^{\nicefrac{\dagger}{2}}$ is symmetric (\lmmref{lmm:z_zpinv_psd})} \\
& = \mat{L}^\dagger \mat{Z}^{\nicefrac{\dagger}{2}} \vecb{y}
    \tag{Since $(\mat{L}^{\!\intercal} \mat{L})^\dagger \mat{L}^{\!\intercal} = \mat{L}^\dagger$
    \citep[\thm{1.2.1}: 6]{campbell_2009_ginv}} \\ %
& = (\mat{Z}^{\nicefrac{\dagger}{2}} \mat{X})^\dagger
    \mat{Z}^{\nicefrac{\dagger}{2}} \vecb{y},
    \tag{Expand $\mat{L}$}
\end{align*}
which can be plugged-in back to the LHS of \eqref{equ:epb_grad} to confirm that
\begin{equation*}
\mat{X}^\intercal \mat{Z}^\dagger \mat{X} (\tilde{\vecb{w}}^*)
= \mat{X}^\intercal \mat{Z}^{\nicefrac{\dagger}{2}} \mat{Z}^{\nicefrac{\dagger}{2}} \mat{X}
    (\mat{L}^\dagger \mat{Z}^{\nicefrac{\dagger}{2}} \vecb{y})
= \{ \mat{L}^{\!\intercal} \mat{L} \mat{L}^\dagger \}
    \mat{Z}^{\nicefrac{\dagger}{2}} \vecb{y}
= \{ \mat{L}^{\!\intercal} \} \mat{Z}^{\nicefrac{\dagger}{2}} \vecb{y}
= \mat{X}^\intercal \mat{Z}^{\nicefrac{\dagger}{2}} \mat{Z}^{\nicefrac{\dagger}{2}} \vecb{y}.
\end{equation*}
Here, we rely on the identity of
$\mat{L}^{\!\intercal} \mat{L} \mat{L}^\dagger = \mat{L}^{\!\intercal}$
\citep[\thm{1.2.1}: 4]{campbell_2009_ginv}.
This concludes the proof.
\end{proof}

%% file: unify.tex
\section{A general procedure for LSTD-based policy evaluation}
\label{sec:unify_lstd_based}

\input{unify_diagram.tex}

Equipped with seminorm LSTD (\secref{sec:seminorm_lstd}), we are now ready to
devise a general unifying procedure for LSTD-based policy evaluation,
which leads to a system of LSTD approximators,
as illustrated in \figref{fig:unify_diagram}.
The proposed procedure is formally presented in \defref{def:general_lstd},
for which we need the definitions of its main components as follows.

\begin{definition} \label{def:neighborhood}
A timestep neighborhood, denoted as $\setname{N}$, is an ordered set of
consecutive timesteps from an anchor timestep $t$ to $(t + \setsize{N} - 1)$.
Every neighborhood $\setname{N}$ has a unique anchor $t$
(\ie the earliest timestep in $\setname{N}$).
Hence, the notation $\setname{N}_t$ denotes a neighboorhod anchored at $t$.
The non-anchor member of $\setname{N}_t$, if any, is called a neighbor.
Hence, every anchor has $(\setsizesubt{N} - 1)$ neighbors.
\end{definition}

\begin{definition} \label{def:pbar}
A state probability distribution of a neighborhood $\setname{N}_t$,
denoted as $\bar{p}_t$, is a lumpsum of stepwise state probabilities $p^\tau$
from $\tau=t$ to $(t + \setsize{N} - 1)$.
That is,
\begin{equation}
\bar{p}_t(s) \eqdef
\frac{1}{\setsizesubt{N}} \sum_{\tau = t}^{t + \setsizesubt{N} - 1} p^\tau(s),
\quad \text{with}\ p^\tau(s) = \E{S_0 \sim \isd}{p^\tau(s| s_0)},
\qquad \forall s \in \setname{S},
\label{equ:p_lumpsum}
\end{equation}
where $p^\tau(s| s_0) \eqdef \prob{S_\tau = s| S_0 = s_0}$, which indicates
the probability of visiting the state $s$ in $\tau$~timesteps from
an initial state $s_0$.
This $p^\tau(s| s_0)$ is equivalent to the $[s_0, s]$-entry of $\mat{P}^\tau$,
which is the one-step transition matrix $\mat{P}$ raised to the power of~$\tau$.
That is, $p^\tau(s| s_0) = \vecb{e}_{s_0}^\intercal \mat{P}^\tau \vecb{e}_s$,
where $\vecb{e}_i \in \real{\setsize{S}}$ denotes the $i$-th standard basis vector.
The $s_0$-th row of $\mat{P}^\tau$ therefore contains the probability values
of the stepwise conditional state distribution $p^\tau(\cdot| s_0)$.
\end{definition}

\begin{definition} \label{def:pbar_support}
The support of a neighborhood $\setname{N}_t$ is defined as
the support of its state distribution $\bar{p}_t$,
denoted as $\setname{S}(\bar{p}_t)$.
That is,
$\setname{S}(\bar{p}_t)
\eqdef \{ s : \bar{p}_t(s) > 0, \forall s \in \setname{S} \}
= \bigcup_{\tau=t}^{t + \setsizesubt{N} - 1} \setname{S}(p^\tau)
\subseteq \setname{S}$, where
$\setname{S}(p^\tau)$ is the support of a stepwise state distribution $p^\tau$.
Note that $\setname{S}(\bar{p}_t)$ may be a proper subset of $\setname{S}$.
\end{definition}

\begin{definition} \label{def:general_lstd}
A general procedure for LSTD-based policy evaluation has three steps as follows.
\begin{enumerate}
\item Specify a number of timestep neighborhoods (\defref{def:neighborhood})
    over the whole time-horizon.

\item Train a seminorm LSTD approximator (\secref{sec:samplingenabler_minimizer})
    for every neighborhood $\setname{N}_t$.
    This approximator minimizes $\epb$ that is based on
    the neighborhood state distribution $\tilde{p} \gets \bar{p}_t$ (\defref{def:pbar})
    and a desired type of policy value functions.
    It is termed as a seminorm LSTD-$\bar{p}_t$.

\item Predict the state values at timestep $\tau$ using
    the approximator of a neighborhood $\setname{N}_t$ where $\tau$ belongs
    (either as an anchor or a neighbor member of $\setname{N}_t$).
\end{enumerate}
The number of timestep neighborhoods (equivalently, the number of anchors or approximators)
is denoted as $\nanchor$.
This procedure forms a system of $\nanchor$ seminorm LSTD-$\bar{p}_t$
(as linear approximators).
\end{definition}

This general procedure unifies two existing approaches to
approximate policy evaluation (\secref{sec:unify_existing}),
as summarized in \tblref{tab:unify_existing}.
It also gives a spectrum of benefits by controlling the number of neighborhoods
(\secref{sec:unify_morethantwo}).
More importantly, it enables value approximation for both
transient and recurrent states in unichain MDPs, which is
the main motivation for this work and is presented in \secref{sec:approxbias}.

\input{unify_existing_tbl.tex}
\input{unify_existing.tex}
\input{unify_morethantwo.tex}
\input{unify_neighborhood.tex}

%% file: unify_diagram.tex
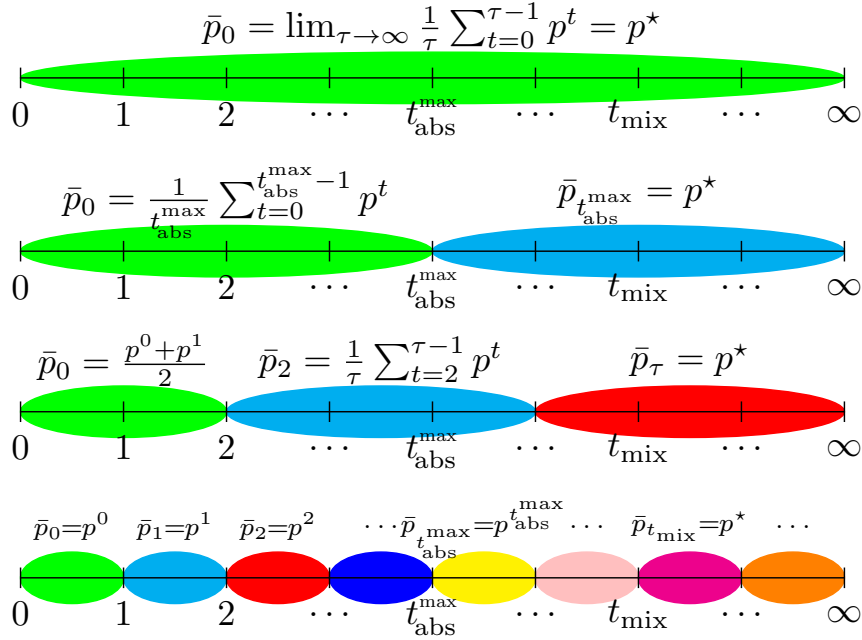
\begin{figure}[t]
\centering

\begin{subfigure}{0.85\textwidth} %
\resizebox{\textwidth}{!}{
\begin{tikzpicture}
\node[ellipse,
draw = green,
fill = green,
minimum width = 8cm,
minimum height = 0.5cm] (e)
at (4,0) {};

\draw (0,0) -- (8,0);
\foreach \i in {0, 1,...,8}
  \draw (\i,0.1) -- + (0,-0.2);

\node at (0,-0.35){$0$};
\node at (1,-0.35){$1$};
\node at (2,-0.35){$2$};
\node at (3,-0.35){$\ldots$};
\node at (4,-0.35){$\tabsmax$};
\node at (5,-0.35){$\ldots$};
\node at (6,-0.35){$\tmix$};
\node at (7,-0.35){$\ldots$};
\node at (8,-0.35){$\infty$};

\node at (4,0.5) [align=center]
    {$\bar{p}_0
        = \lim_{\tau \to \infty} \frac{1}{\tau}\sum_{t=0}^{\tau - 1} p^t
        = p^\star$}; %
\end{tikzpicture}
} %
\end{subfigure}
\begin{subfigure}{0.85\textwidth} %
\resizebox{\textwidth}{!}{
\begin{tikzpicture}
\node[ellipse,
draw = green,
fill = green,
minimum width = 4cm,
minimum height = 0.5cm] (e)
at (2,0) {};
\node at (2,0.5) [align=center]
    {$\bar{p}_0 = \frac{1}{\tabsmax}\sum_{t=0}^{\tabsmax - 1} p^t$};

\node[ellipse,
draw = cyan,
fill = cyan,
minimum width = 4cm,
minimum height = 0.5cm] (e)
at (6,0) {};
\node at (6,0.5) [align=center]
    {$\bar{p}_{\tabsmax} = p^\star$};

\draw (0,0) -- (8,0);
\foreach \i in {0, 1,...,8}
  \draw (\i,0.1) -- + (0,-0.2);
\node at (0,-0.35){$0$};
\node at (1,-0.35){$1$};
\node at (2,-0.35){$2$};
\node at (3,-0.35){$\ldots$};
\node at (4,-0.35){$\tabsmax$};
\node at (5,-0.35){$\ldots$};
\node at (6,-0.35){$\tmix$};
\node at (7,-0.35){$\ldots$};
\node at (8,-0.35){$\infty$};
\end{tikzpicture}
} %
\end{subfigure}
\begin{subfigure}{0.85\textwidth} %
\resizebox{\textwidth}{!}{
\begin{tikzpicture}
\node[ellipse,
draw = green,
fill = green,
minimum width = 2cm,
minimum height = 0.5cm] (e)
at (1,0) {};
\node at (1,0.5) [align=center]
    {$\bar{p}_0 = \frac{p^0 + p^1}{2}$};

\node[ellipse,
draw = cyan,
fill = cyan,
minimum width = 3cm,
minimum height = 0.5cm] (e)
at (3.5,0) {};
\node at (3.5,0.5) [align=center]
    {$\bar{p}_2 = \frac{1}{\tau} \sum_{t=2}^{\tau - 1} p^t$};

\node[ellipse,
draw = red,
fill = red,
minimum width = 3cm,
minimum height = 0.5cm] (e)
at (6.5,0) {};
\node at (6.5,0.5) [align=center]
    {$\bar{p}_\tau = p^\star$};

\draw (0,0) -- (8,0);
\foreach \i in {0, 1,...,8}
  \draw (\i,0.1) -- + (0,-0.2);
\node at (0,-0.35){$0$};
\node at (1,-0.35){$1$};
\node at (2,-0.35){$2$};
\node at (3,-0.35){$\ldots$};
\node at (4,-0.35){$\tabsmax$};
\node at (5,-0.35){$\ldots$};
\node at (6,-0.35){$\tmix$};
\node at (7,-0.35){$\ldots$};
\node at (8,-0.35){$\infty$};
\end{tikzpicture}
} %
\end{subfigure}
\begin{subfigure}{0.85\textwidth}%
\resizebox{\textwidth}{!}{
\begin{tikzpicture}
\node[ellipse,
draw = green,
fill = green,
minimum width = 1cm,
minimum height = 0.5cm] (e)
at (0.5,0) {};
\node at (0.5,0.5) [align=center]
    {\scriptsize $\bar{p}_0\!\!=\!\!p^0$};

\node[ellipse,
draw = cyan,
fill = cyan,
minimum width = 1cm,
minimum height = 0.5cm] (e)
at (1.5,0) {};
\node at (1.5,0.5) [align=center]
    {\scriptsize $\bar{p}_1\!\!=\!\!p^1$};

\node[ellipse,
draw = red,
fill = red,
minimum width = 1cm,
minimum height = 0.5cm] (e)
at (2.5,0) {};
\node at (2.5,0.5) [align=center]
    {\scriptsize $\bar{p}_2\!\!=\!\!p^2$};

\node[ellipse,
draw = blue,
fill = blue,
minimum width = 1cm,
minimum height = 0.5cm] (e)
at (3.5,0) {};
\node at (3.5,0.5) [align=center]
    {\scriptsize $\ldots$};

\node[ellipse,
draw = yellow,
fill = yellow,
minimum width = 1cm,
minimum height = 0.5cm] (e)
at (4.5,0) {};
\node at (4.5,0.5) [align=center]
    {\scriptsize $\bar{p}_{\tabsmax}\!\!=\!\!p^{\tabsmax}$};

\node[ellipse,
draw = pink,
fill = pink,
minimum width = 1cm,
minimum height = 0.5cm] (e)
at (5.5,0) {};
\node at (5.5,0.5) [align=center]
    {\scriptsize $\ldots$};

\node[ellipse,
draw = magenta,
fill = magenta,
minimum width = 1cm,
minimum height = 0.5cm] (e)
at (6.5,0) {};
\node at (6.5,0.5) [align=center]
    {\scriptsize $\bar{p}_{\tmix}\!\!=\!\!p^\star$};

\node[ellipse,
draw = orange,
fill = orange,
minimum width = 1cm,
minimum height = 0.5cm] (e)
at (7.5,0) {};
\node at (7.5,0.5) [align=center]
    {\scriptsize $\ldots$};

\draw (0,0) -- (8,0);
\foreach \i in {0, 1,...,8}
  \draw (\i,0.1) -- + (0,-0.2);
\node at (0,-0.35){$0$};
\node at (1,-0.35){$1$};
\node at (2,-0.35){$2$};
\node at (3,-0.35){$\ldots$};
\node at (4,-0.35){$\tabsmax$};
\node at (5,-0.35){$\ldots$};
\node at (6,-0.35){$\tmix$};
\node at (7,-0.35){$\ldots$};
\node at (8,-0.35){$\infty$};
\end{tikzpicture}
} %
\end{subfigure}

\caption{Illustrations of four systems (rows) of LSTD-$\bar{p}_t$ approximators
along the timestep line in infinite-horizon MDPs.
The first system (top row) consists of only one approximator for
one neighborhood anchored at $t=0$, hence it is based on
the stationary state distribution $\bar{p}_0 = p^\star$.
The second system consists of two approximators whose neighborhoods are
anchored at $t=0$ and $t = \tabsmaxb$, hence they are based on
$\bar{p}_0$ and $\bar{p}_{\tabsmaxb} = p^\star$.
The third system consists of three approximators whose neighborhoods are
anchored at $t=0$, $t=2$ and $t = \tau$ for some timestep $\tau$.
The fourth system consists of an infinity number of approximators
(\ie stepwise approximators), each is based on the stepwise state distribution
$\bar{p}_t = p^t$.
}
\label{fig:unify_diagram}

\end{figure}

%% file: unify_existing_tbl.tex
\begin{table}[]
\centering
\caption{Unification of LSTD-based approximation methods by the general procedure
    (\defref{def:general_lstd}, the right-most column) for
    three types (rows) of environments (Envs).
    Here, $p^{\#}_{\mathrm{tr}}$ denotes the transient part of
    $p^{\#}$ \eqref{equ:ptr_derive}.
    This summary contains the representatives of existing average- and discounted-reward
    LSTD-based methods for the first and second types of environments.
}
\label{tab:unify_existing}
\begin{tabular}{p{0.25\linewidth} p{0.22\linewidth} p{0.21\linewidth} p{0.21\linewidth}}
\toprule
\multicolumn{1}{c}{\textbf{Envs \textbackslash\ Methods}}
& \multicolumn{1}{c}{Norm LSTD-$p^\star$}
& \multicolumn{1}{c}{Norm LSTD-$p^{\#}_{\mathrm{tr}}$}
& \multicolumn{1}{c}{Seminorm LSTD-$\bar{p}_t$} \\ \toprule

Recurrent states only
& {\small \citet{yu_2009_lspe, ueno_2008_lstd}}
& Not applicable since recurrent state information is removed
& One neighborhood $\setname{N}_{t=0}$ with $\bar{p}_0 = p^\star$
    (\figref{fig:unify_diagram}: top row)
\\ \midrule

Multiple transient states and one 0-reward recurrent state
& Not applicable since $p^\star(\strans) = 0$ for every transient state $\strans$ in $\settr$
& {\small \citet{bradtke_1996_lstd, boyan_2002_lstd}}
& {\small Two neighborhoods: $\setname{N}_{t=0}$ and $\setname{N}_{\tabsmax}$
   with $\bar{p}_0 = p^{\#}$, $\bar{p}_{\tabsmax} = p^\star$}
   (\figref{fig:unify_diagram}: second row)
\\ \midrule

Multiple transient states and multiple recurrent states
& Not applicable \newline (same as middle row)
& Not applicable \newline (same as top row)
& At least two neighborhoods (\secref{sec:approxbias})
\\ \bottomrule
\end{tabular}
\end{table}

%% file: unify_existing.tex
\subsection{Existing approaches are special cases with one or two neighborhoods}
\label{sec:unify_existing}

In this section, we show that at least two existing LSTD-based approximators emerge
as special cases of the general procedure (\defref{def:general_lstd}).
These two are of interest because they have the essential components
common to other LSTD-based approximators (see \tblref{tab:unify_existing}).

\textbf{First} is the average-reward LSTD \citep[\secc{II.A}]{yu_2009_lspe},
which was designed for recurrent MCs with rewards.
This is a special case of the general procedure (\defref{def:general_lstd})
when $\nanchor \gets 1$,
yielding a single neighborhood anchored at the initial timestep $t=0$
and with an infinite number of neighbors (due to an infinite time-horizon),
as illustrated in \figref{fig:unify_diagram}:~top-row.
The lumpsum state distribution of $\setname{N}_0$ is obtained by taking
the limit of $p^\tau$ as $\tau$ approaches infinity in \eqref{equ:p_lumpsum}.
This limiting distribution is by definition \eqref{equ:pstar_lim}, equal to
the stationary state distribution, that is $\bar{p}_0 = p^\star$.

Thus, the system of seminorm LSTD-$\bar{p}_t$ reduces to
a single seminorm LSTD-$p^\star$ approximator, then
to a (norm) LSTD-$p^\star$ in a recurrent MC
(where $p^\star(s) > 0, \forall s \in \setname{S}$) whenever
\assref{assume:indep_fea} and a non-singularity condition about
$(\mat{I} - \mat{P})$ are satisfied (see \secref{nbwpval:prelim}).
In such cases, the minimizer \eqref{equ:wstar_seminorm} becomes
$\tilde{\vecb{w}}^*
= (\mat{Z}^{\nicefrac{-1}{2}} \mat{X})^{-1} \mat{Z}^{\nicefrac{-1}{2}} \vecb{y}
= \mat{X}^{-1} \mat{Z}^{\nicefrac{1}{2}} \mat{Z}^{\nicefrac{-1}{2}} \vecb{y}
= \mat{X}^{-1} \vecb{y} = \vecb{w}^*$, which is \eqref{equ:td_fixedpoint}.

\textbf{Second} is the transient-state-only discounted-reward LSTD
\citep[\thm{1}]{bradtke_1996_lstd}, which was designed for
an MC with multiple transient states, plus
a single known 0-reward absorbing terminal state (denoted as $\szrat$).
This is a special case of the general procedure (\defref{def:general_lstd})
when $\nanchor \gets 2$, as illustrated in \figref{fig:unify_diagram}:~second-row.
The first anchor is at $t=0$ as always, whereas the second anchor is
at the maximum absorption time $t = \tabsmax$, which is defined below.
\begin{definition} \label{def:tabsmax}
Let $\ptrmat$ be a non-stochastic $\setsize{S}$-by-$\setsize{S}$ matrix that is
obtained by nullifying (setting to zero) the rows and columns corresponding to
the recurrent states of the one-step transition matrix~$\mat{P}$.
Then, the maximum absorption time $\tabsmax$ is
the time required by a Markov chain such that the $\tabsmax$-th power of
$\ptrmat$ is close to a zero matrix.
That is,
\begin{equation*}
\tabsmax(\varepsilon)
\eqdef \min \{ t : \norm{\ptrmat^t}_{\mathrm{F}} \le \varepsilon \},
\quad \text{and for an infinitesimally small $\varepsilon$}, \quad
\tabsmax \eqdef \tabsmax(\varepsilon = 10^{-8}).
\end{equation*}
This $\tabsmax$ can be interpreted as the timestep at which there is
(almost) no probability mass over all transient states for the first time.
Such probability mass has moved to one or more recurrent states.
Here, $\norm{\ptrmat^t}_{\mathrm{F}}$ denotes the Frobenius matrix norm of
$\ptrmat^t$.
\end{definition}

Setting the second anchor to $\tabsmax$ induces the following two desirable properties.
\begin{enumerate} [label=\roman{*}.]
\item The support of the first neighborhood $\setname{N}_0$ contains
all transient states (as long as the initial state distribution allows),
\ie $\settrans \subseteq \setname{S}(\bar{p}_0)$,
where $\settrans$ denotes the transient state subset.
This cannot be achieved by setting the second anchor to
the minimum or the \emph{expected} absorption time,
by which some transient states may not be contained in $\setname{S}(\bar{p}_0)$.
Note that since the absorption time is a random variable,
it cannot be set as an anchor (\defref{def:neighborhood}).

\item The first neighborhood's state distribution $\bar{p}_0$
yields reasonable weighting for transient states in $\epb$ \eqref{equ:epb}.
It is reasonable in that $\bar{p}_0$ reflects the frequencies of
visiting transient states before absorption.
This is in contrast to, for example, setting the second anchor to $t=1$
whenever the initial state distribution $\isd$ is uniform over $\settrans$.
It induces Property~i. above as $\settrans = \setname{S}(\bar{p}_0 \gets \isd)$,
but does not reflect transient state visitation since
transient states may be visited beyond the first timestep (till absorption).
\end{enumerate}

For an MC with multiple transient states and a known $\szrat$ (which is recurrent),
the first neighborhood's state distribution $\bar{p}_0$ can be modified
such that the probability mass is completely concentrated over $\settrans$.
Let $p^{\#}$ be the modified $\bar{p}_0$ and
$\tabsmax$ be the maximum absorption time (\defref{def:tabsmax}).
Then,
\begin{align}
(\vecb{p}^{\#})^{\!\intercal}
= \isdvecrow \left[
    \frac{1}{\tabsmax} \sum_{t=0}^{\tabsmax - 1} \ptrmat^t \right]
= \isdvecrow \left[
    \frac{1}{\tabsmax} \Bigg\{
        \underbrace{(\mat{I} - \ptrmat)^{-1} }_{\sum_{t=0}^{\infty} \ptrmat^t}
        - \sum_{t=\tabsmax}^{\infty} \ptrmat^t
    \Bigg\} \right],
    \label{equ:ptr_derive}
\end{align}
which is then normalized to
$\vecb{p}^{\#} \gets \vecb{p}^{\#} / \norm{\vecb{p}^{\#}}_1$
to be a vector of probability values of $p^{\#}$.
Here, $\sum_{t=\tabsmax}^{\infty} \ptrmat^t$ is infinitesimally small by
\defref{def:tabsmax}, whereas $(\mat{I} - \ptrmat)^{-1}$ is
a non-stochastic $\setsize{S}$-by-$\setsize{S}$ matrix whose $[s_0, s]$-entry
indicates the expected number of times the agent visits
the state $s$, when it begins in the initial state $s_0$
\citep[\thm{11.4}]{grinstead_2012_prob}.\footnote{
    \citet[\thm{11.3}]{grinstead_2012_prob} proved that
    $\lim_{t \to \infty} \ptrmat^t = \mat{0}$, suggesting that $\ptrmat$ is
    generally not nilpotent.
    However, for some MCs (\eg those with deterministic transition),
    $\ptrmat$ is a nilpotent matrix whose degree is equal to $\tabsmax \le \settrsize$.
}
Therefore, the state visitation from $t=0$ (in a non-absorbing transient state)
until absorption is mainly distributed according to $p^{\#}$.

Thus, the first approximator (of a system of two approximators anchored
at $t=0$ and $\tabsmax$) is devoted to estimating the value of transient states.
It is originally a seminorm LSTD-$(\bar{p}_0 \gets p^{\#})$, but
``forced'' to become a norm variant by the following two ways
(in addition to satisfying \assref{assume:indep_fea}).
\begin{itemize}
\item The $\szrat$ entries in $\vecb{p}^{\#}$, $\vecb{r}$,
    $\mat{F}$, and $\mat{P}$ are removed.
    Such removal is possible because the state classification is known, namely
    the transient states are states that are visited before termination
    (\ie before visiting the only recurrent $\szrat$).
    It is also justifiable because there is no need to estimate the
    value of $\szrat$, which is known to be zero (due to a zero reward).

\item All transient states in $\settr$ have positive probabilities in $\vecb{p}^{\#}$.
    This is guaranteed for example, whenever the support of the initial state distribution
    $\isd$ contains the whole $\settr$.
    Otherwise, a certain transition structure is needed such that
    $p^{\#}(s) > 0, \forall s \in \settr$.
\end{itemize}

The second and the last approximator,
\ie a seminorm LSTD-$(\bar{p}_{\tabsmax} \gets p^\star)$,
concerns with estimating the value of the only recurrent state $\szrat$.
As a result, it is never needed because the value of $\szrat$ is known to be zero
to the agent operating in an MC with a single 0-reward absorbing terminal state.
Note that the last neigborhood's state distribution is
the stationary distribution $p^\star$ (as always),
which is the limit of the lumpsum state distribution (\defref{def:pbar})
as $t$ goes to infinity (from $t = \tabsmax$).

%% file: unify_morethantwo.tex
\subsection{Potential benefits with more than two neighborhoods}
\label{sec:unify_morethantwo}

One extreme of the general procedure (\defref{def:general_lstd}) is
to specify as many neighborhoods as timesteps,
as illustrated in \figref{fig:unify_diagram}:~bottom-row.
This implies one seminorm LSTD-$(\bar{p}_t \gets p^t)$ approximator for each timestep,
where $\bar{p}_t$ takes its specific form of $p^t$ since there is
merely an anchor (without any neighbor) in every neighborhood $\setname{N}_t$.
At first, such stepwise treatment may seem as an overkill for
a time-homogenous MDP with two state types.
It is however, beneficial in three folds as follows.

\textbf{First}, each stepwise LSTD-$p^t$ approximator is fed with
independent and identically distributed (i.i.d) samples
drawn from the corresponding stepwise state distribution $p^t$ across
multiple independent trials (see \algref{alg:training_nbwpval}).
This is in constrast to samples from a lumpsum state distribution $\bar{p}_t$
in a neighborhood with one or multiple neighbors.
Those drawn from such $\bar{p}_t$ in the same trial are Markovian samples,
which yield biased sample means for $\mat{X}$, $\mat{Z}$, and $\vecb{y}$
for the LSTD-$p^t$ minimizer in \thmref{thm:epb_minimizer}.

\textbf{Second}, the stepwise distribution may have a support smaller
than the whole state set, that is $|\setname{S}(p^t)| < \setsize{S}$.
The stepwise approximator's generalization therefore can be focussed on
fewer states, rather than all states in $\setname{S}$.
By product, a system of stepwise LSTD-$p^t$ also enables stepwise trade-off
between approximation accuracy and capacity
(which is limited due to \eg the number of parameters).

\textbf{Third}, putting computation cost aside, stepwise treatment is
a way to deal with unknown state classification in model-free RL by
exploiting what the agent knows, that is the timestep $t$ along with
the corresponding state and reward samples at~$t$.
This is crucial for unichain MDPs with multiple transient states
and multiple recurrent states.
In addition, the unichain category includes recurrent MDPs and
those with transient states and one recurrent state
(as in \tblref{tab:unify_existing}).
Because of this generality  therefore,
the unichain category should be used to model an environment
for which we are not sure about its MDP classification
(and are not willing to make some restrictive assumption).

We propose a resolution to the scalability issue of stepwise LSTD-$p^t$
in the next \secref{sec:neighborhood}.
It accommodates the use-case of more than two but much less than many neighborhoods
(required by the stepwise extreme).
This is at the cost of getting a reduced amount of the above-mentioned benefits.

%% file: unify_neighborhood.tex
\subsection{Parsimoniously specifying the number of neighborhoods}
\label{sec:neighborhood}

The very first step of the general procedure (\defref{def:general_lstd}) is
neighborhood specification.
It amounts to locating $\nanchor$ anchor timesteps along the whole time-horizon
(see \figref{fig:unify_diagram}).
As discussed in \secrefand{sec:unify_existing}{sec:unify_morethantwo},
the initial timestep $t=0$ always serves as the first anchor.
For $\nanchor = 2$, the second anchor is at $t = \tabsmax$
(whenever $\tabsmax$ is known to the agent).
Afterwards, the anchor locations are obvious only for $\nanchor = \infty$,
by which every timestep is an anchor.

It is desirable to be able to locate a finite number of anchors,
\ie $2 \le \nanchor < \infty$, in the context of model-free RL for unichain MDPs
where the state classification (hence, the $\tabsmax$) is unknown.
Therefore, we aim to balance the advantage of having more anchors (\secref{sec:unify_morethantwo})
and the computation of the required seminorm LSTD approximators.
The latter can be indicated by the number of learnable parameters per approximator,
that is $\dim(\vecb{w}) < \setsize{S}$.
Moreover, a stricter computation budget limits the overall number of parameters
in the system of $\nanchor$ approximators.
That is,
\begin{equation}
\{ \dim(\vecb{w}) \cdot \nanchor \} < \setsize{S}
\Longleftrightarrow
\bigg\{ \rho \eqdef \frac{\dim(\vecb{w})}{\setsize{S}} \bigg\} < \frac{1}{\nanchor}
\Longleftrightarrow
\nanchor <  \bigg\{ \frac{\setsize{S}}{\dim(\vecb{w})} = \frac{1}{\rho}\bigg\},
\label{equ:nparam_constraint}
\end{equation}
where $\rho$ indicates the numbers of parameters (per approximator) to
states proportion, and $\dim(\vecb{w})$ is equal to the number of
feature dimensions in linear approximators.
The above inequalities come from the motivation of
using weighted error functions: trading-off approximation accuracy across states
whenever the number of learnable parameters, $\dim(\vecb{w})$,
is (much) less than the number of states, $\setsize{S}$.

For parsimoniously specifying timestep neighborhoods, we propose restricting
neighbors' state distribution to be \emph{in average},
within a tolerance distance $\Delta$ from their anchors'.
This is an attempt to resemble one state distribution per neighborhood
(so that states sampled from such a neighboorhood's distribution are
identically distributed) as much as $\Delta$ allows.
Setting $\Delta \gets 0$ yields one anchor per timestep (till the mixing time),
whereas $\Delta \gets \infty$ trivially yields one anchor at $t=0$.
Consequently, the effective range is at $0 < \Delta < \infty$, where
$\Delta$ is also anticipated to be greater than the threshold used for
determining whether the MC process has been mixing.
For a desired $\nanchor$ anchors, the tolerance $\Delta$ ideally puts
the last anchor close and after the unknown $\tabsmax$ so that
the preceding $(\nanchor - 1)$ approximators are mostly for estimating
transient state values (whereas the last is for recurrent state values, as always).
It is reasonable to have more anchors at the outset, during which
the stepwise state distributions $p^t$ are likely to be non-stationary (time-variant).

Our proposal above relies on the distance between
two unknown state distributions, namely
the anchor's $p^t$ and the candidate neighbor's $p^{t + \tau}$.\footnote{
    One alternative is the distribution ratio of $p^{t + \tau}/p^t$.
    However, density-ratio estimation typically requires another set of
    learnable parameters \citep{sugiyama_2012_dre}.
    Moreover, our use-case involves many such ratios.
}
We identify the following properties for determining the proper distance metric.
\textbf{First}, the distance should be approximated based on two sets of i.i.d samples,
without estimating the distribution directly (\eg via empirical probabilities).
This precludes the use of the total-variation and the earth mover distance
since generally they require distribution estimations as an intermediate step.
\textbf{Second}, the supports of those two distributions are
likely to be different, even disjoint.
This renders the Kullback–Leibler divergence inapplicable.

Based on the aforementioned properties, we choose the maximum mean discrepancy (MMD)
\citep{gretton_2012_mmd} as the distribution distance metric.
It relies on mapping the state distribution $p^t$ into
their so-called mean embedding $\mu_{p^t}$.
That is,
\begin{align}
\mu : \setname{P}_{\!\!\setname{S}} & \mapsto \setname{H}_k,
    \label{equ:mean_embedding} \\
    p^t & \mapsto \mu_{p^t}
                \eqdef \sum_{s \in \setname{S}} p^t(s) k(s, \cdot)
                = \E{S \sim p^t}{k(S, \cdot)}, \label{equ:mean_embedding_2}
\end{align}
where $\setname{P}_{\!\!\setname{S}}$ denotes the space of probability distributions
over $\setname{S}$, and
$\setname{H}_k$ the reproducing kernel Hilbert space, which is induced by
a positive definite kernel $k: \setname{S} \times \setname{S} \mapsto \real{}$.
Here, $\setname{H}_k$ is a space of functions mapping $\setname{S}$ into $\real{}$.
That is,
\begin{align*}
\phi : \setname{S} & \mapsto \setname{H}_k,
    \tag{$\setname{H}_k = \spanspace\{k(s, \cdot) | s \in \setname{S}\}$} \\
s & \mapsto \phi(s) \eqdef \phi(s)(\cdot) = k(s, \cdot),
    \tag{Compare with \eqref{equ:mean_embedding_2}, where the mean of $k(S, \cdot)$ is taken}
\end{align*}
where $\phi(s) \eqdef \phi(s)(\cdot)$ denotes a function that assigns
the value $k(s, s') \in \real{}$ to any $s' \in \setname{S}$.

The MMD is an instance of an integral probability metric,
whose supremum is over functions $\psi$ in the unit ball of $\setname{H}_k$.
Such MMD, denoted by $M_{\setname{H}_k}$, is formulated by
\citet{gretton_2012_mmd} as
\begin{align}
M_{\setname{H}_k}^2(p^t, p^{t+\tau})
& \eqdef \left[ \sup_{\norm{\psi}{} \le 1} \{
        \sum_{} p^t(s) \psi(s) - \sum_{} p^{t + \tau}(s) \psi(s) \} \right]^2
    = \left[ \sup_{\norm{\psi}{} \le 1} \{
        \langle \psi, \mu_{p^t} - \mu_{p^{t + \tau }}\rangle \} \right]^2
    \notag \\
& = \norm{\mu_{p^t} - \mu_{p^{t+\tau}}}_{\setname{H}_k}^2
= \langle \mu_{p^t}, \mu_{p^t} \rangle_{\setname{H}_k}
    + \langle \mu_{p^{t + \tau}}, \mu_{p^{t + \tau}} \rangle_{\setname{H}_k}
    - 2 \langle \mu_{p^t}, \mu_{p^{t + \tau}} \rangle_{\setname{H}_k}
    \notag \\
& = \E{}{\langle \phi(S_t), \phi(\acute{S}_t) \rangle_{\setname{H}_k}}
    + \E{}{\langle \phi(S_{t+\tau}), \phi(\acute{S}_{t+\tau}) \rangle_{\setname{H}_k}}
    - 2 \mathbb{E} \Big[
        \langle \phi(S_t), \phi(S_{t+\tau}) \rangle_{\setname{H}_k} \Big]
    \notag \\
& = \E{}{k(S_t, \acute{S}_t)} + \E{}{k(S_{t+\tau}, \acute{S}_{t+\tau})}
    - 2 \mathbb{E} \Big[k(S_t, S_{t+\tau})\Big]
    \notag \\
& \approx \frac{2}{\nxep} \sum_{i=1}^{\nxep/2}
    k(s_{t}^{2i-1}, s_{t}^{2i}) + k(s_{t+\tau}^{2i-1}, s_{t+\tau}^{2i})
    - k(s_{t}^{2i-1}, s_{t+\tau}^{2i}) - k(s_{t}^{2i}, s_{t+\tau}^{2i-1}),
    \label{equ:mmd_squared_linear}
\end{align}
where $\nxep$ denotes the number of experiment-episodes (trials), and
$s_t^i$ the state sample at timestep $t$ in the $i$-th trial.
The last expression is an unbiased estimator (of the \emph{squared} MMD) that
can be computed in linear time, and may be negative
\citep[\lmm{14}]{gretton_2012_mmd}.
Note that state samples from the same $i$-th trial are not used in such estimation
in \eqref{equ:mmd_squared_linear}.

To become a metric (instead of a pseudo-metric), MMD requires
characteristic kernels, which subsumes universal kernels.
This ensures that each distribution maps to a unique mean embedding in
$\setname{H}_k$ (\ie $\mu_{p^t}$ in \eqref{equ:mean_embedding} is injective,
thus characterizes the distribution $p^t$).
Ideally, we have such a kernel that operates in the original state representation,
which may not be in a Euclidean space.
For discrete states, one example is the identity (Dirac) kernel,
namely $k(s, s') \eqdef \mathbb{I}[s = s']$,
whenever the identity operator $\mathbb{I}$ is available to the agent.
It induces a positive definite kernel Gram matrix, hence
a \emph{strictly} positive definite kernel that is always universal on
discrete domains \citep[\page{42}]{muandet_2017_kme}.
We note that converting a state distance
(\eg based on the bisimulation metric \citep{ferns_2006_met})
into a kernel Gram matrix is likely to yield a kernel that is not even positive definite,
unless it satisfies certain conditions \citep{haasdonk_2004_dsk}.

In some cases, discrete states are represented as numerical feature vectors
in a Euclidean space.
They are obtained via a state feature function
$\vecb{f}(s) \in \integer{\dim(\vecb{w})}$ (\eg one-hot encoding),
or $\vecb{f}(s) \in \real{\dim(\vecb{w})}$.
For these, one popular choice is the Gaussian radial-basis-function (RBF) kernel,
namely
\begin{equation}
k(\vecb{f}(s), \vecb{f}(s'))
\eqdef \exp \Big( - \frac{\norm{\vecb{f}(s) - \vecb{f}(s')}_2^2}{2 \sigma^2} \Big),
\quad \text{with a width (length-scale) hyperparameter $\sigma$},
\label{equ:gauss_kernel}
\end{equation}
which is a universal kernel on compact domains \citep[\tbl{3.1}]{muandet_2017_kme}.
This kernel is relatively interpretable in that it involves
a squared Euclidean distance between $\vecb{f}(s)$ and $\vecb{f}(s')$
scaled by the width hyperparameter $\sigma$.
A very small $\sigma$ yields a kernel matrix that is close to an identity matrix,
implying every state is different.
On the other hand, a very large $\sigma$ yields a kernel matrix whose entries
are all close to~1, implying all states are the same.
Some RL works use this kernel for discrete state environments,
\eg \citet{song_2016_klstd, grunewalder_2012_rkhs, xu_2005_klstd}.

%% file: approxbias.tex
\section{Approximating the bias of transient and recurrent states}
\label{sec:approxbias}

In this section, we describe our proposed approach to approximating
the bias values of unichain MDPs in model-free RL settings.
It is devised from the general procedure (\secref{sec:unify_lstd_based}) with
two additional components specific to bias computation.
They are about reference states and offsets, presented in \secref{sec:sref_offset}.
Subsequently, we explain our proposed pseudocode in \secref{sec:approxbias_pcode}.
Its entry point is \algref{alg:training_nbwpval}, which is about
training (learning) the estimator in model-free RL settings.

\input{approxbias_offset.tex}

\input{approxbias_pcode.tex}

%% file: approxbias_offset.tex
\subsection{Reference states and offset calibration}
\label{sec:sref_offset}

As explained in \secref{nbwpval:prelim}, the projected Bellman error $\epb$
\eqref{equ:epb} is derived from the average-reward evaluation equation
\eqref{equ:poisson_avgrew} for unichain MDPs.
The equation is re-written below,
where the bias state value is denoted as $b$, instead of $v$
(from now on, $v$ denotes the \emph{relative} bias state value).
\begin{equation}
\tilde{\vecb{b}} = \vecb{r} - \tilde{g} \vecb{1} + \mat{P} \tilde{\vecb{b}}
\quad \Longleftrightarrow \quad
(\mat{I} - \mat{P}) \tilde{\vecb{b}} = \vecb{r} - \tilde{g} \vecb{1},
\quad \text{(An underdetermined linear system)}
\label{equ:bias_eval}
\end{equation}
whose solutions are $\tilde{g} = g$, and $\tilde{\vecb{b}} = \vecb{b} + o \vecb{1}$,
where $g$ is the scalar gain (which is constant across states in unichain MDPs),
$\vecb{b} \in \real{\setsize{S}}$ is the bias vector, and
$o \in \real{}$ is an arbitrary offset
\citep[\cor{8.2.7}]{puterman_1994_mdp}.\footnote{
    Another equation, \ie $\mat{P}^\star \tilde{\vecb{b}} = \vecb{0}$,
    is required to be able to determine $\tilde{g} = g$,
    and $\tilde{\vecb{b}} = \vecb{b}$ uniquely without any offset.
    Note that plugging the true gain (\eg from $g = (\vecb{p}^\star)^\intercal\ \vecb{r}$)
    to \eqref{equ:bias_eval} does not change the situation in that
    \eqref{equ:bias_eval} still admits multiple solutions
    (even though the issue of underdetermination has been remedied).
    This is because \eqref{equ:bias_eval} involves a singular matrix
    $(\mat{I} - \mat{P})$.
}
To obtain a solution for $\tilde{\vecb{b}}$ in \eqref{equ:bias_eval} that is
unique (but not necessarily equal to $\vecb{b}$),
we set the arbitrary offset to a certain value, \eg $o \gets -\tilde{b}(\sref)$
for an arbitrary reference state $\sref$.
This yields $\tilde{\vecb{b}} \gets \tilde{\vecb{b}} -\tilde{b}(\sref) \vecb{1}$,
whose resulting value is called the \emph{relative} bias value at $\sref$.

Thus, the bias approximation (by minimizing $\epb$) actually estimates
the relative bias value $\vecb{v}$, which is equal to the bias $\vecb{b}$
up to some offset~$o$.
That is, $\hat{\vecb{v}} \approx \{ \vecb{v} = \vecb{b} + o \vecb{1} \}$.
Since any arbitrary offset satisfies \eqref{equ:bias_eval}, we can
adjust the offset $o$ to be $o \gets - b(\sref)$ in a similar fashion
as determining a unique $\tilde{\vecb{b}}$ (in the previous passage).
This is somewhat advantageous since at least,
one true relative-value at $\sref$ is known to be zero, namely
$v(\sref) = b(\sref) + \{ o = -b(\sref) \} = 0$.
Therefore, we introduce a \emph{prediction} offset, denoted as $\tilde{o}$, and
set it to $\tilde{o} \gets - \hat{v}(\sref)$ in order to ensure that
the predicted relative value at $\sref$ matches with its true value.
That is, $\hat{\vecb{v}} \gets \hat{\vecb{v}} + \tilde{o} \vecb{1}$ such that
$\hat{v}(\sref) = \hat{v}(\sref) + \{ \tilde{o} = - \hat{v}(\sref) \} = v(\sref) = 0$.

Adjusting the prediction of multiple approximators requires a bit of work,
which we explain in the rest of this section.
They are about identifying reference states and calibrating prediction offsets.

\input{approxbias_offset_sref.tex}
\input{approxbias_offset_calib.tex}

%% file: approxbias_offset_sref.tex
\subsubsection{Identifying reference states in a system of multiple approximators}
\label{sec:approxbias_offset_sref}

For a system of approximators of the general procedure (\defref{def:general_lstd}),
one strategic choice for $\sref$ is the most common state across
all neighborhood supports (\defref{def:pbar}) with
a tie-breaking rule as in \defref{def:tie_breaking}.
This $\sref$ is deemed as the \emph{main} reference state of the system.
Note that $\sref$ is not necessarily a recurrent state due to
neighborhood specification and in practice, because
there is a finite number of timesteps and
the neighborhood supports are estimated based on empirical state samples.

\begin{definition} \label{def:tie_breaking}
When determining the most common states across multiple neighborhood supports,
any tie (including when the state frequencies are all ones) is resolved
by selecting any state from the earliest neighborhood for prioritizing
the estimation accuracy of transient state values.
\end{definition}

For the remaining neighborhoods whose supports do not contain the main $\sref$,
we search for potentially multiple auxilary reference states $\sref'$
via the following procedure.
\begin{enumerate}
\item Initialize the reference state set $\srefset \gets \{ \sref \}$.
\item Search for an auxilary reference state $\sref'$ that
    simultaneously satisfies two rules below,
    \begin{enumerate} [label=\roman{*}.]
    \item the most common among neighborhood supports that are disjoint with $\srefset$
        (using the same tie-breaking as for identifying the main $\sref$
        (\defref{def:tie_breaking})), \emph{and}
    \item contained in any neighborhood support that is not disjoint with $\srefset$.
    \end{enumerate}
\item If a new auxilary $\sref'$ is found \emph{and} there is at least
    one neighborhood that still does not have any reference state,
    then $\srefset \gets \srefset \cup \{ \sref' \}$ and go to Step 2.
    Otherwise, stop.
\end{enumerate}

After applying the above procedure, there may exist neighborhoods whose supports
still do not contain either the main $\sref$ or any auxiliary $\sref'$.\footnote{
    We conjecture that in theory, an auxilary $\sref'$ exists for every neighborhood
    whenever at least one of the following conditions is fulfilled, namely
    i) each state has non-zero probabilities for transitioning to itself,
    as well as for transitioning to another state, and
    ii) the initial state distribution $\isd$ has the whole state set as its support.
}
Such neigboorhoods are eventually merged to their nearest (in terms of timesteps)
neighborhood with any type of reference states
(the precedence is given to the preceding neighborhood whenever tie occurs).
The earliest anchor (among those of the merged neigborhoods) becomes
the anchor of the newly-formed neighborhood, whereas
the other (now defunct) anchors become the neighbors.

%% file: approxbias_offset_calib.tex
\subsubsection{Calibrating prediction offsets in a system of multiple approximators}
\label{sec:offset_calib}

Once every neighborhood $\setname{N}_t$ (whose approximator is denoted by $\hat{v}_t$)
is assigned a reference state, its prediction offset $\tilde{o}_t$ is set as follows.
\begin{align}
\tilde{o}_t
& \gets - \hat{v}_t(\sref),
\quad \text{for each approximator $\hat{v}_t$ with the main $\sref$, and}
\label{equ:offset_main_sref} \\
\tilde{o}_{t'}
& \gets [\hat{v}_t(\sref') + \tilde{o}_t] - \hat{v}_{t'}(\sref'),
\quad \text{for each approximator $\hat{v}_{t'}$ with an auxilary $\sref'$}
\label{equ:offset_aux_sref}
\end{align}
where $\hat{v}_t$ in \eqref{equ:offset_aux_sref} is of any neighborhood
whose support contains $\sref'$.
These $\tilde{o}_t$ are then applied to the corresponding prediction as
$\hat{\vecb{v}}_t \gets \hat{\vecb{v}}_t + \tilde{o}_t \vecb{1}$.
\algref{alg:calibrate_offset} implements this prediction calibration.

Applying the prediction offset $\tilde{o}$ to a system of approximators forces
at least one state to have the same approximated value in two neighborhoods.\footnote{
    Recall that if \eqref{equ:bias_eval}, from which $\epb$ is derived, did not
    admit multiple solutions, neighborhood-wise approximators would allow
    different value estimates for all states in different neighborhoods
    (\secref{sec:unify_morethantwo}).
}
This is the cost we pay for two purposes.
\textbf{First} is to propagate the unique and true relative bias value at
the main reference state $\sref$, namely $\hat{v}(\sref) = v(\sref) = 0$,
throughout all neighborhoods' approximators.
This propagation is carried out exactly for neighborhoods with
the main reference state $\sref$ via \eqref{equ:offset_main_sref}.
For those with an auxilary reference state $\sref'$, it is carried out approximately
via \eqref{equ:offset_aux_sref}.
\textbf{Second} is to accomodate the joint-use of multiple relative-value approximators,
which originally have different offsets with respect to the true bias.
Such a use-case arises for example, when computing a quantity that involves
relative values of multiple transient and recurrent states
whose estimates come from multiple approximators.

%% file: approxbias_pcode.tex
\subsection{Pseudocode} \label{sec:approxbias_pcode}

In this section, we present the pseudocode for the proposed relative-value
approximator from multiple transient and recurrent states.
The central pseudocode is \algref{alg:training_nbwpval}, which contains
the training protocol in model-free RL settings.
After obtaining state and reward samples,
it specifies a list of neighborhoods (\algref{alg:identify_anchor}),
computes the minimizer of the seminorm LSTD for each neighborhood
(\secref{sec:samplingenabler_minimizer}),
and finally calibrates the prediction offset (\algref{alg:calibrate_offset}).

Specifically, \algref{alg:identify_anchor} approximates
the timestep locations of anchors (via \algref{alg:approx_anchor}), then
identifies the anchors' reference states (via \algref{alg:identify_sref}).
\algref{alg:approx_anchor} relies on MMD to measure
the distributional distance between each anchor and its neighbor candidates
(\secref{sec:neighborhood}).
For an anchor at timestep $t$, one may select (or sample)
a reasonable number of timesteps from a set $\{ t+1, t+2, \ldots, \txepmax \}$
as neighbor candidates, for which their state distribution distances from
the anchor's are approximated.

\input{approxbias_pcode_train.tex}

\input{approxbias_pcode_neighborhood.tex}
\input{approxbias_pcode_anchor.tex}
\input{approxbias_pcode_sref.tex}
\input{approxbias_pcode_offset.tex}

%% file: approxbias_pcode_train.tex
\begin{algorithm}[]
\caption{Training a system of relative-value approximators for
    transient and recurrent states in model-free settings.
    This implements the general procedure (\defref{def:general_lstd}).}
\label{alg:training_nbwpval}
\DontPrintSemicolon

\KwInput{
A stationary policy $\pi$,
a state feature function $\vecb{f}$,
a kernel function~$k$,
a number of anchors~$\nanchor$,
a number of experiment-episodes (trials) $\nxep$
(each is with $\txepmax + 1$ timesteps).
}
\KwOutput{
    A list of neighborhoods $\tilde{\setname{N}}$ along with
    their parameters $\tilde{\setname{W}}$ and
    prediction offsets $\tilde{\setname{O}}$.
}

Initialize stepwise state and reward lists of sample lists:
    $\tilde{\setname{S}} \gets \varnothing$ and
    $\tilde{\setname{R}} \gets \varnothing$, respectively. \\
\For{Each experiment-episode (trial) $i = 0, 1, \ldots, \nxep - 1$}{
    Reset the environment, and obtain an initial state $s_0$,
        then set the state variable $s \gets s_0$. \\
    \For{Each timestep $t = 0, 1, \ldots, \txepmax$}{
        Choose to then execute an action $a$ based on $\pi(\cdot|s)$. \\
        Observe the next state $s'$ and the reward $r$. \\
        Append samples $s$ and $r$ to the corresponding
            $\tilde{\setname{S}}$ and $\tilde{\setname{R}}$ at index $t$. \\
        Update the current state variable $s \gets s'$.
    }
    \If{Desired (at least once)}{
        Get a list of neighborhoods,
            $\tilde{\setname{N}} \gets
                \mathrm{SpecifyNeighborhoods}(\tilde{\setname{S}}, \nanchor, k)$.
                \Comment*[r]{\algref{alg:identify_anchor}}

        Set an empty list of learned weights
            $\tilde{\setname{W}} \gets \varnothing$. \\
        \For{Each neighborhood $\setname{N}_t \in \tilde{\setname{N}}$ }{
            Learn the minimizing parameter
                $\tilde{\vecb{w}}_t$ using $\setname{N}_t$, $\tilde{\setname{S}}$,
                $\tilde{\setname{R}}$, and $\vecb{f}$.
                \Comment*[r]{\thmref{thm:epb_minimizer}}
            Put this learned $\tilde{\vecb{w}}_t$ to $\tilde{\setname{W}}$
                at index $t$.
        }
        Calibrate the prediction offsets,
            $\tilde{\setname{O}} \gets \mathrm{CalibrateOffset}(
            \tilde{\setname{W}}, \tilde{\setname{N}}, \vecb{f})$.
            \Comment*[r]{\algref{alg:calibrate_offset}}

    }
}

\Return $\tilde{\setname{N}}$, $\tilde{\setname{W}}$,  and $\tilde{\setname{O}}$.

\end{algorithm}

%% file: approxbias_pcode_neighborhood.tex
\begin{algorithm}[]
\caption{$\mathrm{SpecifyNeighborhoods}(\tilde{\setname{S}}, \nanchor, k)$
    specifies the timestep neighborhoods (\defref{def:neighborhood}).}
\label{alg:identify_anchor}
\DontPrintSemicolon

\KwInput{
    A state sample list $\tilde{\setname{S}}$,
    a desired number of anchors $\nanchor \ge 2$, and
    a kernel function $k$.
}
\KwOutput{
    A list of neighborhoods $\tilde{\setname{N}}$,
    each is augmented with reference state information.
}

Initialize a list of anchors, $\tilde{\setname{T}} \gets \varnothing$. \\
\If{$2 \le \nanchor < \{ \txepmax = \mathrm{length}(\tilde{\setname{S}}) - 1 \} $} {
    $\tilde{\setname{T}}
    \gets \mathrm{ApproximateAnchors}(\tilde{\setname{S}}, \nanchor, k)$.
    \Comment*[r]{\algref{alg:approx_anchor}}
}
\Else {
    $\tilde{\setname{T}} \gets [0, 1, \ldots, \txepmax]$.
    \Comment*[r]{Stepwise anchors}
}
Construct a neighborhood list $\tilde{\setname{N}}$
    based on $\tilde{\setname{T}}$.
    \Comment*[r]{\defref{def:neighborhood}}
Identify a reference state list,
    $\srefset \gets \mathrm{IdentifySref}(\tilde{\setname{N}}, \tilde{\setname{S}})$.
    \Comment*[r]{\algref{alg:identify_sref}}
Retain or merge neighborhoods based on reference state existence in $\srefset$.
    \Comment*[r]{\secref{sec:approxbias_offset_sref}}

\Return $\tilde{\setname{N}}$.
\end{algorithm}

%% file: approxbias_pcode_anchor.tex
\begin{algorithm}[]
\caption{$\mathrm{ApproximateAnchors}(\tilde{\setname{S}}, \nanchor, k)$
    approximates the anchor locations (\secref{sec:neighborhood}).}
\label{alg:approx_anchor}
\DontPrintSemicolon

\KwInput{
    A state sample list $\tilde{\setname{S}}$,
    a desired number of anchors $\nanchor \ge 2$, and
    a kernel function $k$.
}
\KwOutput{
    A list of anchors $\tilde{\setname{T}}$ (recall: an anchor is a timestep index).
}
Initialize the distance tolerance $\Delta$ to a small positive number, and \newline
    the current number of approximate anchors $\bar{\nanchor} \gets \txepmax$
    with $\txepmax = \mathrm{length}(\tilde{\setname{S}}) - 1$. \\
\While{$\bar{\nanchor} > \nanchor$}{
    Reset a list of anchors $\tilde{\setname{T}} \gets [0]$, and
        the average neighbor-candidate distance $\bar{\Delta}^2 \gets 0$. \\
    \For{Each timestep $t = 1, 2, \ldots, \txepmax$}{
        Set the last anchor, $\tilde{t} \gets \tilde{\setname{T}}[-1]$.
            \Comment*[r]{The last item is at index $-1$}
        Get the squared MMD, $\hat{d}_t^2$ by plugging-in
            state samples $\tilde{\setname{S}}[\ t, \tilde{t}\ ]$ and kernel $k$
            to \eqref{equ:mmd_squared_linear}. \\
        Update $\bar{\Delta}^2
            \gets \bar{\Delta}^2 + (\hat{d}_t^2 - \bar{\Delta}^2)/(t - \tilde{t})$.
            \Comment*[r]{Update the running average}
        \If{The average neighbor-candidate distance $\bar{\Delta}^2 > \Delta^2$}{
            Append $t$ to the anchor list,
                $\tilde{\setname{T}} \gets \tilde{\setname{T}} + [t]$. \\
            Reset $\bar{\Delta}^2 \gets 0$.
        }
    }
    Update $\bar{\nanchor} \gets \mathrm{length}(\tilde{\setname{T}})$. \\
    Increase the distance tolerance $\Delta \gets 2 \Delta$.
        \Comment*[r]{For some multiplier, \eg $2$}
}

\Return $\tilde{\setname{T}}$.
\end{algorithm}

%% file: approxbias_pcode_sref.tex
\begin{algorithm}[]
\caption{$\mathrm{IdentifySref}(\tilde{\setname{N}}, \tilde{\setname{S}})$
    identifies the neigborhoods' reference states (\secref{sec:approxbias_offset_sref}).}
\label{alg:identify_sref}
\DontPrintSemicolon

\KwInput{
    A list of neighborhoods $\tilde{\setname{N}}$ and
    a list of state samples $\tilde{\setname{S}}$.
}
\KwOutput{
    A list of reference states $\srefset$.
}

Initialize lists of neighborhoods
    without reference states $\mathring{\setname{N}} \gets \tilde{\setname{N}}$ and
    with reference states $\check{\setname{N}} \gets \varnothing$,
    and a list of reference states $\srefset \gets \varnothing$
    (which will contain $\sref$ of neighborhoods in $\check{\setname{N}}$),
    and two variables:
    \texttt{ForAuxSref} to False and \texttt{CandidateSrefIsValid} to True.

\While{\emph{\texttt{CandidateSrefIsValid}} and $\mathring{\setname{N}}$ is not empty}{
    Set $\mathring{\setname{S}}$ to the list of the supports of neighborhoods
    in $\mathring{\setname{N}}$. \\
    Set $\check{\setname{S}}$ to the list of the supports of
    neighborhoods in $\check{\setname{N}}$. \\
    Clear the list of banned $\sref$, if any.
        \Comment*[r]{See the banning of $\sref$ in \lineref{line:ban_sref}}

    \For{Each $i = 1, 2, \ldots$ until
            the number of unique states in $\mathring{\setname{S}}$}{
        Set $\sref$ to the most-common unbanned support in
        $\mathring{\setname{S}}$ (with tie-breaking as in \defref{def:tie_breaking}). \\
        \If{\emph{\texttt{ForAuxSref}} and $(\sref$ is not in $\check{\setname{S}})$ }{
            Set \texttt{CandidateSrefIsValid} to False, ban this $\sref$, and \Continue.
                \label{line:ban_sref}
        }
        \For{Each neighborhood in $\mathring{\setname{N}}$}{
            \If{$\sref$ is contained in the support of this neighborhood}{
                Set $\mathring{t}$ to this neigborhood's anchor. \\
                Set $\srefset[\ \mathring{t}\ ] \gets \sref$.
                    \Comment*[r]{Either as a main or aux reference state}
                Remove this neighborhood from $\mathring{\setname{N}}$, then
                add it to $\check{\setname{N}}$.
            }
        }
        Set \texttt{CandidateSrefIsValid} to True, then \Break.
    }
    Set \texttt{ForAuxSref} to True.
        \Comment*[r]{Looking for the main $\sref$ is only in the 1st pass}
}

\Return $\srefset$.
\end{algorithm}

%% file: approxbias_pcode_offset.tex
\begin{algorithm}[]
\caption{$\mathrm{CalibrateOffset}(
    \tilde{\setname{W}}, \tilde{\setname{N}}, \vecb{f})$
    calibrates the offsets of all approximators (\secref{sec:offset_calib}).}
\label{alg:calibrate_offset}
\DontPrintSemicolon

\KwInput{
    Lists of learned parameters $\tilde{\setname{W}}$
    and of neighborhoods $\tilde{\setname{N}}$,
    and a state feature function $\vecb{f}$.
}
\KwOutput{
    A list of calibrated prediction offsets $\tilde{\setname{O}}$, which
    corresponds to $\tilde{\setname{W}}$.
}

Initialize a list of learned parameters (representing approximators)
    without calibrated offsets $\mathring{\setname{W}} \gets \tilde{\setname{W}}$,
    as well as a list of calibrated prediction offsets $\tilde{\setname{O}} \gets \varnothing$.

\For{Each neighborhood's anchor $\tilde{t}$
        with the main $\sref$ in $\tilde{\setname{N}}$}{
    Set $\tilde{\vecb{w}}$ to the entry of $\tilde{\setname{W}}$ at $\tilde{t}$,
        and remove the entry $\tilde{\vecb{w}}$ from $\mathring{\setname{W}}$. \\
    $\tilde{\setname{O}}[\tilde{\vecb{w}}] \gets
        - \vecb{f}^\intercal(\sref)\ \tilde{\vecb{w}}$.
    \Comment*[r]{This implements \eqref{equ:offset_main_sref}}
}

\While{$\mathring{\setname{W}}$ is not empty}{
    Choose any entry $\mathring{\vecb{w}}$ in $\mathring{\setname{W}}$. \\
    Look up the auxilary $\sref'$ of a neighborhood whose approximator is
        represented by $\mathring{\vecb{w}}$. \\
    \For{Each calibrated offset $\tilde{o}$ in $\tilde{\setname{O}}$}{
        Set $\tilde{\vecb{w}}$ to the learned parameter corresponding to $\tilde{o}$. \\
        \If{$\sref'$ is in the support of a neighborhood whose approximator
                is represented by $\tilde{\vecb{w}}$}{
            $\tilde{\setname{O}}[\mathring{\vecb{w}}] \gets
            (\vecb{f}^\intercal(\sref')\ \tilde{\vecb{w}} + \tilde{o})
            - \vecb{f}^\intercal(\sref')\ \mathring{\vecb{w}}$.
            \Comment*[r]{This implements \eqref{equ:offset_aux_sref}}
            Remove the entry $\mathring{\vecb{w}}$ from $\mathring{\setname{W}}$,
                then \Break.
        }
    }
}
\Return $\tilde{\setname{O}}$.
\end{algorithm}

%% file: xprmtsetup.tex
\section{Experimental setup} \label{nbwpval:xprmtsetup}

In this section, we describe the setup of our experiments, whose
results are presented in \secref{nbwpval:xprmtresult}.
We begin with the environment specifications in \secref{sec:env_nbwpval},
followed by state features and state kernels (\secref{sec:state_fea_ker}).
Then, various experiment schemes are described in \secref{sec:xprmt_scheme}.
Lastly, we explain the evaluation metrics and protocols
in \secref{nbwpval:xprmt_metric}.

\input{xprmtsetup_env.tex}

\input{xprmtsetup_fea.tex}
\input{xprmtsetup_scheme.tex}
\input{xprmtsetup_eval.tex}

%% file: xprmtsetup_env.tex
\subsection{Environments} \label{sec:env_nbwpval}

We evaluate our proposed method on environments whose
all stationary deterministic policies induce unichain MCs.
Those environments are formed by connecting a recurrent MDP to a transient structure.
Each environment is identified by a mnemonic, \eg x123c, where
the first letter (\ie `x') denotes a particular recurrent MDP,
followed by a total number of states (\ie `123'), and an identifier for
the transient structure (\ie `c'), which will be explained shortly.

We use the following recurrent MDPs from the literature.
They are listed by their single-letter identifiers
(which become the first letter in their environment mnemonics) as follows:
`h' is with 3 recurrent states \citep{hordijk_1985_disc}, and
`c' is with 5 recurrent states \citep{strens_2000_bfrl}.
In these recurrent MDPs, every state has two available actions.

For transient structures, we use a generic structure as depicted
in \figref{fig:unichain_generic}.
A specific instance of it is mainly characterized by
the number of streams of transient states, for which multiple streams of
$2, 3, \ldots$ are denoted by a single letter of `b', `c', $\ldots$, respectively
(which becomes the last letter in the environment mnenomic).
For simplicity, all streams have an identical transition and reward structure.
They are also all connected to an arbitrary recurrent state.
In all streams, every transient state has two available actions, where
every action leads to two possible outcomes, namely
its own state (self-loop, self-transition) and another state.

In these environments, the initial state is always transient.
Specifically, the initial state distribution $\isd$ assigns a probability of
$1/\settrsize$ for every transient state and $0$ for every recurrent state.

\input{xprmtsetup_env_diagram.tex}

%% file: xprmtsetup_env_diagram.tex
\begin{figure}[t]
\centering

\resizebox{0.75\textwidth}{!}{
\begin{tikzpicture}[
node distance = 2.5cm and 2.5cm, on grid,
-latex, %
semithick, %
strans/.style={circle, fill=green,
draw=black, text=black, minimum width = 1cm},
srecur/.style={circle, fill=lightgray,
draw=black, dashed, text=black, minimum width = 1cm},
]
\node[strans](A)[] {$s^0$};
\node[strans](B)[right=of A] {$s^1$};
\node[srecur](I)[below right=of B] {$s^8$};
\node[strans](D)[below left=of I] {$s^3$};
\node[strans](C)[left=of D] {$s^2$};
\node[strans](E)[above right=of I] {$s^4$};
\node[strans](F)[right=of E] {$s^5$};
\node[strans](G)[below right=of I] {$s^6$};
\node[strans](H)[right=of G] {$s^7$};

\path (A) edge [out=30, in=120, red] node[] {} (B);
\path (A) edge [out=30, in=90, red, loop] node[] {} (A);
\path (A) edge [out=-30, in=210, blue] node[] {} (B);
\path (A) edge [out=-30, in=-90, loop, blue] node[] {} (A);

\path (B) edge [out=-30, in=120, red] node[] {} (I);
\path (B) edge [out=-30, in=90, loop, red] node[] {} (B);
\path (B) edge [out=-60, in=180, blue] node[] {} (I);
\path (B) edge [out=-60, in=-120, blue, loop] node[] {} (B);

\path (I) edge [out=90, in=30, dashed, loop, red] node[] {} (I);
\path (I) edge [out=270, in=150, dashed, loop, blue] node[] {} (I);

\path (C) edge [out=30, in=135, red] node[] {} (D);
\path (C) edge [out=30, in=90, red, loop] node[] {} (C);
\path (C) edge [out=-30, in=225, blue] node[] {} (D);
\path (C) edge [out=-30, in=-90, loop, blue] node[] {} (C);

\path (D) edge [out=60, in=210, red] node[] {} (I);
\path (D) edge [out=60, in=120, loop, red] node[] {} (D);
\path (D) edge [out=0, in=240, blue] node[] {} (I);
\path (D) edge [out=0, in=-90, blue, loop] node[] {} (D);

\path (E) edge [out=210, in=60, red] node[] {} (I);
\path (E) edge [out=210, in=120, loop, red] node[] {} (E);
\path (E) edge [out=240, in=0, blue] node[] {} (I);
\path (E) edge [out=240, in=300, blue, loop] node[] {} (E);

\path (F) edge [out=150, in=60, red] node[] {} (E);
\path (F) edge [out=150, in=90, red, loop] node[] {} (F);
\path (F) edge [out=210, in=-30, blue] node[] {} (E);
\path (F) edge [out=210, in=-90, loop, blue] node[] {} (F);

\path (G) edge [out=120, in=-30, red] node[] {} (I);
\path (G) edge [out=120, in=60, loop, red] node[] {} (G);
\path (G) edge [out=180, in=-60, blue] node[] {} (I);
\path (G) edge [out=180, in=270, blue, loop] node[] {} (G);

\path (H) edge [out=150, in=30, red] node[] {} (G);
\path (H) edge [out=150, in=90, red, loop] node[] {} (H);
\path (H) edge [out=210, in=-30, blue] node[] {} (G);
\path (H) edge [out=210, in=-90, loop, blue] node[] {} (H);

\end{tikzpicture}
} %

\caption{The diagram of an instance `x9d' of generic unichain MDPs.
In this instance, there are eight transient states (green solid circles) that
belongs to four (labelled `d') streams of transient states
(here, each stream consists of two transient states).
In every transient state, there are two available actions (red and blue solid edges),
each has two possible outcomes: either the current state (self-loop) or another state.
All four stream of transient states are connected to a recurrent MDP (labelled `x'),
which has a single recurrent state (a gray dashed circle) with
two self-loop actions (red and blue dashed edges).
If there were multiple recurrent states, they would form a single recurrent class,
which could be lumped together and represented by a single absorbing state in the diagram.
}
\label{fig:unichain_generic}

\end{figure}
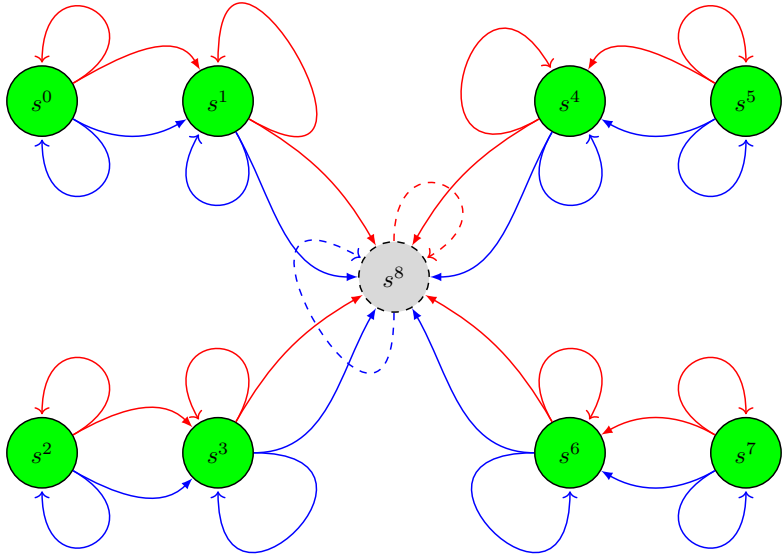

%% file: xprmtsetup_fea.tex
\subsection{State features and state kernels}
\label{sec:state_fea_ker}

Since feature extraction and selection are not the focus of this work,
we use a random feature vector to represent a state.
That is, the state feature
$\vecb{f}(s) \in \real{\dim(\vecb{w})}, \forall s \in \setname{S}$,
is constructed by randomly sampling each $i$-th dimension feature value
as $f_i(s) \sim \setname{G}(\mu= \mathrm{index}(s), \sigma^2=1)$,
where $\setname{G}$ is a Gaussian distribution with
a mean $\mu$ (which is set to the non-negative integer index of a state $s$)
and a unit variance $\sigma^2$.

The aforementioned state features accomodates the use of Gaussian RBF kernels
for computing the state-distribution distances using MMD (\secref{sec:neighborhood}).
For experiments, we set the kernel width in \eqref{equ:gauss_kernel} to~$1$ after
the variance of the standard Gaussian distribution for $f_i(s)$.

%% file: xprmtsetup_scheme.tex
\subsection{Experiment schemes} \label{sec:xprmt_scheme}

Our experiment schemes are products of three sources of variations.
\textbf{First} is the number of feature dimensions, indicated by
the feature-to-state dimensional proportion $\rho$.
Since we need at least two approximators (anchors) representing
two states classes in unichain MDPs,
the strict computation limit in \eqref{equ:nparam_constraint} yields
an upper bound of $\mathbbm{\rho} < 1/(\nanchor = 2)$.
For experiments therefore, we select the following six feature dimension regimes,
which in turn constraint the maximum number of anchors $\nanchormax$
according to \eqref{equ:nparam_constraint}.
They are collected as a set of tuples $(\rho, \nanchormax)$ as follows,
\begin{equation}
\Big\{ (0.49, \floor*{2.04}), (0.33, \floor*{3.03}), (0.19, \floor*{5.3}),
(0.09, \floor*{11.1}), (0.06, \floor*{16.7}),
(0.03, \floor*{33.3}) \Big\},
\label{equ:nfea_nanchor_tuples}
\end{equation}
where $\floor*{x}$ indicates the flooring operation,
\ie the greatest integer less than or equal to $x$.

\textbf{Second} is based on algorithmic variations of the proposed method
(\secref{sec:approxbias}), as well as the baseline.
Such variations come from varying the number of anchors\footnote{
    Note that if some neighborhoods do not have joint supports,
    there will be fewer anchors than what is specified to
    \algrefand{alg:training_nbwpval}{alg:identify_anchor},
    as explained in \secref{sec:approxbias_offset_sref}.
} and the state-distribution distance metrics.
There are ten approximation schemes in four groups as follows.

\begin{enumerate} [label=\roman{*}.]
\item \textbf{`buw' and `p01'}:
These mnemonics refer to the unweighted baseline and
the proposed one-approximator scheme, respectively.
The latter `p01' has a single anchor at $t=0$, and uses the seminorm LSTD
because $\bar{p}_0 = p^\star$ has only recurrent states as its support
(hence, $\diag(\bar{p}_0)$ is PSD).
On the other hand, `buw' uses a (norm) LSTD approximator because
the state distribution is set to be uniform over all states.
This uniformity also implies that the state-wise value errors are
unweighted in $\epb$ \eqref{equ:epb}.
We are not aware of any other baseline besides `buw' for estimating the bias values
from transient and recurrent states with parametric function approximators.
Note that the existing methods are with a single norm LSTD,
but they are applicable solely for recurrent MDPs \citep{yu_2009_lspe},
or unichain MDPs with one zero-reward recurrent state \citep{bradtke_1996_lstd}.

\item \textbf{`p02am', `p02tv', `p02ot', and `p02md'}:
These mnemonics refer to the proposed two-approximator schemes using various ways
to determine the second anchor location, namely
at a given maximum absorption time $\tabsmax$ (`am'), or
based on three different state-distribution distance metrics:
total variation (TV, `tv'), optimal transport (OT, `ot'), and MMD ('md').
The variant with $\tabsmax$ (`p02am') is instantiated so that
the first neighborhood's state distribution has all transient states in its support.
It also matches the existing method for unichain MDPs with
one 0-reward recurrent state (as explained in \secref{sec:unify_existing}).
The variant with TV (`p02tv') is motivated by the fact that
TV is a typical metric for determining the mixing time.
We also experiment with an OT-based variant (`p02ot') because OT considers
the state distance (as the underlying non-probabilistic metric)\footnote{
    We use the OT implementation of \citet{flamary_2021_pot}.},
making it on par with the MMD variant (`p02md').
Moreover, the neighborhood specification based on OT can serve as ground-truth
whenever OT uses a state distance metric that does not depend on state representation
(\cf MMD involves a state kernel whose hyperparameters are heuristically determined).
For OT computation therefore, we use a behavioral state similarity derived from
environment properties, such as transition and reward functions.
It is the $\pi$-bisimulation pseudo-metric \citep[\thm{2}]{castro_2020_sim},
specifically its state-action counterpart \citep[\lmm{7}]{lan_2021_mcrl}.
That is, the distance between two states $s$ and $s'$ under a policy $\pi$ is
given by
$d_\pi(s, s') \eqdef \max_{a \in \setname{A}} |q_\gamma^\pi(s, a) - q_\gamma^\pi(s', a)|$,
where $q_\gamma^\pi$ is the discounted state-action value of $\pi$
(here, the discount factor $\gamma$ is set to $0.999$).\footnote{
    To our knowledge, there is no behavioral state similarity metric for
    non-discounted rewards thus far.
}
The variant `p02md' relies on \algref{alg:approx_anchor} to determine
the anchor locations based on the MMD metric.
\label{item:p02_variant}

\item \textbf{`paxtv', `paxot', and `paxmd'}:
These mnemonics refer to the proposed schemes with the maximum number of anchors
$\nanchormax$ (`ax') as allowed by the computation constraint
\eqref{equ:nparam_constraint} given the feature-to-state dimensional proportion $\rho$.
Such $\nanchormax$ values are specified in \eqref{equ:nfea_nanchor_tuples}.
The three variants here are due to different distribution distance metrics
with the same justification as for `p02$\cdot\ \cdot$' in
\itemref{item:p02_variant} above.

\item \textbf{`pinf'}:
This mnemonic refers to the proposed stepwise-approximator variant, where
there are as many anchors as timesteps.
In theory, there is an infinite number (`inf') of anchors since the horizon is infinite.
\end{enumerate}

\textbf{Third} is whether the experiments involve approximation due to sampling
the initial state $S_0 \sim \isd$ and the next state $S_{t+1} \sim p(\cdot|s_t, a_t)$,
which affects the next reward $R_{t+1} \eqdef r(s_t, a_t, s_{t+1})$ for
a deterministic reward function $r(\cdot)$ given $s_t$, $a_t$, and $s_{t+1}$.
This leads to two kinds of experiments, namely sampling and non-sampling.
Both share the following common properties (which are feasible to obtain
for environments described in \secref{sec:env_nbwpval}).
\begin{itemize}
\item The exact gain of a policy is used so that
    the effect of our proposed method can be isolated.

\item Each experiment-episode (trial) is run long enough in order to
    well approximate the infinite-horizon MDP model.
    The maximum timestep in each experiment-episode is set to
    a multiple of the mixing time, \ie $\txepmax \gets 10 \tmix$.
    Here, $\tmix$ is exactly computed with high precision.

\end{itemize}
In sampling experiments, the scheme `buw' is not feasible because
a model-free RL agent generally cannot sample the states uniformly
during the whole interaction with its environment.
The same goes to the scheme `p02am' in that $\tabsmax$ is unknown to the agent.
The schemes involving TV and OT (\ie `p$\cdot\ \cdot$tv', `p$\cdot\ \cdot$ot')
also cannot be conducted in sampling experiments since they require constructing
intermediate empirical probabilities based on samples (\secref{sec:neighborhood}).

%% file: xprmtsetup_eval.tex
\subsection{Experimental evaluation metrics and protocols}
\label{nbwpval:xprmt_metric}

The training of a system of approximators is as follows.
We select one policy uniformly at random from the set of
all stationary deterministic policies, and
sample the random feature values as specified in \secref{sec:state_fea_ker}.
Then, we run multiple $\nxep$ independent experiment-episodes (trials),
each is with $\txepmax+1$ timesteps, as prescribed in \algref{alg:training_nbwpval}.
This training procedure is carried out for each environment and
each approximation scheme (\secref{sec:xprmt_scheme}).

The quality of a system of approximators is indicated by
the accumulative total error of the square roots of stepwise errors
along one evaluation experiment-episode.
That is,
\begin{equation}
\varepsilon_{\!x} \eqdef \sum_{t=0}^{\txepmax} \sqrt{e_{\!x}^t (\vecb{w}_{\!t})},
\ \text{
    where $x$ is either $\mathbb{P} \mathbb{B}$ or $\mathrm{MS}$, and
    $\vecb{w}_{\!t}$ is of the neighborhood to which $t$ belongs}.
\label{equ:tot_rooterr}
\end{equation}
We perform evaluations using both stepwise $\epbt$ and $\emst$, which utilize
the stepwise state distribution $p^t$ to weight state-wise errors as in
\eqref{equ:epb} and \eqref{equ:ems}, respectively.
In particular, the exact $p^t$ is used so that there
is no sampling-error in evaluation
(hence, one experiment-episode is sufficient for evaluation).

The use of $\epbt$ and $\emst$ in \eqref{equ:tot_rooterr} yields
two evaluation metrics, \ie $\epbtot$ and $\emstot$.
The former $\epbtot$ serves as the gold standard since $\epbt$ is
what the stepwise approximator (`pinf') minimizes.
The $\epbtot$ value is computed by plugging-in
the learned (trained) parameter $\vecb{w}_{\!t}$ into
the $\epb$ formula \eqref{equ:epb} with $\tilde{p} \gets p^t$.
On the other hand, the latter $\emstot$ is natural whenever
the true value is known (but is never told to the RL agent)
as for the environments described in \secref{sec:env_nbwpval}.
For this, we predict the value of every state at every timestep $t$ using
$\vecb{w}_{\!t}$, apply the prediction offset, \ie
$\hat{v}(s) \gets \vecb{w}_{\!t}^\intercal \vecb{f}(s) + o_t$,
then plug-in the predicted value $\hat{v}(s)$ to $\ems$ formula \eqref{equ:ems}
weighted by $\tilde{p} \gets p^t$.

%% file: xprmtresult.tex
\section{Experimental results} \label{nbwpval:xprmtresult}

In this section, we present the experimental results, whose setup is
described in the previous \secref{nbwpval:xprmtsetup}.
There are two groups of results, namely non-sampling and sampling experiments,
as explained in \secref{sec:xprmt_scheme}.
Each is evaluated with two error metrics, namely $\epbtot$ and $\emstot$
(\secref{nbwpval:xprmt_metric}).

\input{xprmtresult_exact}
\input{xprmtresult_sampling}

\input{xprmtresult_sampling_lc.tex}
\input{xprmtresult_exact_tbl}

%% file: xprmtresult_exact.tex
\subsection{Non-sampling experimental results}

\tblrefto{tbl:xprmt_exact_049}{tbl:xprmt_exact_003} present the non-sampling results
of ten schemes in six feature-to-state dimensional ratios (\secref{sec:xprmt_scheme})
and six environments, modelled as unichain MDPs with transient states
(\secref{sec:env_nbwpval}).

From the $\epbtot$ standpoint, the results are as anticipated in that
the lowest error is from the stepwise approximator (pinf), whereas
the second and third lowests are from the maximum number of approximators
(pax$\cdot\ \cdot$) allowed by the feature-to-state dimensional ratios.
More specifically, those with MMD (paxmd) are on par with OT (paxot)
in most cases, where occasionally those with TV (paxtv) become
either the second or third lowest errors (in lieu of paxmd or paxot).

The advantage of having multiple approximators is also obvious based on $\epbtot$,
especially as the dimensional ratio $\rho$ decreases.
Those with a single approximator (\ie buw and p01) have up to
100-fold larger errors than those with two approximators (p02$\cdot\ \cdot$).
The similar behaviour is also observed between `p02$\cdot\ \cdot$' and
those with even more approximators, \ie `pax$\cdot\ \cdot$'.
Among two-approximator schemes, those with a given $\tabsmax$ (p02am) do not
necessarily yield the lowest $\epbtot$.
This is because the first approximator of p02am may not estimate the values of
the least number of recurrent states, compared to p02tv, p02ot, and p02md.
Recall that for `p02$\cdot\ \cdot$', the first approximator should be devoted,
as much as possible, to estimating transient states.
Some recurrent states however, may already have non-zero probabilities
before $\tabsmax$.

From the $\emstot$ standpoint, the stepwise approximator (pinf) achieves
the lowest value or at least, the second lowest in some environments.
This is a direct result of obtaining small $\epbtot$.
In contrast, the other approximator schemes do not achieve small $\epbtot$.
As a consequence, their $\epbtot$ do not correlate with their $\emstot$ counterparts.
That is, lower $\epbtot$ do not necessarily mean lower $\emstot$.
This phenomenon is also observed by \citet[\secc{3.2.2}]{dann_2014_petd}.
Recall that directly minimizing $\emstot$ is not possible in RL since it requires
the knowledge of true (ground-truth) state values as in \eqref{equ:ems}.

Interestingly, the second lowest $\emstot$ is achieved by
the two-approximator scheme with a given $\tabsmax$ (\ie p02am) in most cases;
otherwise, p02am achieves even better results as the lowest.
The third lowest $\emstot$ is attained by various approximator schemes,
including those with one approximators (namely buw and p01).
Such second and third lowest $\emstot$ are up to 100-fold larger than the lowest.

%% file: xprmtresult_sampling.tex
\subsection{Sampling-based experimental results}

\figref{fig:xprmtresult_sampling} depicts the experimental results of
sampling-based approximation settings.
They are from environments \textbf{h6} and \textbf{c10} with
six and ten states, respectively.
We evaluated two numbers of feature dimensions on every environment.
This yields 4 subfigures: each is with two evalution error metrics, namely
$\epbtot$ (total repb) and $\emstot$ (total rems),
on the left and right vertical axes, respectively.

As can be observed, the magnitude ordering of the final $\epbtot$ match with
the exact results (hence, the theory).
This is obvious from results on environment \textbf{h6}, where
pinf attains the lowest final $\epbtot$, followed by paxmd, p02md,
then p01 (the highest).
We speculate that a similar pattern will emerge on the plot of
environment \textbf{c10} as the number of experiment-episodes increases.
That is, the error of pinf will keep decreasing until it crosses those of
p02md then of paxmd (just as it crosses p01).

Such crossing occurs because the number of available samples per experiment-episode
is inversely proportional to the number of anchors (approximators).
In one extreme, each stepwise approximator of the pinf scheme receives only
one sample per experiment-episode.
On the other extreme, a single approximator of the p01 scheme receives
as many samples as timesteps in an experiment-episode.
This therefore results in p01 has lower $\epbtot$ than pinf in the beginning,
but then it plateaus at relatively high errors after some number of samples
(the initial error decrease is not captured in the plot due to
the coarse experiment-episode resolution in the horizontal axis).
The $\epbtot$ behaviours of p02md and paxmd are anticipated to be in
between these two extremes.

On environments \textbf{h6} and \textbf{c10}, the progression and final values
of $\emstot$ roughly follow those of $\epbtot$.
Generally though, the rate of change of $\emstot$ is not as significant as $\epbtot$.
A substantial drop in $\epbtot$ may correspond to merely small drop in $\emstot$,
likewise with the increase.
We can also observe less number of crossing in that $\emstot$ values of
most schemes stay above or below the others: moving up, down or plateau
together simultaneously.

%% file: xprmtresult_sampling_lc.tex
\begin{figure*}
\centering

\begin{subfigure}{0.495\textwidth}
\includegraphics[width=\textwidth]{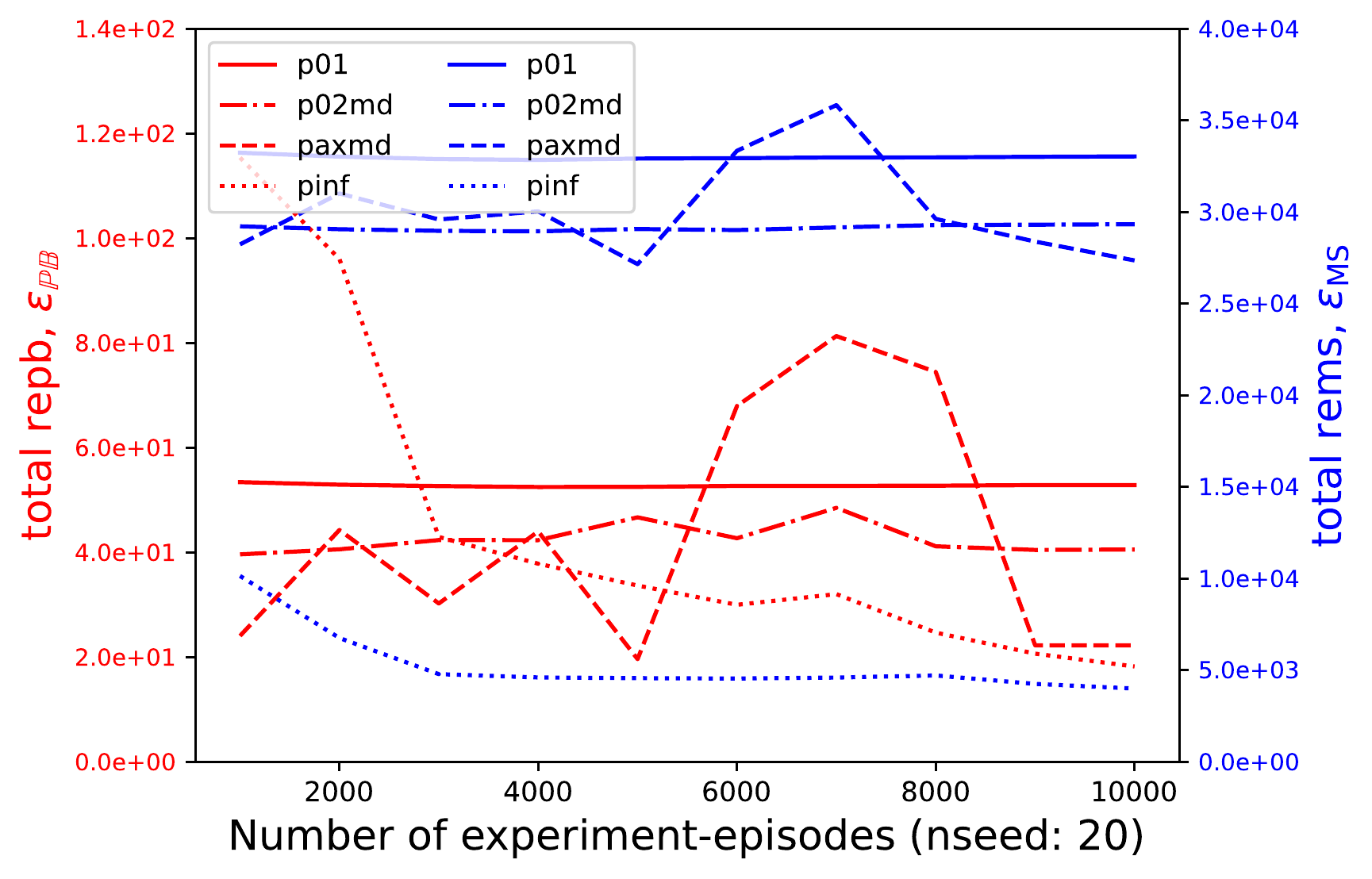}
\subcaption{\textbf{h6} with 1 feature ($\rho=0.19, \nanchormax=5$)}
\end{subfigure}
\begin{subfigure}{0.495\textwidth}
\includegraphics[width=\textwidth]{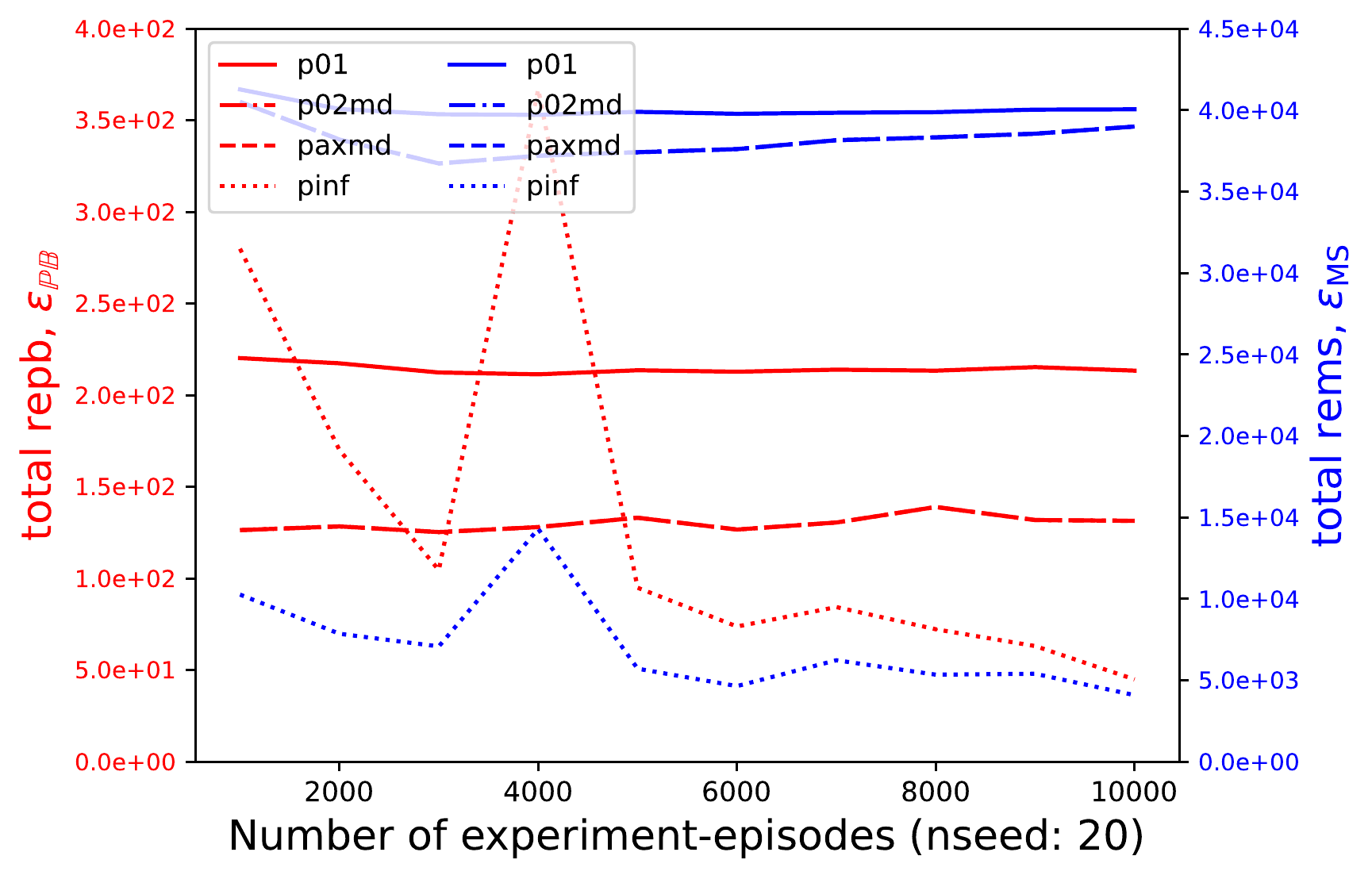}
\subcaption{\textbf{h6} with 2 features ($\rho=0.49, \nanchormax=2$)}
\end{subfigure}
\begin{subfigure}{0.495\textwidth}
\includegraphics[width=\textwidth]{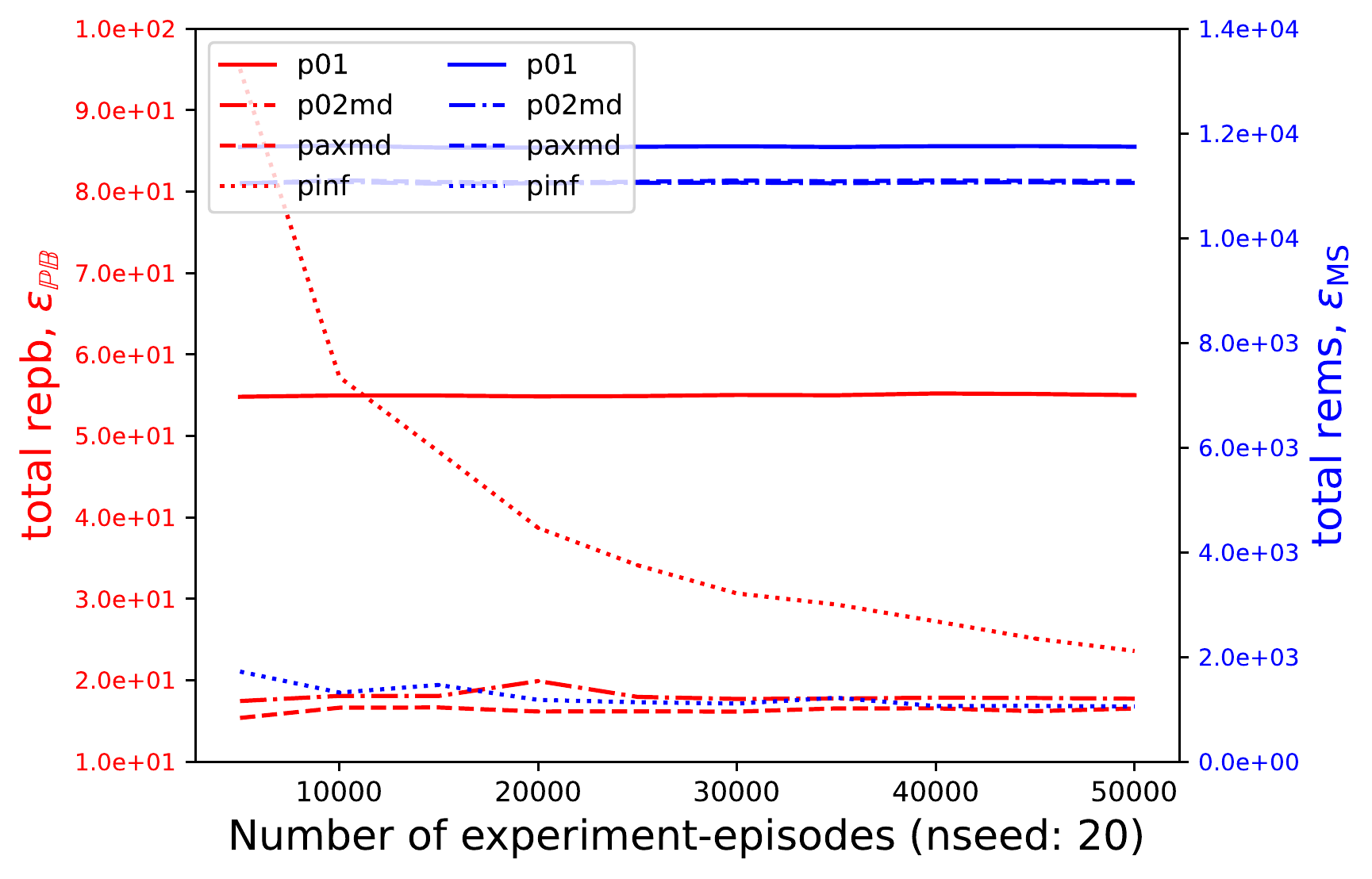}
\subcaption{\textbf{c10} with 3 features ($\rho=0.33, \nanchormax=3$)}
\end{subfigure}
\begin{subfigure}{0.495\textwidth}
\includegraphics[width=\textwidth]{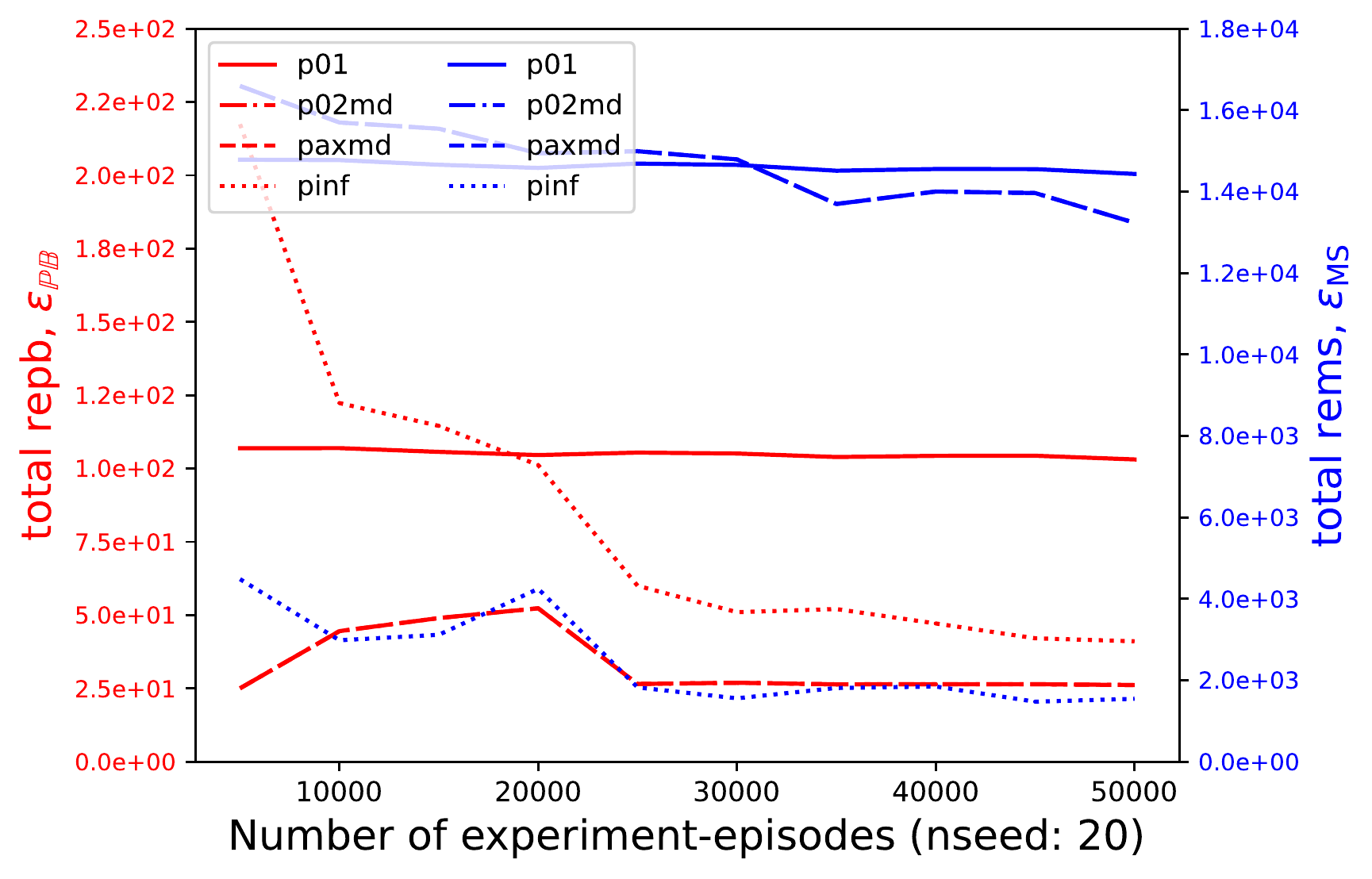}
\subcaption{\textbf{c10} with 4 features ($\rho=0.49, \nanchormax=2$)}
\end{subfigure}

\caption{Sampling-based approximation results on environments
\textbf{h6} (top row) and \textbf{c10} (bottom row).
In each subplot, the {\color{red} left red axis is $\epbtot$ (total repb)},
whereas the {\color{blue} right blue is $\emstot$ (total rems)},
as in \eqref{equ:tot_rooterr}.
Each line is interpolated from 10 data points, and
averaged across 20 repetitions (20 different random-number-generator seeds).
Note that the horizontal axes are different for \textbf{h6} and \textbf{c10},
where there are 10,000 and 50,000 total experiment-episodes (trials), respectively.
The ticks in the vertical axes also vary across subplots.
}
\label{fig:xprmtresult_sampling}
\end{figure*}

%% file: xprmtresult_exact_tbl.tex
\begin{landscape}

\begin{table}[t]
\caption{Non-sampling experimental results with feature-to-state dimensional ratio $\rho = 0.49$,
allowing a maximum of $\nanchormax = 2$ anchors.
The left most column contains environment identifiers (\secref{sec:env_nbwpval}),
then the next ten columns contain $\epbtot$ (total repb), whereas
the last ten columns (gray highlighted) contain $\emstot$ (total rems),
as in \eqref{equ:tot_rooterr}.
For each evaluation metric, its ten columns represent ten schemes (\secref{sec:xprmt_scheme})
with the following mnemonics:
`buw' is for the baseline unweighted scheme,
`p' is for the proposed general schemes whose variations are indicated by
`02', `ax', or `inf' for two anchors, the maximum number of anchors,
or stepwise anchors, respectively, followed by
`am', `tv', `ot', or `md' for the maximum absorption time, total variation,
optimal transport, and MMD state distribution distance metrics, respectively.
The {\color{red} red}, {\color{darkgreen} green}, and {\color{blue} blue}
highlighted numbers indicate {\color{red} the lowest},
{\color{darkgreen} the second lowest}, and {\color{blue} the third lowest}
error values in each 10-column group per row, respectively.
We desire the lowest error (red).
The NaN (Not a Number) indicates that $\floor*{\rho \setsize{S}}$
yields either 0 feature dimension (hence, it cannot be implemented) or
the same feature dimension as that of the lower $\rho$.
In the latter case, the experiment is carried out only for the lower $\rho$
(which induces a higher $\nanchormax$).
}
\label{tbl:xprmt_exact_049}

{\footnotesize %
\setlength{\tabcolsep}{2.25pt} %
\newcolumntype{g}{>{\columncolor[gray]{0.8}}r}
\input{tbl/exact-xprmt/table_exact_nfearatio0.49_nseed20.tex}
} %
\end{table}

\begin{table}[t]
\caption{Non-sampling experimental results with $\rho = 0.33$, and $\nanchormax = 3$ anchors.
For more descriptions, refer to the caption of \tblref{tbl:xprmt_exact_049}.
}
\label{tbl:xprmt_exact_033}

{\footnotesize
\setlength{\tabcolsep}{2.25pt} %
\newcolumntype{g}{>{\columncolor[gray]{0.8}}r}
\input{tbl/exact-xprmt/table_exact_nfearatio0.33_nseed20.tex}
} %
\end{table}

\begin{table}[t]
\caption{Non-sampling experimental results with $\rho = 0.19$, and $\nanchormax = 5$ anchors.
For more descriptions, refer to the caption of \tblref{tbl:xprmt_exact_049}.
}
\label{tbl:xprmt_exact_019}

{\footnotesize
\setlength{\tabcolsep}{2.25pt} %
\newcolumntype{g}{>{\columncolor[gray]{0.8}}r}
\input{tbl/exact-xprmt/table_exact_nfearatio0.19_nseed20.tex}
} %
\end{table}

\begin{table}[t]
\caption{Non-sampling experimental results with $\rho = 0.09$, and $\nanchormax = 11$ anchors.
For more descriptions, refer to the caption of \tblref{tbl:xprmt_exact_049}.
}
\label{tbl:xprmt_exact_009}

{\footnotesize
\setlength{\tabcolsep}{2.25pt} %
\newcolumntype{g}{>{\columncolor[gray]{0.8}}r}
\input{tbl/exact-xprmt/table_exact_nfearatio0.09_nseed20.tex}
} %
\end{table}

\begin{table}[t]
\caption{Non-sampling experimental results with $\rho = 0.06$, and $\nanchormax = 16$ anchors.
For more descriptions, refer to the caption of \tblref{tbl:xprmt_exact_049}.
}
\label{tbl:xprmt_exact_006}

{\footnotesize
\setlength{\tabcolsep}{2.25pt} %
\newcolumntype{g}{>{\columncolor[gray]{0.8}}r}
\input{tbl/exact-xprmt/table_exact_nfearatio0.06_nseed20.tex}
} %
\end{table}

\begin{table}[t]
\caption{Non-sampling experimental results with $\rho = 0.03$, and $\nanchormax = 33$ anchors.
For more descriptions, refer to the caption of \tblref{tbl:xprmt_exact_049}.
}
\label{tbl:xprmt_exact_003}

{\footnotesize
\setlength{\tabcolsep}{2.25pt} %
\newcolumntype{g}{>{\columncolor[gray]{0.8}}r}
\input{tbl/exact-xprmt/table_exact_nfearatio0.03_nseed20.tex}
} %
\end{table}

\end{landscape}

%% file: tbl/exact-xprmt/table_exact_nfearatio0.49_nseed20.tex
\begin{tabular}{lrr|rrrr|rrr|rgg|gggg|ggg|g}
\toprule
{} & \multicolumn{10}{c}{total repb, $\epbtot$} & \multicolumn{10}{c}{total rems, $\emstot$} \\
{} &     buw &     p01 &   p02am &   p02tv &   p02ot &   p02md &   paxtv &   paxot &   paxmd &    pinf &     buw &     p01 &   p02am &   p02tv &   p02ot &   p02md &   paxtv &   paxot &   paxmd &    pinf \\
\midrule
\textbf{h6  } & 5.1e+03 & 2.1e+02 & 3.2e+02 & 4.8e+02 & {\color{darkgreen}9.9e+01} & {\color{blue}1.3e+02} & 4.8e+02 & {\color{darkgreen}9.9e+01} & {\color{blue}1.3e+02} & {\color{red}2.8e-03} & 5.5e+04 & {\color{blue}4.0e+04} & {\color{darkgreen}2.0e+03} & 4.2e+04 & {\color{blue}4.0e+04} & {\color{blue}4.0e+04} & 4.2e+04 & {\color{blue}4.0e+04} & {\color{blue}4.0e+04} & {\color{red}1.1e+03}\\
\textbf{h36 } & 8.2e+03 & 2.3e+02 & 2.0e+02 & {\color{darkgreen}1.2e+02} & {\color{blue}1.3e+02} & 1.4e+02 & {\color{darkgreen}1.2e+02} & {\color{blue}1.3e+02} & 1.4e+02 & {\color{red}6.0e-02} & 1.2e+05 & {\color{blue}1.1e+05} & {\color{darkgreen}1.4e+03} & 1.2e+05 & 4.1e+05 & 1.4e+05 & 1.2e+05 & 4.1e+05 & 1.4e+05 & {\color{red}1.3e+03}\\
\textbf{h36c} & 1.3e+04 & 2.8e+02 & 2.4e+02 & {\color{darkgreen}1.5e+02} & 2.1e+02 & {\color{blue}1.9e+02} & {\color{darkgreen}1.5e+02} & 2.1e+02 & {\color{blue}1.9e+02} & {\color{red}4.0e-02} & 1.9e+05 & {\color{blue}1.0e+05} & {\color{red}2.1e+03} & 4.3e+05 & 6.4e+05 & 4.4e+05 & 4.3e+05 & 6.4e+05 & 4.4e+05 & {\color{darkgreen}2.7e+03}\\
\textbf{h70 } & 1.3e+04 & 3.8e+02 & 3.4e+02 & {\color{darkgreen}2.4e+02} & {\color{blue}2.9e+02} & 3.2e+02 & {\color{darkgreen}2.4e+02} & {\color{blue}2.9e+02} & 3.2e+02 & {\color{red}1.1e+00} & 3.9e+05 & 2.0e+05 & {\color{darkgreen}2.7e+03} & 2.3e+05 & {\color{blue}1.9e+05} & 5.5e+05 & 2.3e+05 & {\color{blue}1.9e+05} & 5.5e+05 & {\color{red}2.6e+03}\\
\textbf{h100} & 1.9e+04 & 5.3e+02 & 4.9e+02 & {\color{blue}3.6e+02} & {\color{darkgreen}3.4e+02} & 4.2e+02 & {\color{blue}3.6e+02} & {\color{darkgreen}3.4e+02} & 4.2e+02 & {\color{red}5.0e+00} & 7.4e+05 & 3.2e+05 & {\color{red}4.7e+03} & {\color{blue}2.3e+05} & 3.7e+05 & 2.5e+05 & {\color{blue}2.3e+05} & 3.7e+05 & 2.5e+05 & {\color{darkgreen}5.7e+03}\\
\textbf{c10 } & 1.6e+04 & 2.9e+02 & 2.5e+02 & 1.2e+02 & {\color{blue}7.3e+01} & {\color{darkgreen}5.7e+01} & 1.2e+02 & {\color{blue}7.3e+01} & {\color{darkgreen}5.7e+01} & {\color{red}7.4e-06} & 2.4e+05 & {\color{blue}1.6e+04} & {\color{darkgreen}1.1e+03} & {\color{blue}1.6e+04} & 1.7e+04 & 1.7e+04 & {\color{blue}1.6e+04} & 1.7e+04 & 1.7e+04 & {\color{red}3.2e+02}\\
\textbf{c35 } & 1.8e+05 & 1.2e+02 & 7.0e+01 & {\color{darkgreen}3.4e+01} & {\color{blue}3.9e+01} & 4.2e+01 & {\color{darkgreen}3.4e+01} & {\color{blue}3.9e+01} & 4.2e+01 & {\color{red}1.2e-02} & 4.1e+06 & {\color{blue}2.4e+04} & {\color{darkgreen}4.0e+02} & 4.2e+04 & 4.6e+04 & 3.9e+04 & 4.2e+04 & 4.6e+04 & 3.9e+04 & {\color{red}3.1e+02}\\
\textbf{c35c} & 1.3e+04 & 6.4e+01 & 4.6e+01 & {\color{darkgreen}2.4e+01} & {\color{blue}2.8e+01} & 2.9e+01 & {\color{darkgreen}2.4e+01} & {\color{blue}2.8e+01} & 2.9e+01 & {\color{red}1.9e-02} & 2.6e+05 & {\color{blue}1.7e+04} & {\color{darkgreen}3.5e+02} & 2.6e+04 & 2.9e+04 & 1.8e+05 & 2.6e+04 & 2.9e+04 & 1.8e+05 & {\color{red}3.3e+02}\\
\textbf{c75 } & 4.8e+03 & 1.7e+02 & 1.2e+02 & {\color{darkgreen}7.2e+01} & {\color{darkgreen}7.2e+01} & {\color{blue}8.5e+01} & {\color{darkgreen}7.2e+01} & {\color{darkgreen}7.2e+01} & {\color{blue}8.5e+01} & {\color{red}1.8e-01} & 3.0e+05 & {\color{blue}5.5e+04} & {\color{darkgreen}9.8e+02} & 1.2e+05 & 7.5e+04 & 9.8e+04 & 1.2e+05 & 7.5e+04 & 9.8e+04 & {\color{red}8.2e+02}\\
\textbf{c100} & 4.9e+03 & 1.4e+02 & 1.2e+02 & {\color{blue}8.3e+01} & {\color{darkgreen}7.7e+01} & 9.6e+01 & {\color{blue}8.3e+01} & {\color{darkgreen}7.7e+01} & 9.6e+01 & {\color{red}7.2e-01} & 4.0e+05 & 9.4e+04 & {\color{red}1.3e+03} & 3.1e+05 & 1.3e+05 & {\color{blue}7.7e+04} & 3.1e+05 & 1.3e+05 & {\color{blue}7.7e+04} & {\color{darkgreen}1.4e+03}\\
\bottomrule
\end{tabular}

%% file: tbl/exact-xprmt/table_exact_nfearatio0.33_nseed20.tex
\begin{tabular}{lrr|rrrr|rrr|rgg|gggg|ggg|g}
\toprule
{} & \multicolumn{10}{c}{total repb, $\epbtot$} & \multicolumn{10}{c}{total rems, $\emstot$} \\
{} &     buw &     p01 &   p02am &   p02tv &   p02ot &   p02md &   paxtv &   paxot &   paxmd &    pinf &     buw &     p01 &   p02am &   p02tv &   p02ot &   p02md &   paxtv &   paxot &   paxmd &    pinf \\
\midrule
\textbf{h6  } &     NaN &     NaN &     NaN &     NaN &     NaN &     NaN &     NaN &     NaN &     NaN &     NaN &     NaN &     NaN &     NaN &     NaN &     NaN &     NaN &     NaN &     NaN &     NaN &     NaN\\
\textbf{h36 } & 1.3e+04 & 4.5e+02 & 3.1e+02 & 1.5e+02 & 1.5e+02 & 1.8e+02 & 1.5e+02 & {\color{darkgreen}1.2e+02} & {\color{blue}1.3e+02} & {\color{red}2.5e-02} & 1.5e+05 & {\color{blue}9.0e+04} & {\color{darkgreen}1.5e+03} & 1.5e+05 & 1.4e+05 & 2.7e+05 & 1.5e+05 & 1.4e+05 & 4.3e+05 & {\color{red}1.4e+03}\\
\textbf{h36c} & 1.5e+04 & 3.1e+02 & 2.5e+02 & 1.6e+02 & 1.6e+02 & 1.6e+02 & 1.6e+02 & {\color{blue}1.4e+02} & {\color{darkgreen}1.2e+02} & {\color{red}9.1e-02} & 1.3e+05 & {\color{blue}8.0e+04} & {\color{red}2.3e+03} & 1.7e+05 & 2.1e+05 & 1.7e+05 & 1.7e+05 & 1.6e+05 & 1.7e+05 & {\color{darkgreen}2.5e+03}\\
\textbf{h70 } & 1.7e+04 & 4.2e+02 & 3.5e+02 & {\color{blue}2.3e+02} & 2.5e+02 & 2.8e+02 & {\color{blue}2.3e+02} & {\color{darkgreen}2.0e+02} & {\color{darkgreen}2.0e+02} & {\color{red}1.3e-01} & 4.8e+05 & {\color{blue}1.9e+05} & {\color{darkgreen}2.8e+03} & 3.0e+05 & 2.5e+05 & 2.6e+05 & 3.0e+05 & 2.5e+05 & 2.4e+05 & {\color{red}2.6e+03}\\
\textbf{h100} & 2.1e+04 & 5.6e+02 & 4.7e+02 & {\color{blue}3.5e+02} & {\color{blue}3.5e+02} & 4.0e+02 & {\color{blue}3.5e+02} & {\color{darkgreen}3.2e+02} & {\color{darkgreen}3.2e+02} & {\color{red}1.5e+00} & 8.9e+05 & 3.1e+05 & {\color{darkgreen}5.1e+03} & 9.5e+05 & {\color{blue}2.8e+05} & 3.4e+05 & 9.5e+05 & 3.2e+05 & 3.4e+05 & {\color{red}4.6e+03}\\
\textbf{c10 } & 9.1e+02 & 9.4e+01 & 6.0e+01 & 3.0e+01 & {\color{darkgreen}2.0e+01} & 2.7e+01 & 3.0e+01 & {\color{blue}2.3e+01} & 2.5e+01 & {\color{red}6.5e-09} & 1.4e+04 & 1.2e+04 & {\color{darkgreen}2.0e+03} & {\color{blue}1.1e+04} & {\color{blue}1.1e+04} & {\color{blue}1.1e+04} & {\color{blue}1.1e+04} & {\color{blue}1.1e+04} & {\color{blue}1.1e+04} & {\color{red}1.9e+03}\\
\textbf{c35 } & 4.0e+03 & 2.4e+02 & 9.3e+01 & 5.1e+01 & 4.9e+01 & 6.0e+01 & 5.1e+01 & {\color{darkgreen}4.0e+01} & {\color{blue}4.4e+01} & {\color{red}1.5e-01} & 1.1e+05 & {\color{blue}3.1e+04} & {\color{darkgreen}4.9e+02} & 3.3e+04 & 4.5e+04 & 1.1e+05 & 3.3e+04 & 5.2e+04 & 4.3e+04 & {\color{red}3.4e+02}\\
\textbf{c35c} & 2.1e+03 & 1.1e+02 & 5.8e+01 & 3.9e+01 & 4.0e+01 & {\color{blue}3.3e+01} & 3.9e+01 & 4.2e+01 & {\color{darkgreen}2.5e+01} & {\color{red}3.3e-01} & 2.9e+04 & {\color{blue}2.0e+04} & {\color{red}4.1e+02} & 3.9e+04 & 5.4e+04 & 3.7e+04 & 3.9e+04 & 6.0e+04 & 2.9e+04 & {\color{darkgreen}7.4e+02}\\
\textbf{c75 } & 4.3e+03 & 2.6e+02 & 1.4e+02 & 6.6e+01 & 7.6e+01 & 8.7e+01 & 6.6e+01 & {\color{darkgreen}6.0e+01} & {\color{blue}6.4e+01} & {\color{red}3.7e-02} & 3.5e+05 & 6.0e+04 & {\color{darkgreen}1.1e+03} & 6.5e+04 & 7.2e+04 & 1.5e+05 & 6.5e+04 & {\color{blue}5.9e+04} & 7.1e+04 & {\color{red}8.2e+02}\\
\textbf{c100} & 4.4e+03 & 1.7e+02 & 1.2e+02 & {\color{blue}7.3e+01} & 7.9e+01 & 1.1e+02 & {\color{blue}7.3e+01} & 7.4e+01 & {\color{darkgreen}6.9e+01} & {\color{red}5.7e-02} & 4.7e+05 & 8.8e+04 & {\color{darkgreen}1.4e+03} & 9.1e+04 & {\color{blue}7.2e+04} & 3.6e+05 & 9.1e+04 & 7.8e+04 & 8.5e+04 & {\color{red}1.2e+03}\\
\bottomrule
\end{tabular}

%% file: tbl/exact-xprmt/table_exact_nfearatio0.19_nseed20.tex
\begin{tabular}{lrr|rrrr|rrr|rgg|gggg|ggg|g}
\toprule
{} & \multicolumn{10}{c}{total repb, $\epbtot$} & \multicolumn{10}{c}{total rems, $\emstot$} \\
{} &     buw &     p01 &   p02am &   p02tv &   p02ot &   p02md &   paxtv &   paxot &   paxmd &    pinf &     buw &     p01 &   p02am &   p02tv &   p02ot &   p02md &   paxtv &   paxot &   paxmd &    pinf \\
\midrule
\textbf{h6  } & 1.5e+03 & 5.3e+01 & 1.2e+02 & 1.5e+02 & 1.3e+02 & 4.2e+01 & 9.7e+01 & {\color{darkgreen}2.6e+01} & {\color{blue}2.7e+01} & {\color{red}2.3e-14} & 1.0e+05 & 3.3e+04 & {\color{red}4.6e+03} & 3.7e+04 & 3.6e+04 & {\color{blue}3.0e+04} & 3.6e+04 & 3.1e+04 & 3.1e+04 & {\color{darkgreen}4.8e+03}\\
\textbf{h36 } & 1.7e+04 & 1.1e+03 & 4.7e+02 & 2.0e+02 & 2.6e+02 & 2.5e+02 & 1.1e+02 & {\color{blue}1.0e+02} & {\color{darkgreen}8.6e+01} & {\color{red}9.8e-02} & 1.7e+05 & {\color{blue}1.4e+05} & {\color{darkgreen}2.0e+03} & 1.8e+05 & 1.9e+05 & 2.7e+05 & 6.8e+05 & 3.8e+05 & 2.2e+05 & {\color{red}1.6e+03}\\
\textbf{h36c} & 1.2e+04 & 4.4e+02 & 2.3e+02 & 1.1e+02 & 1.2e+02 & 1.2e+02 & 7.0e+01 & {\color{blue}5.3e+01} & {\color{darkgreen}5.2e+01} & {\color{red}7.0e-01} & 9.7e+04 & {\color{blue}5.7e+04} & {\color{darkgreen}2.4e+03} & 8.0e+04 & 8.6e+04 & 8.1e+04 & 4.0e+05 & 8.5e+04 & 8.2e+04 & {\color{red}2.2e+03}\\
\textbf{h70 } & 2.1e+04 & 5.9e+02 & 3.9e+02 & 2.4e+02 & 2.5e+02 & 2.9e+02 & 1.8e+02 & {\color{blue}1.6e+02} & {\color{darkgreen}1.3e+02} & {\color{red}4.3e-01} & 5.5e+05 & {\color{blue}1.8e+05} & {\color{darkgreen}3.0e+03} & 2.1e+05 & 1.9e+05 & 2.3e+05 & 9.0e+05 & 1.7e+06 & 5.9e+05 & {\color{red}2.6e+03}\\
\textbf{h100} & 2.6e+04 & 7.0e+02 & 5.0e+02 & 3.2e+02 & 3.5e+02 & 3.8e+02 & 2.3e+02 & {\color{blue}2.2e+02} & {\color{darkgreen}1.9e+02} & {\color{red}1.1e-01} & 1.1e+06 & {\color{blue}2.9e+05} & {\color{darkgreen}5.2e+03} & 4.0e+05 & {\color{blue}2.9e+05} & 3.0e+05 & 4.3e+05 & 6.1e+05 & 5.5e+05 & {\color{red}4.8e+03}\\
\textbf{c10 } & 2.4e+02 & 8.3e+00 & 1.1e+01 & 8.7e+00 & 8.8e+00 & 7.2e+00 & 3.4e+00 & {\color{darkgreen}2.6e+00} & {\color{blue}2.8e+00} & {\color{red}6.1e-14} & 2.2e+04 & {\color{blue}8.8e+03} & {\color{red}3.0e+03} & 8.9e+03 & 9.0e+03 & {\color{blue}8.8e+03} & 9.4e+03 & {\color{blue}8.8e+03} & 9.0e+03 & {\color{darkgreen}3.2e+03}\\
\textbf{c35 } & 4.4e+03 & 8.5e+03 & {\color{darkgreen}1.1e+02} & 5.2e+02 & 4.7e+02 & 4.5e+02 & 2.6e+02 & {\color{blue}1.7e+02} & 2.0e+02 & {\color{red}6.0e-02} & 1.3e+05 & 1.5e+06 & {\color{darkgreen}6.7e+02} & 1.8e+05 & 1.8e+05 & 1.8e+05 & 1.5e+05 & {\color{blue}4.3e+04} & 2.1e+05 & {\color{red}4.8e+02}\\
\textbf{c35c} & 2.2e+03 & 7.3e+02 & {\color{blue}7.8e+01} & 3.0e+02 & {\color{darkgreen}7.3e+01} & 2.1e+02 & 1.4e+02 & 1.3e+02 & 1.6e+02 & {\color{red}5.7e-02} & {\color{blue}3.6e+04} & 6.7e+04 & {\color{darkgreen}5.1e+02} & 7.4e+04 & 7.2e+04 & 7.6e+04 & 8.4e+04 & 9.7e+04 & 8.2e+04 & {\color{red}4.0e+02}\\
\textbf{c75 } & 4.3e+03 & 4.2e+02 & 1.5e+02 & 9.6e+01 & 1.0e+02 & 7.9e+01 & {\color{blue}5.1e+01} & {\color{darkgreen}4.6e+01} & 6.8e+01 & {\color{red}6.7e-02} & 4.1e+05 & {\color{blue}6.7e+04} & {\color{darkgreen}1.2e+03} & 7.2e+04 & 9.8e+04 & 7.8e+04 & 8.4e+04 & 9.3e+04 & 2.6e+05 & {\color{red}9.3e+02}\\
\textbf{c100} & 4.3e+03 & 2.2e+02 & 1.1e+02 & 7.3e+01 & 7.8e+01 & 7.3e+01 & {\color{blue}4.3e+01} & 4.7e+01 & {\color{darkgreen}3.8e+01} & {\color{red}9.8e-02} & 5.8e+05 & 8.6e+04 & {\color{darkgreen}1.5e+03} & 8.0e+04 & {\color{blue}7.2e+04} & 8.4e+04 & 1.2e+05 & 2.0e+05 & 1.1e+05 & {\color{red}1.3e+03}\\
\bottomrule
\end{tabular}

%% file: tbl/exact-xprmt/table_exact_nfearatio0.09_nseed20.tex
\begin{tabular}{lrr|rrrr|rrr|rgg|gggg|ggg|g}
\toprule
{} & \multicolumn{10}{c}{total repb, $\epbtot$} & \multicolumn{10}{c}{total rems, $\emstot$} \\
{} &     buw &     p01 &   p02am &   p02tv &   p02ot &   p02md &   paxtv &   paxot &   paxmd &    pinf &     buw &     p01 &   p02am &   p02tv &   p02ot &   p02md &   paxtv &   paxot &   paxmd &    pinf \\
\midrule
\textbf{h6  } &     NaN &     NaN &     NaN &     NaN &     NaN &     NaN &     NaN &     NaN &     NaN &     NaN &     NaN &     NaN &     NaN &     NaN &     NaN &     NaN &     NaN &     NaN &     NaN &     NaN\\
\textbf{h36 } & 1.8e+04 & 5.6e+04 & 7.1e+02 & 5.7e+03 & 3.8e+03 & 7.4e+02 & 4.4e+02 & {\color{blue}3.6e+02} & {\color{darkgreen}2.2e+02} & {\color{red}1.9e-02} & {\color{blue}2.2e+05} & 4.9e+06 & {\color{darkgreen}3.1e+03} & 2.0e+06 & 1.1e+06 & 5.2e+05 & 4.2e+05 & 5.1e+05 & 4.4e+05 & {\color{red}2.4e+03}\\
\textbf{h36c} & 7.8e+03 & 3.3e+03 & {\color{blue}2.4e+02} & 1.3e+03 & 1.5e+03 & 1.8e+03 & 5.2e+02 & {\color{darkgreen}2.1e+02} & 4.1e+02 & {\color{red}2.7e-02} & {\color{blue}8.2e+04} & 2.2e+05 & {\color{darkgreen}2.6e+03} & 2.9e+05 & 2.2e+05 & 9.5e+05 & 8.2e+05 & 5.2e+05 & 6.0e+05 & {\color{red}2.5e+03}\\
\textbf{h70 } & 2.7e+04 & 1.8e+03 & 7.0e+02 & 3.5e+02 & 3.6e+02 & 4.2e+02 & {\color{blue}1.0e+02} & {\color{darkgreen}8.5e+01} & 1.9e+02 & {\color{red}5.7e-02} & 6.2e+05 & {\color{blue}1.4e+05} & {\color{darkgreen}3.8e+03} & 2.4e+05 & 2.2e+05 & 3.0e+05 & 3.5e+05 & 3.2e+05 & 5.4e+06 & {\color{red}3.1e+03}\\
\textbf{h100} & 3.1e+04 & 1.4e+03 & 6.8e+02 & 4.0e+02 & 5.0e+02 & 8.2e+02 & 1.3e+02 & {\color{blue}1.2e+02} & {\color{darkgreen}9.6e+01} & {\color{red}7.9e-02} & 1.1e+06 & {\color{blue}3.2e+05} & {\color{darkgreen}5.9e+03} & 3.6e+05 & {\color{blue}3.2e+05} & 9.7e+06 & 5.2e+05 & 4.8e+05 & 4.5e+05 & {\color{red}5.1e+03}\\
\textbf{c10 } &     NaN &     NaN &     NaN &     NaN &     NaN &     NaN &     NaN &     NaN &     NaN &     NaN &     NaN &     NaN &     NaN &     NaN &     NaN &     NaN &     NaN &     NaN &     NaN &     NaN\\
\textbf{c35 } & 2.7e+03 & 7.4e+02 & 7.7e+01 & 1.8e+02 & 1.4e+02 & 8.9e+01 & 2.7e+02 & {\color{blue}3.5e+01} & {\color{darkgreen}1.4e+01} & {\color{red}1.0e-06} & 1.3e+05 & 3.2e+05 & {\color{red}4.3e+03} & 4.4e+05 & 2.6e+05 & 3.2e+05 & 1.5e+05 & 1.1e+05 & {\color{blue}1.0e+05} & {\color{darkgreen}4.5e+03}\\
\textbf{c35c} & 1.2e+03 & 5.7e+02 & 4.1e+01 & 1.2e+02 & 6.6e+01 & 2.9e+01 & 4.5e+01 & {\color{blue}1.7e+01} & {\color{darkgreen}1.5e+01} & {\color{red}3.8e-07} & {\color{blue}3.4e+04} & 4.8e+05 & {\color{red}2.6e+03} & 2.2e+05 & 1.6e+05 & 1.5e+05 & 1.8e+05 & 1.4e+05 & 1.4e+05 & {\color{darkgreen}2.9e+03}\\
\textbf{c75 } & 4.6e+03 & 2.1e+03 & 1.7e+02 & 8.5e+02 & 7.7e+02 & 5.2e+02 & {\color{blue}1.6e+02} & 9.3e+03 & {\color{darkgreen}7.7e+01} & {\color{red}2.3e-01} & 4.4e+05 & {\color{darkgreen}1.8e+05} & {\color{red}1.4e+03} & 2.9e+05 & 2.6e+05 & {\color{blue}2.1e+05} & 3.5e+05 & 3.9e+06 & 2.8e+05 & {\color{red}1.4e+03}\\
\textbf{c100} & 4.9e+03 & 4.9e+02 & 1.2e+02 & 1.4e+02 & 1.7e+02 & {\color{blue}7.9e+01} & {\color{darkgreen}2.3e+01} & 2.9e+02 & 2.8e+02 & {\color{red}8.4e-02} & 6.1e+05 & 8.7e+04 & {\color{darkgreen}1.5e+03} & 7.8e+04 & {\color{blue}7.3e+04} & 9.5e+04 & 3.0e+05 & 2.3e+07 & 4.2e+06 & {\color{red}1.4e+03}\\
\bottomrule
\end{tabular}

%% file: tbl/exact-xprmt/table_exact_nfearatio0.06_nseed20.tex
\begin{tabular}{lrr|rrrr|rrr|rgg|gggg|ggg|g}
\toprule
{} & \multicolumn{10}{c}{total repb, $\epbtot$} & \multicolumn{10}{c}{total rems, $\emstot$} \\
{} &     buw &     p01 &   p02am &   p02tv &   p02ot &   p02md &   paxtv &   paxot &   paxmd &    pinf &     buw &     p01 &   p02am &   p02tv &   p02ot &   p02md &   paxtv &   paxot &   paxmd &    pinf \\
\midrule
\textbf{h6  } &     NaN &     NaN &     NaN &     NaN &     NaN &     NaN &     NaN &     NaN &     NaN &     NaN &     NaN &     NaN &     NaN &     NaN &     NaN &     NaN &     NaN &     NaN &     NaN &     NaN\\
\textbf{h36 } & 1.1e+04 & 1.8e+03 & 3.9e+02 & 2.4e+02 & 5.1e+02 & 2.2e+02 & {\color{blue}5.0e+01} & 7.3e+01 & {\color{darkgreen}3.4e+01} & {\color{red}4.6e-01} & {\color{blue}2.6e+05} & 1.6e+06 & {\color{red}4.9e+03} & 1.5e+06 & 4.4e+06 & 7.8e+05 & 6.5e+05 & 7.2e+05 & 2.2e+06 & {\color{darkgreen}5.1e+03}\\
\textbf{h36c} & 5.1e+03 & 1.2e+03 & 1.4e+02 & 1.3e+02 & 1.2e+03 & 1.4e+02 & {\color{darkgreen}1.4e+01} & 4.2e+01 & {\color{blue}1.5e+01} & {\color{red}9.6e-02} & {\color{blue}7.2e+04} & 8.4e+05 & {\color{red}2.9e+03} & 5.2e+05 & 3.0e+06 & 1.2e+06 & 2.1e+06 & 3.8e+05 & 3.1e+05 & {\color{darkgreen}8.3e+03}\\
\textbf{h70 } & 2.9e+04 & 7.1e+03 & 1.1e+03 & 1.2e+03 & 1.1e+03 & 5.9e+02 & {\color{darkgreen}9.9e+01} & {\color{blue}1.3e+02} & 1.4e+02 & {\color{red}4.8e-01} & 6.3e+05 & {\color{blue}3.8e+05} & {\color{darkgreen}5.5e+03} & 4.4e+05 & 4.7e+05 & 4.3e+05 & 5.2e+05 & 7.5e+05 & 9.8e+05 & {\color{red}4.0e+03}\\
\textbf{h100} & 3.2e+04 & 2.3e+03 & 8.0e+02 & 4.7e+02 & 5.8e+02 & 5.4e+02 & {\color{darkgreen}7.2e+01} & 7.9e+01 & {\color{blue}7.6e+01} & {\color{red}3.1e-01} & 1.2e+06 & 3.1e+05 & {\color{darkgreen}6.2e+03} & 3.4e+05 & {\color{blue}2.7e+05} & 6.9e+05 & 6.3e+05 & 4.0e+05 & 5.5e+05 & {\color{red}5.3e+03}\\
\textbf{c10 } &     NaN &     NaN &     NaN &     NaN &     NaN &     NaN &     NaN &     NaN &     NaN &     NaN &     NaN &     NaN &     NaN &     NaN &     NaN &     NaN &     NaN &     NaN &     NaN &     NaN\\
\textbf{c35 } & 4.2e+03 & 3.3e+02 & 6.0e+01 & 3.5e+01 & 3.1e+01 & 2.4e+01 & {\color{darkgreen}4.7e+00} & {\color{blue}8.0e+00} & 9.1e+00 & {\color{red}3.2e-07} & 2.3e+05 & 5.1e+05 & {\color{red}6.3e+03} & 1.3e+05 & 1.1e+05 & 8.4e+04 & 8.0e+04 & 8.9e+04 & {\color{blue}7.3e+04} & {\color{darkgreen}8.0e+03}\\
\textbf{c35c} & 8.2e+02 & 2.5e+02 & 2.2e+01 & 3.4e+01 & 1.3e+01 & 1.4e+01 & {\color{darkgreen}3.9e+00} & 7.6e+00 & {\color{blue}4.9e+00} & {\color{red}4.6e-10} & {\color{blue}3.7e+04} & 3.7e+05 & {\color{red}3.8e+03} & 1.1e+05 & 9.1e+04 & 9.1e+04 & 7.3e+04 & 8.7e+04 & 8.9e+04 & {\color{darkgreen}4.0e+03}\\
\textbf{c75 } & 4.1e+03 & 1.9e+03 & 1.3e+02 & 5.4e+03 & 5.8e+02 & 2.4e+02 & {\color{darkgreen}2.5e+01} & 1.2e+02 & {\color{blue}9.3e+01} & {\color{red}1.8e-03} & {\color{blue}4.4e+05} & 5.4e+05 & {\color{red}1.4e+03} & 3.6e+07 & 1.9e+06 & 1.1e+06 & 6.0e+05 & 7.4e+05 & 1.5e+06 & {\color{darkgreen}1.6e+03}\\
\textbf{c100} & 5.3e+03 & 1.3e+03 & 1.4e+02 & 5.3e+02 & 5.0e+02 & 1.4e+02 & {\color{blue}3.6e+01} & 4.2e+01 & {\color{darkgreen}2.6e+01} & {\color{red}1.3e-01} & 6.3e+05 & {\color{blue}1.1e+05} & {\color{darkgreen}1.6e+03} & {\color{blue}1.1e+05} & {\color{blue}1.1e+05} & 1.4e+05 & 1.7e+05 & 1.3e+05 & 1.4e+05 & {\color{red}1.4e+03}\\
\bottomrule
\end{tabular}

%% file: tbl/exact-xprmt/table_exact_nfearatio0.03_nseed20.tex
\begin{tabular}{lrr|rrrr|rrr|rgg|gggg|ggg|g}
\toprule
{} & \multicolumn{10}{c}{total repb, $\epbtot$} & \multicolumn{10}{c}{total rems, $\emstot$} \\
{} &     buw &     p01 &   p02am &   p02tv &   p02ot &   p02md &   paxtv &   paxot &   paxmd &    pinf &     buw &     p01 &   p02am &   p02tv &   p02ot &   p02md &   paxtv &   paxot &   paxmd &    pinf \\
\midrule
\textbf{h6  } &     NaN &     NaN &     NaN &     NaN &     NaN &     NaN &     NaN &     NaN &     NaN &     NaN &     NaN &     NaN &     NaN &     NaN &     NaN &     NaN &     NaN &     NaN &     NaN &     NaN\\
\textbf{h36 } & 8.4e+02 & 1.2e+03 & 2.0e+02 & 8.3e+01 & 9.6e+01 & 7.4e+01 & {\color{blue}1.7e+01} & {\color{darkgreen}1.6e+01} & {\color{darkgreen}1.6e+01} & {\color{red}1.3e-12} & {\color{blue}3.3e+05} & 2.5e+06 & {\color{darkgreen}2.1e+04} & 1.1e+06 & 1.0e+06 & 9.4e+05 & 5.9e+05 & 6.0e+05 & 6.0e+05 & {\color{red}1.6e+04}\\
\textbf{h36c} & 2.2e+02 & 3.0e+02 & 4.8e+01 & 4.9e+01 & 9.8e+02 & 2.9e+01 & {\color{darkgreen}1.7e+00} & {\color{blue}1.8e+00} & 2.2e+00 & {\color{red}2.3e-13} & {\color{blue}7.9e+04} & 4.8e+05 & {\color{darkgreen}8.6e+03} & 4.2e+05 & 8.0e+06 & 3.2e+05 & 2.3e+05 & 2.4e+05 & 2.4e+05 & {\color{red}8.2e+03}\\
\textbf{h70 } & 1.4e+04 & 4.4e+03 & 4.6e+02 & 4.2e+02 & 3.6e+02 & 3.3e+02 & {\color{darkgreen}2.0e+01} & {\color{blue}3.3e+01} & 4.1e+01 & {\color{red}1.1e-04} & {\color{blue}5.8e+05} & 5.5e+06 & {\color{red}6.4e+03} & 3.5e+06 & 3.5e+06 & 4.3e+06 & 1.8e+06 & 2.1e+06 & 1.7e+06 & {\color{darkgreen}1.1e+04}\\
\textbf{h100} & 4.2e+04 & 2.6e+04 & 4.1e+03 & 7.0e+03 & 7.3e+03 & 1.0e+04 & 8.3e+02 & {\color{darkgreen}3.9e+02} & {\color{blue}7.2e+02} & {\color{red}3.7e+00} & 1.1e+06 & 1.6e+06 & {\color{darkgreen}1.4e+04} & 1.8e+06 & 1.5e+06 & 2.4e+06 & {\color{blue}1.0e+06} & 1.3e+06 & 1.4e+06 & {\color{red}7.2e+03}\\
\textbf{c10 } &     NaN &     NaN &     NaN &     NaN &     NaN &     NaN &     NaN &     NaN &     NaN &     NaN &     NaN &     NaN &     NaN &     NaN &     NaN &     NaN &     NaN &     NaN &     NaN &     NaN\\
\textbf{c35 } & 6.5e+02 & 9.1e+01 & 2.4e+01 & 9.4e+00 & 8.5e+00 & 6.1e+00 & {\color{darkgreen}1.7e+00} & {\color{blue}2.0e+00} & 2.7e+00 & {\color{red}4.0e-13} & 4.2e+05 & 2.6e+05 & {\color{darkgreen}8.4e+03} & 1.0e+05 & 9.3e+04 & 7.6e+04 & 6.9e+04 & {\color{blue}6.2e+04} & {\color{blue}6.2e+04} & {\color{red}8.0e+03}\\
\textbf{c35c} & 6.5e+01 & 6.0e+01 & 2.1e+01 & 8.7e+00 & 4.0e+00 & 5.0e+00 & 1.3e+00 & {\color{darkgreen}1.1e+00} & {\color{blue}1.2e+00} & {\color{red}6.2e-13} & {\color{blue}3.6e+04} & 1.3e+05 & {\color{darkgreen}5.2e+03} & 7.1e+04 & 6.3e+04 & 6.6e+04 & 6.3e+04 & 4.5e+04 & 6.3e+04 & {\color{red}5.1e+03}\\
\textbf{c75 } & 2.1e+03 & 4.0e+03 & 5.6e+01 & 4.0e+02 & 9.2e+01 & 3.6e+01 & {\color{darkgreen}3.5e+00} & {\color{blue}8.6e+00} & {\color{darkgreen}3.5e+00} & {\color{red}7.0e-09} & 4.3e+05 & 1.5e+07 & {\color{red}1.1e+04} & 7.1e+06 & 7.7e+05 & 2.2e+05 & {\color{blue}1.6e+05} & 2.4e+05 & {\color{blue}1.6e+05} & {\color{darkgreen}1.2e+04}\\
\textbf{c100} & 3.8e+03 & 6.2e+04 & 1.2e+02 & 3.6e+02 & 2.1e+02 & 1.2e+02 & {\color{darkgreen}9.5e+00} & 2.5e+01 & {\color{blue}1.4e+01} & {\color{red}2.3e-05} & 6.7e+05 & 6.7e+07 & {\color{red}4.8e+03} & 8.1e+05 & 5.9e+05 & 5.0e+05 & {\color{blue}3.0e+05} & 1.2e+06 & 3.3e+05 & {\color{darkgreen}5.0e+03}\\
\bottomrule
\end{tabular}

%% file: conclude.tex
\section{Conclusions, limitations, and future works}
\label{nbwpval:conclude}

We propose a system of seminorm LSTD approximators for estimating
the relative bias value from multiple transient and multiple recurrent states
in unichain MDPs.
To this end, we derive an expression for the minimizer of the seminorm LSTD
that enables approximation through sampling (as for model-free RL).
We also devise a general procedure for LSTD-based policy evaluation,
from which the relative value approximator emerges as a special case;
so do the other existing LSTD-based approximators for recurrent MDPs and
unichain MDPs with one 0-reward recurrent state.
Experimental results validate that a system with more approximators yields
the lower projected Bellman errors.
It is also empirically shown that timestep-neighborhoods can reasonably
be specified based on estimating the squared MMD among stepwise state distributions
(using their state samples).

The proposed method addresses the problem of minimizing MSBPE that is based on
the one-step Bellman operator $\bo$ in \eqref{equ:poisson_avgrew}.
It is known that $\bo$ can be extended to its multi-step variant.
That is,
\begin{equation*}
\vecb{v} = \bo^m [\vecb{v}]
\eqdef \left\{ \sum_{\tau=0}^{m - 1} \mat{P}^\tau (\vecb{r} - \vecb{g}) \right\}
    + \mat{P}^m \vecb{v},
\quad \text{for $m \ge 1$ (where $m=1$ gives the one-step variant)}.
\end{equation*}
Another known extension is to utilize a weighted average over all $\bo^m$
for $m=1, 2, \ldots$, which gives
\begin{equation*}
\vecb{v} = \bo_{\!\lambda} [\vecb{v}],
\quad \text{where}\
\bo_{\!\lambda} \eqdef (1 - \lambda) \sum_{m=1}^{\infty} \lambda^{m-1} \bo^m,\
\text{for a trace-decay factor $\lambda \in [0, 1)$}.
\end{equation*}
These two extensions essentially provide a way to control the bias-variance trade-off
in the value approximation.\footnote{
    One way to see this bias-variance trade-off is from
    the semi-gradient TD viewpoint as follows.
    Setting $m$ to $1$ leads to the one-step TD algorithm.
    As in \eqref{equ:semigrad_update}, it approximates the true value
    $v(s_t) \approx r_{t+1} - g + \hat{v}(s_{t+1})$, which has low variance
    (as it involves only one-step next state and reward samples)
    but is biased towards the estimator $\hat{v}$.
    On the other hand, $m$-step TD with $m$ approaches infinity approximates
    $v(s_t) \approx (r_{t+1} - g) + (r_{t+2} - g) + \ldots$, which is
    unbiased (due to no involvement of $\hat{v}$)
    but has high variance (due to an infinitely long sequence of reward samples).
    Note that in practice, the gain $g$ should also be approximated.
}
They lead to $m$-step TD, TD($\lambda$) and LSTD($\lambda$) algorithms
\citep{sutton_1988_td, boyan_2002_lstd}.
It is interesting therefore to extend our proposed method to a system of
multiple $m$-step seminorm LSTD (or seminorm LSTD($\lambda$)) approximators.
This includes examination about how the size of a neighborhood affects
the suitable values for $m$ and $\lambda$.

This work has not taken the advantage of iterative techniques for calculating
the Moore-Penrose pseudoinverse (for computing the minimizer of the seminorm LSTD).
It also has not exploited the fact that $(1, 3)$-pseudoinverse is sufficient
for an LS solution that is not necessarily a minimum-norm solution
\citep[\tbl{6.1}]{campbell_2009_ginv}.
Note that the pseudoinverse used throughout this work,
\ie $\mat{M}^\dagger$ for a matrix $\mat{M}$, is the full $(1, 2, 3, 4)$-pseudoinverse.

Additionally, the (finite) sample complexity of the proposed method deserves
a careful study.
This includes the relationship between the number of neighborhoods and
the number of samples to achieve a certain level of errors.
We anticipate an intricate interplay because a fewer number of neighborhoods
(equivalently, more neighbors per neighborhood) yields more violation to
the i.i.d sample condition in computing the sample means for
the minimizing parameter of the seminorm LSTD (\thmref{thm:epb_minimizer}).

Lastly, our proposed method can be modified to become
a system of semi-gradient seminorm TD approximators.
This is worth investigating because semi-gradient algorithms involve
an initial parameter value, which is paradoxically beneficial for
policy iteration RL methods in that it can be set to the parameter of
the last policy's value approximator.
The modification should leverage the fact that
both semi-gradient TD and LSTD algorithms converge to the same TD fixed point
(\secref{nbwpval:prelim}).